%% file: arxiv_multi.tex
\documentclass[twoside,11pt]{article}

\usepackage{jmlr2e}
\usepackage{amsmath,amssymb,amstext,bbm,algorithm,algorithmic,graphicx,wrapfig,
  color,multirow,bigstrut,wasysym}
\RequirePackage[OT1]{fontenc}
\usepackage{stmaryrd,enumerate,latexsym,bm,amsfonts,
  subfigure,wrapfig,verbatim,cases}
\usepackage[small]{caption}

\input{summary.tex}

\newcommand{\thesis}[1]{}
\newcommand{\notthesis}[1]{#1}
\newcommand{\old}[1]{#1}

\ShortHeadings{A Theory of Multiclass Boosting}{I. Mukherjee and
  R. E. Schapire}

\begin{document}

\title{A Theory of Multiclass Boosting}

\author{\name Indraneel Mukherjee \email imukherj@cs.princeton.edu\\
    \addr   Princeton University,
  Department of Computer Science,
  Princeton, NJ 08540 USA
  \AND
  \name Robert E. Schapire \email schapire@cs.princeton.edu\\
  \addr   Princeton University,
  Department of Computer Science,
  Princeton, NJ 08540 USA}

\editor{}

\maketitle

\begin{abstract}%
     Boosting combines weak classifiers to form highly accurate
  predictors. Although the case of binary classification is well
  understood, in the multiclass setting, the ``correct'' requirements
  on the weak classifier, or the notion of the most efficient boosting
  algorithms are missing. In this paper, we create a broad and general
  framework, within which we make precise and identify the optimal
  requirements on the weak-classifier, as well as design the most
  effective, in a certain sense, boosting algorithms that assume such
  requirements.
\end{abstract}

\begin{keywords}
  Multiclass, boosting, weak learning condition, drifting games
\end{keywords}

\input{multiclass.tex}

\bibliography{newbib}

\end{document}

%% file: summary.tex
\renewcommand{\th}{^{\rm th}}

 \newcommand{\floor}[1]{\lfloor #1 \rfloor}
 \newcommand{\ceil}[1]{\lceil #1 \rceil}

\providecommand{\abs}[1]{\lvert#1\rvert}
\providecommand{\norm}[1]{\lVert#1\rVert}

\newcommand{\set}[1]{\left\{#1\right\}}

\newcommand{\mat}[1]{{\mathbf #1}}

\newcommand{\B}{\mat{B}}
\newcommand{\C}{\mat{C}}
\newcommand{\vh}{\mat{h}}

\newcommand{\1}{\mat{1}}

 \renewcommand{\H}{\mathcal{H}}
\newenvironment{proof*}{\noindent{\bf Proof:}}{}
\newcommand{\ignore}[1]{}

\newcommand{\bb}[1]{\llbracket{} #1 \rrbracket{}}

\newcommand{\vzero}{\mathbf{0}}

\newcommand{\U}{\mat{U}}
\newcommand{\M}{\mat{M}}
\newcommand{\Q}{\mat{Q}}

\newcommand{\dotp}[2]{\left \langle #1 , #2 \right \rangle}
\newcommand{\E}{\mathbb{E}}
\newcommand{\eps}{\varepsilon}
\newcommand{\vx}{\vec{x}}

\newcommand{\ve}{\vec{e}}

\newcommand{\vs}{\vec{s}}
\newcommand{\vu}{\vec{u}}

\newcommand{\vv}{\vec{v}}
\newcommand{\vp}{\vec{p}}

\newcommand{\vlams}{\bm{\lambda}^*}

\newcommand{\vc}{\vec{c}}
\newcommand{\vq}{\vec{q}}

\newcommand{\f}{\mat{f}}

\newcommand{\va}{\mat{a}}
\newcommand{\vb}{\mat{b}}
\newcommand{\err}{\text{err}}

\newcommand{\enc}[1]{\left( #1 \right)}
\newcommand{\enco}[1]{\left [ #1 \right ]}
\newcommand{\enct}[1]{\left \{ #1 \right \}}
\newcommand{\one}{\mathbbm{1}}

\newcommand{\rw}{\mathcal{R}}

\newcommand{\wlsq}{\mat{W}\mat{L}^{\bm{\sigma}}_\vq}

\newcommand{\vlambda}{{\bm{\lambda}}}

\newcommand{\argmax}{\operatornamewithlimits{argmax}}
\newcommand{\argmin}{\operatornamewithlimits{argmin}}

\newcommand{\To}{\rightarrow}
\newcommand{\R}{\mathbb{R}}

\newcommand{\Z}{\mathbb{Z}}

\newcommand{\eqdef}{\stackrel{\vartriangle}{=}}

{\vspace{2ex} \begin{center} \textbf{\large Homework 4}
  \end{center}
  
  \begin{enumerate}}{\end{enumerate}}

\renewcommand{\vec}[1]{\mathbf{#1}}


%% file: multiclass.tex
\newcommand{\picfile}{paper_graphs}
\newcommand{\Csp}{\mathcal{C}}
\newcommand{\Crow}{\mathcal{C}_0}
\newcommand{\Csam}{{\mathcal{C}^{\text{SAM}}}}
\newcommand{\ada}{\texttt{M1}}
\newcommand{\greedy}{\texttt{Greedy}}
\newcommand{\greedyinfo}{\texttt{Greedy-Info}}
\newcommand{\mh}{\texttt{MH}}
\newcommand{\opt}{\texttt{MM}}
\newcommand{\abe}{\texttt{Abe}}
\newcommand{\eqwls}{\exists\Q.\wlsq}
\newcommand{\Cbin}{\mathcal{C}^{\text{bin}}}
\newcommand{\Csig}{\mathcal{C}^{\text{eor}}}
\newcommand{\csig}{\mathcal{C}^{\text{eor}}_0}
\newcommand{\Bgam}{\mathcal{B}^{\text{eor}}_{\gamma}}
\newcommand{\cB}{\mathcal{B}^{\text{eor}}}
\newcommand{\Ubin}{\mat{U}^{\text{bin}}}
\newcommand{\Ugam}{\mat{U}^{}_{\gamma}}
\newcommand{\Cmr}{\mathcal{C}^{\text{MR}}}
\newcommand{\Cmh}{\mathcal{C}^{\text{MH}}}
\newcommand{\Cmone}{\mathcal{C}^{\text{M1}}}
\newcommand{\Hall}{\mathcal{H}^{\text{all}}}
\newcommand{\Mone}{\mat{B}^{\text{M1}}}
\newcommand{\MH}{\mat{B}^{\text{MH}}}
\newcommand{\MR}{\mat{B}^{\text{MR}}}
\renewcommand{\vh}{{\mat{1}_h}}
\renewcommand{\1}{\mathbbm{1}}
\newcommand{\mow}{(\Cmone,\Mone_\gamma)}
\newcommand{\mhw}{(\Cmh,\MH_\gamma)}
\newcommand{\phib}{\phi}
\newcommand{\dgam}{\Delta_\gamma^k}
\newcommand{\phin}{\tilde{\phi}}
\newcommand{\Lexpe}{L^{\exp}_\eta}
\newcommand{\Lexp}{L^{\exp}}
\newcommand{\Lzero}{L^{{\rm err}}}
\newcommand{\ptope}{\mathcal{P}}
\newcommand{\vps}{\vp^*}
\newcommand{\tp}{\tilde{p}^+}

\ifdefined\thesismode
\renewcommand{\X}{\mathcal{X}}
\else
\newcommand{\X}{\mathcal{X}}
\fi
\newcommand{\Y}{\mathcal{Y}}
\newcommand{\Hopt}{H_{{\rm opt}}}
\renewcommand{\err}{{\rm err}_D}
\newcommand{\ersk}{\widehat{\rm risk}}
\newcommand{\rsk}{{\rm risk}_D}
\newcommand{\cF}{\bar{F}}
\newcommand{\cH}{\bar{H}}
\newcommand{\Fopt}{F^*}

\newcommand{\tf}[1]{\widetilde{#1}}

\section{Introduction}
\label{multi:intro:sec}

Boosting~\citep{SchapireFr12} refers to a general technique of
 combining rules of thumb, or weak classifiers,
 to form highly accurate combined classifiers.
Minimal
 demands are placed on the
 weak classifiers, so that a variety of learning algorithms, also
 called weak-learners, can be employed to discover these simple rules,
 making the algorithm widely applicable. The theory of boosting is
 well-developed for the case of binary classification. In particular,
 the exact requirements on the weak classifiers in this setting are known:
 any algorithm that predicts better than random on any distribution
 over the training set is said to satisfy the weak learning
 assumption. Further, boosting algorithms that minimize loss as
 efficiently as possible have been designed.  Specifically, it is
 known that the Boost-by-majority~\citep{Freund95}  algorithm is optimal
 in a certain sense, and that AdaBoost~\citep{FreundSc97} is a practical
 approximation.

 Such an understanding would be desirable in the multiclass setting as
 well, since many natural classification problems involve more than
 two labels, e.g. recognizing a digit from its image, natural language
 processing tasks such as part-of-speech tagging, and object
 recognition in vision.
However, for such multiclass problems, a complete
 theoretical understanding of boosting is lacking.
In particular, we
 do not know the ``correct'' way to define the requirements on the
 weak classifiers, nor has the notion of optimal boosting been explored in
 the multiclass setting.

 Straightforward extensions of the binary weak-learning condition to
 multiclass do not work. Requiring less error than random guessing on
 every distribution, as in the binary case, turns out to be too weak
 for boosting to be possible when there are more than two labels. On
 the other hand, requiring more than 50\% accuracy even when the
 number of labels is much larger than two is too stringent, and simple
 weak classifiers like decision stumps fail to meet this criterion,
 even though they often can be combined to produce highly accurate
 classifiers~\citep{FreundSc96}. The most common approaches so far have
 relied on reductions to binary classification~\citep{AllweinScSi00},
 but it is hardly clear that the weak-learning conditions implicitly
 assumed by such reductions are the most appropriate.

The purpose of a weak-learning condition is to clarify the goal of the
weak-learner, thus aiding in its design, while providing a specific
minimal guarantee on performance that can be exploited by a boosting
algorithm.
These considerations may significantly impact learning and
generalization because
 knowing the correct weak-learning conditions might allow the use of
 simpler weak classifiers, which in turn can help prevent
 overfitting. Furthermore, boosting algorithms that more efficiently
 and effectively minimize training error may prevent underfitting,
 which can also be important.

 In this \thesis{chapter}\notthesis{paper},
 we create a broad and general framework for studying
 multiclass boosting that formalizes the interaction between the
 boosting algorithm and the weak-learner.
Unlike much, but not all, of the previous work on multiclass boosting,
we focus specifically on the most natural, and perhaps weakest,
 case in which the weak classifiers are
 genuine classifiers in the sense of predicting a single multiclass
 label for each instance.
Our new framework allows us to
 express a range of weak-learning conditions, both new ones and most
 of the ones that had previously been assumed (often only implicitly).
Within this formalism, we can also now finally
 make precise what is meant by \emph{correct} weak-learning conditions
 that are neither too weak nor too strong.

 We focus particularly on a family of novel weak-learning conditions
 that have an especially appealing form: like the binary conditions,
 they require performance that is only slightly better than random
 guessing, though with respect to performance measures that are more
 general than ordinary classification error.  We introduce a whole
 family of such conditions since there are many ways of randomly
 guessing on more than two labels, a key difference between the binary
 and multiclass settings.  Although these conditions impose seemingly
 mild demands on the weak-learner, we show that each one of them is
 powerful enough to guarantee boostability,
 meaning that some combination of the weak classifiers has high
 accuracy.  And while no individual member of the family is necessary
 for boostability, we also show that the entire family taken together
 is necessary in the sense that for every boostable learning problem,
 there exists one member of the family that is satisfied.  Thus, we
 have identified a family of conditions which, as a whole, is
 necessary and sufficient for multiclass boosting.  Moreover, we can
 combine the entire family into a single weak-learning condition that
 is necessary and sufficient by taking a kind of union, or logical
 {\sc or}, of all the members.  This combined condition can also be
 expressed in our framework.

With this understanding, we are able to characterize previously
studied weak-learning conditions.  In particular, the condition
implicitly used by AdaBoost.MH~\citep{SchapireSi99}, which
is based on a one-against-all reduction to binary, turns out to be
strictly stronger than necessary for boostability.  This also applies
to AdaBoost.M1~\citep{FreundSc96}, the most direct generalization of
AdaBoost to multiclass, whose conditions can be shown to be equivalent
to those of AdaBoost.MH in our setting.  On the other hand, the
condition implicit to the SAMME
algorithm by ~\citet{ZhuZoRoHa09} is too weak in the sense that even when
the condition is satisfied,
no boosting algorithm can guarantee to drive down the
training error.  Finally, the condition implicit to
AdaBoost.MR~\citep{SchapireSi99,FreundSc96} (also called AdaBoost.M2)
turns out to be exactly necessary and sufficient for boostability.

Employing proper weak-learning conditions is important, but we also
need boosting algorithms that can exploit these conditions to
effectively drive down error. For a given weak-learning condition, the
boosting algorithm that drives down training error most efficiently in
our framework can be understood as the optimal strategy for playing a
certain two-player game. These games are non-trivial to
analyze. However, using the powerful machinery of drifting
games~\citep{FreundOp02,Schapire01}, we are able to compute
the optimal strategy for the games arising out of each weak-learning
condition in the family described above.
\old{
Compared to earlier work, our
optimality results hold more generally and also achieve tighter
bounds.
}
These optimal strategies
have a natural interpretation in terms of random walks, a phenomenon
that has been observed in other
settings~\citep{AbernethyBaRaTe08,Freund95}.

We also analyze the optimal boosting strategy when using the minimal
weak learning condition, and this poses additional challenges.
Firstly, the minimal weak learning condition has multiple natural
formulations --- e.g., as the union of all the conditions in the family
described above, or the formulation used in AdaBoost.MR --- and each
formulation leading to a different game specification.
A priori, it is not clear which game would lead to the best
strategy. 
We resolve this dilemma by proving that the optimal strategies arising out of
different formulations of the same weak learning condition lead to
algorithms that are essentially equally good, and therefore we are
free to choose whichever formulation leads to an easier analysis
without fear of suffering in performance.
We choose the union of conditions formulation, since it leads to
strategies that share the same interpretation in terms of random walks
as before.
However, even with this choice, the resulting games are hard to
analyze, and although we can explicitly compute the optimum
strategies in general, the computational complexity is usually
exponential. 
Nevertheless, we identify key situations under which efficient
computation is possible.  

The game-theoretic strategies are non-adaptive in that they presume
prior knowledge about the {\em edge}, that is, how much better than
random are the weak classifiers.
Algorithms that are adaptive, such as AdaBoost, are much more
practical because they do not require such prior information.
We show therefore how to derive an adaptive
boosting algorithm by modifying the game-theoretic strategy based
on the minimal condition.
This algorithm enjoys a number of theoretical guarantees.
Unlike some of the non-adaptive strategies, it is efficiently
computable, and since it is based on the minimal weak learning
condition, it makes minimal assumptions.
In fact, whenever presented with a boostable learning problem, this
algorithm can approach zero training error at an exponential rate.
More importantly, the algorithm is
effective even beyond the boostability framework.
In particular, we show empirical consistency, i.e.,
the algorithm always converges to the minimum of a certain exponential loss over the
training data, whether or not the dataset is boostable.
Furthermore, using the results in \citep{MukherjeeRuSc11} we can show
that this convergence occurs rapidly.

Our focus in this \thesis{chapter}\notthesis{paper} is only on minimizing
training error, which, 
for the algorithms we derive, provably decreases exponentially fast
with the number of rounds of boosting under boostability assumptions.
Such results can be used in
turn to derive bounds on the generalization error using standard
techniques that have been applied to other boosting
algorithms~\citep{SchapireFrBaLe98,FreundSc97,KoltchinskiiPa02}.
\old{
Consistency in the multiclass classification setting has been studied
by \citet{TewariBa07} and has been shown to be trickier than binary
classification consistency.
Nonetheless, by following the approach in \citep{BartlettTr07} for
showing consistency in the binary setting,
we are able to extend the empirical consistency guarantees to general
consistency guarantees in the multiclass setting: we show that
under certain conditions and with sufficient data, our adaptive
algorithm approaches the Bayes-optimum error on the \emph{test}
dataset.
}
 
We present experiments aimed at testing the
efficacy of the adaptive algorithm when working with a very weak
 weak-learner to check that the conditions we have identified are
indeed weaker than others that had previously been used.
We find that our new adaptive strategy
 achieves low test error compared to other multiclass boosting
 algorithms which usually heavily underfit.
This validates the potential practical benefit of a better theoretical
understanding of multiclass boosting.

{\bf Previous work.}
The first boosting algorithms were given
by \citet{Schapire90b} and \citet{Freund95}, followed by
their AdaBoost algorithm~\citep{FreundSc97}.
Multiclass boosting techniques include
~AdaBoost.M1 and AdaBoost.M2~\citep{FreundSc97},
as well as AdaBoost.MH and AdaBoost.MR~\citep{SchapireSi99}.
Other approaches include the work by \citet{EiblPf05,ZhuZoRoHa09}. 
There are also more general approaches that can be applied to boosting
including~\citep{AllweinScSi00,BeygelzimerLaRa09,DietterichBa95,HastieTi98}.
Two game-theoretic perspectives have been applied to boosting.
The first one~\citep{FreundSc96b,RatschWa05} views the weak-learning
condition as a minimax game, while drifting
games~\citep{Schapire01,Freund95} were 
designed to analyze the most efficient boosting algorithms. These
games have been further analyzed in the multiclass and continuous time
setting in~\citep{FreundOp02}.

\section{Framework}
\label{multi:prelim:sec}

We introduce some notation. Unless otherwise stated, matrices will be
denoted by bold capital letters like $\M$, and vectors by bold small
letters like $\vv$. Entries of a matrix and vector will be denoted as
$M(i,j)$ or $v(i)$, while $\M(i)$ will denote the $i$th row of a
matrix. Inner product of two vectors $\vu,\vv$ is denoted by
$\dotp{\vu}{\vv}$. The Frobenius inner product of two matrices
$\text{Tr}(\M\M')$ will be denoted by $\M\bullet\M'$, where $\M'$ is
the transpose of $\M$.
The indicator
function is denoted by $\1\enco{\cdot}$.
The set of all distributions over the set
$\set{1,\ldots,k}$ will be denoted by $\Delta\set{1,\ldots,k}$, and in
general, the set of all distributions over any set $S$ will be denoted
by $\Delta(S)$.

In multiclass classification, we want to predict the labels of
examples lying in some set $X$. We are provided
a training set of labeled examples $\set{(x_1,y_1),\ldots,(x_m,y_m)}$,
where each example $x_i\in X$ has a label $y_i$ in the set
$\set{1,\ldots,k}$. 

Boosting combines several mildly powerful predictors, called
\emph{weak classifiers}, to form a highly accurate combined
classifier, and has been previously applied for multiclass
classification. In this \thesis{chapter}\notthesis{paper},
we only allow weak classifier that
predict a single class for each example. This is appealing, since the
combined classifier has the same form, although it differs from what
has been used in much previous work. Later we will expand our
framework to include \emph{multilabel} weak classifiers, that may
predict multiple labels per example.

We adopt a game-theoretic view of boosting. A game is played between
two players, Booster and Weak-Learner, for a fixed number of rounds
$T$. With binary labels, Booster outputs a distribution in each round,
and Weak-Learner returns a weak classifier achieving more than $50\%$
accuracy on that distribution. The multiclass game is an extension of
the binary game. In particular, in each round $t$:
\begin{itemize}
\item
  Booster creates a cost-matrix $\C_t \in \R^{m\times k}$,
  specifying to Weak-Learner that the cost of classifying example
  $x_i$ as $l$ is $C_t(i,l)$. The cost-matrix may not be arbitrary, but
  should conform to certain restrictions as discussed below.
\item
  Weak-Learner returns some weak classifier $h_t\colon X\To
  \set{1,\ldots,k}$ from a fixed space $h_t\in\H$ so that the
  cost incurred is
  \[
  \C_t \bullet \mat{1}_{h_t} = \sum_{i=1}^m C_t(i,h_t(x_i)),
  \]
  is ``small enough'', according to some conditions discussed
  below. Here by $\vh$ we mean the $m\times k$ matrix whose $(i,j)$-th
  entry is $\1\enco{h(i)=j}$.
\item
  Booster computes a weight $\alpha_t$ for the current
  weak classifier based on how much cost was incurred in this round.
\end{itemize}

At the end, Booster predicts according to the weighted plurality vote
of the classifiers returned in each round:
\begin{equation}
  \label{multi:fdef:eqn}
  H(x) \eqdef \argmax_{l\in \set{1,\ldots,k}}f_T(x,l), \mbox{ where }
  f_T(x,l) \eqdef \sum_{t=1}^T\1 \enco{h_{t}(x) = l} \alpha_{t}.
\end{equation}
By carefully choosing the cost matrices in each round, Booster aims to
minimize the training error of the final classifer $H$, even when
Weak-Learner is adversarial. The restrictions on cost-matrices created
by Booster, and the maximum cost Weak-Learner can suffer in each
round, together define the \emph{weak-learning condition} being
used. For binary labels, the traditional weak-learning condition
states: for any non-negative weights $w(1),\ldots,w(m)$ on the
training set, the error of the weak classfier returned is at most
$(1/2 - \gamma/2)\sum_iw_i$.  Here $\gamma$ parametrizes the
condition. There are many ways to translate this condition into our
language. The one with fewest restrictions on the cost-matrices
requires labeling correctly should be less costly than labeling
incorrectly:
\[
\forall i: C(i,y_i) \leq C(i,\bar{y}_i)
\mbox{ (here $\bar{y}_i \neq y_i$ is the other binary label),}
\]
while the
restriction on the returned weak classifier $h$ requires less cost
than predicting randomly:
\[
\sum_i C(i,h(x_i)) \leq
\sum_i\enct{\enc{\frac{1}{2}-\frac{\gamma}{2}}C(i,\bar{y}_i) +
  \enc{\frac{1}{2}+\frac{\gamma}{2}}C(i,y_i)}.
\]
By the correspondence $w(i) = C(i,\bar{y}_i) - C(i,y_i)$, we may
verify the two conditions are the same.

We will rewrite this condition after making some simplifying
assumptions. Henceforth, without loss of generality, we assume that
the true label is always $1$.  Let $\Cbin \subseteq \R^{m\times 2}$
consist of matrices $\C$ which satisfy $C(i,1) \leq C(i,2)$. Further,
let $\Ubin_\gamma\in \R^{m\times 2}$ be the matrix whose each row is
$\enc{1/2+\gamma/2,1/2-\gamma/2}$. Then, Weak-Learner searching space
$\H$ satisfies the binary weak-learning condition if:
$
\forall \C\in\Cbin, \exists h\in\H: \C\bullet\enc{\vh-\Ubin_\gamma}
\leq \vzero.
$
There are two main benefits to this reformulation. With linear
homogeneous constraints, the mathematics is simplified, as will be
apparent later. More importantly, by varying the restrictions $\Cbin$
on the cost vectors and the matrix $\Ubin$, we can generate a vast
variety of weak-learning conditions for the multiclass setting $k\geq
2$ as we now show.

Let $\Csp\subseteq \R^{m\times k}$ and let $\B\in\R^{m\times
  k}$ be a matrix which we call the \emph{baseline}.
We say a weak classifier
space $\H$ satisfies the condition $\enc{\Csp,\B}$ if
\begin{eqnarray}
  \label{multi:wl:eqn}
  \forall \C \in \Csp, \exists h\in\H: & \C \bullet\enc{\vh - \B}
  \leq \vzero, &  \mbox{ i.e., }  \sum_{i=1}^m C(i,h(i)) \leq \sum_{i=1}^m
  \dotp{\C(i)}{\B(i)}.
\end{eqnarray}
In \eqref{multi:wl:eqn}, the variable matrix $\C$ specifies how costly each
misclassification is, while the baseline $\B$ specifies a weight for each
misclassification. The condition therefore states that a weak classifier
should not exceed the average cost when weighted according to baseline
$\B$.  This large class of weak-learning conditions captures many
previously used conditions, such as the ones used by
AdaBoost.M1~\citep{FreundSc96}, AdaBoost.MH~\citep{SchapireSi99} and
AdaBoost.MR~\citep{FreundSc96,SchapireSi99} (see below), as well as
novel conditions introduced in the next section.

By studying this vast class of weak-learning conditions, we hope to
find the one that will serve the main purpose of the boosting game:
finding a convex combination of weak classifiers that has zero
training error. For this to be possible, at the minimum the weak
classifiers should be sufficiently rich for such a perfect combination
to exist. Formally, a collection $\H$ of weak classifiers is
\emph{boostable} if it is eligible for boosting in the sense that
there exists a distribution $\vlambda$ on this space that linearly separates the
data: $\forall i:
\argmax_{l\in\set{1,\ldots,k}}\sum_{h\in\H}\lambda(h)\1\enco{h(x_i)=l}
= y_i$. The weak-learning condition plays two roles. It rejects spaces
that are not boostable, and provides an algorithmic means of searching
for the right combination. Ideally, the second factor will not cause
the weak-learning condition to impose additional restrictions on the
weak classifiers; in that case, the weak-learning condition is merely
a reformulation of being boostable that is more appropriate for
deriving an algorithm. In general, it could be \emph{too strong},
i.e. certain boostable spaces will fail to satisfy the conditions. Or
it could be \emph{too weak} i.e., non-boostable spaces might satisfy
such a condition. Booster strategies relying on either of these
conditions will fail to drive down error, the former due to
underfitting, and the latter due to overfitting. Later we will
describe conditions captured by our framework that avoid being too
weak or too strong. But before that, we show in the next section how
our flexible framework captures weak learning conditions that have
appeared previously in the literature.

\section{Old conditions}
\label{multi:old:sec}

\ignore{SAMME, M1, MH, MR, MH=M1}

In this section, we rewrite, in the language of our framework, the
weak learning conditions explicitly or implicitly employed in the
multiclass boosting algorithms SAMME~\citep{ZhuZoRoHa09},
AdaBoost.M1~\citep{FreundSc96}, and AdaBoost.MH and
AdaBoost.MR~\citep{SchapireSi99}.
This will be useful later on for comparing the strengths and
weaknesses of the various conditions.
We will end this section with a
curious equivalence between the conditions of AdaBoost.MH and
AdaBoost.M1. 

Recall that we have
assumed the correct label is 1 for every example.
Nevertheless, we
continue to use $y_i$ to denote the correct label in this
section.

\subsection{Old conditions in the new framework}
\label{multi:oldnew:sec}
Here we restate, in the language of our new framework, the weak
learning conditions of four algorithms that 
have earlier appeared in the literature.

\paragraph{SAMME.} The SAMME algorithm~\citep{ZhuZoRoHa09} requires
less error than random guessing on any distribution on the
examples. Formally, a space $\H$ satisfies the condition if there is a
$\gamma' > 0$ such that,
\begin{equation}
  \label{multi:SAMMEwl:eqn}
\forall d(1),\ldots,d(m)\geq 0, \exists h\in\H:
\sum_{i=1}^m d(i)\one\enco{h(x_i)\neq y_i} \leq (1-1/k - \gamma')\sum_{i=1}^md(i).
\end{equation}
Define a cost matrix $\C$ whose entries are given by 
\[
C(i,j) = 
\begin{cases}
  d(i) & \mbox{ if } j \neq y_i, \\
  0    & \mbox{ if } j = y_i.
  \end{cases}
  \]
  Then the left hand side of \eqref{multi:SAMMEwl:eqn} can be written
  as
 \[
 \sum_{i=1}^mC(i,h(x_i)) = \C \bullet \vh.
 \]
 Next let $\gamma = (1-1/k)\gamma'$ and define baseline $\Ugam$ to be
 the multiclass extension of $\Ubin$,
 \[
 U_\gamma(i,l) =
 \begin{cases}
   \frac{(1-\gamma)}{k} + \gamma & \mbox{ if } l = y_i, \\
   \frac{(1-\gamma)}{k} & \mbox { if } l \neq y_i.
   \end{cases}
 \]
 Then the right hand side of \eqref{multi:SAMMEwl:eqn} can be written
 as
\[
\sum_{i=1}^m\sum_{l\neq y_i}C(i,l)U_{\gamma}(i,l) = \C \bullet \Ugam,
\]
since $C(i,y_i) = 0$ for every example $i$.  Define $\Csam$ to be the
following collection of cost matrices:
\[
\Csam \eqdef \set{\C: C(i,l) =
  \begin{cases}
    0 & \mbox{ if } l = y_i, \\
    t_i & \mbox { if } l \neq y_i,
  \end{cases}
  \mbox { for non-negative } t_1, \ldots, t_m.
}
\]
Using the last two equations,
\eqref{multi:SAMMEwl:eqn} is equivalent to
\[
\forall \C\in\Csam, \exists h\in\H:
 \C \bullet \enc{\vh - \Ugam} \leq 0.
 \]
 Therefore, the weak-learning condition of SAMME is given by
 $(\Csam,\Ugam)$.

 \paragraph{AdaBoost.M1}
 Adaboost.M1~\citep{FreundSc97} measures the
 performance of weak classifiers using ordinary error.
 It requires $1/2 +
 \gamma/2$ accuracy with respect to any non-negative weights
 $d(1),\ldots,d(m)$ on the training set:
\begin{eqnarray}
  \label{multi:m1wlbasic:eqn}
  \sum_{i=1}^m d(i)\1\enco{h(x_i)\neq y_i} &\leq&
  \enc{1/2 - \gamma/2}\sum_{i=1}^m d(i), \\
   \mbox { i.e. } \sum_{i=1}^m d(i) \bb{h(x_i) \neq y_i} &\leq&
   -\gamma\sum_{i=1}^m d(i). \nonumber 
 \end{eqnarray}
 where $\bb{\cdot}$ is the $\pm 1$ indicator function, taking value $+1$
 when its argument is true, and $-1$ when false. Using the
 transformation
 \begin{equation}
   \label{multi:m1one:eqn}
C(i,l) = \bb{l\neq y_i} d(i)
\end{equation}
we may rewrite \eqref{multi:m1one:eqn} as
\begin{eqnarray}
  \label{multi:m1set:eqn}
  \lefteqn{\forall C\in\R^{m\times k} \mbox { satisfying }
     0 \leq -C(i,y_i) =
      C(i,l) \mbox { for } l\neq y_i},\\
    && \exists h\in \H:
    \sum_{i=1}^m C(i,h(x_i)) \leq
    \gamma\sum_{i=1}^m C(i,y_i)
  \nonumber\\
  \label{multi:m1wl:eqn}
  &\mbox{ i.e. }& \forall \C\in\Cmone, \exists h\in\H: \C  \bullet
  \enc{\vh - \Mone_\gamma} \leq 0, 
\end{eqnarray}
where $\Mone_\gamma(i,l) = \gamma \1\enco{l=y_i}$, and $\Cmone
\subseteq \R^{m\times k}$ consists of matrices satisfying
the constraints in \eqref{multi:m1set:eqn}.

\paragraph{AdaBoost.MH}
AdaBoost.MH \citep{SchapireSi99} is a popular multiclass boosting
algorithm that is based on the one-against-all reduction,
and was originally designed to use
weak-hypotheses that return a prediction for every example and every
label.
The implicit weak learning condition requires that for any matrix with
non-negative entries $d(i,l)$, the weak-hypothesis should achieve $1/2
+ \gamma$ accuracy
\begin{eqnarray}
  \label{multi:mhwlbasic:eqn}
  \sum_{i=1}^m
  \enct{\1\enco{h(x_i)\neq y_i}d(i,y_i) +
    \sum_{l\neq  y_i} \1\enco{h(x_i)=l} d(i,l)}
  &\leq&
  \enc{\frac{1}{2} - \frac{\gamma}{2}}
  \sum_{i=1}^m\sum_{l=1}^k d(i,l). \nonumber \\
&&
\end{eqnarray}
This can be rewritten as
\begin{eqnarray*}
  \lefteqn{\sum_{i=1}^m
    \enct{-\1\enco{h(x_i)= y_i}d(i,y_i) +
      \sum_{l\neq y_i} \1\enco{h(x_i)=l} d(i,l)}} \nonumber \\
  &\leq&
  \sum_{i=1}^m\enct{
    \enc{\frac{1}{2}-\frac{\gamma}{2}}\sum_{l\neq y_i}d(i,l) -
  \enc{\frac{1}{2} + \frac{\gamma}{2}}d(i,y_i)}.  \nonumber
\end{eqnarray*}
Using the mapping
\[
  C(i,l) =
  \begin{cases}
    d(i,l) & \mbox{ if } l\neq y_i \\
    -d(i,l) & \mbox { if } l = y_i,
  \end{cases}
\]
their weak-learning condition may be rewritten as follows
\begin{eqnarray}
  \label{multi:mhset:eqn}
  \lefteqn{\forall \C\in\R^{m\times k}
    \mbox{ satisfying }
    C(i,y_i)\leq 0, C(i,l) \geq 0
      \mbox { for  }
      l\neq y_i,} \\
  & & \exists h\in \H: \nonumber\\
  & &\sum_{i=1}^m C(i,h(x_i))\leq
  \sum_{i=1}^m
  \enct{\enc{\frac{1}{2} + \frac{\gamma}{2}}
    C(i,y_i) +
    \enc{\frac{1}{2}-\frac{\gamma}{2}}
    \sum_{l\neq y_i}C(i,l)}.
\end{eqnarray}
Defining $\Cmh$ to be the space of all cost matrices satisfying the
constraints in \eqref{multi:mhset:eqn}, the above condition is the
same as
\[
\forall \C\in\Cmh, \exists h\in\H: \C \bullet \enc{\vh - \MH_\gamma} \leq 0,
\]
where $\MH_\gamma(i,l) = (1/2 + \gamma\bb{l=y_i}/2)$.

\paragraph{AdaBoost.MR}
AdaBoost.MR
\citep{SchapireSi99} is based on the all-pairs multiclass to binary
reduction.
Like AdaBoost.MH, it was originally designed to use
weak-hypotheses that return a prediction for every example and every
label.
The weak learning condition for AdaBoost.MR requires that for any non-negative
cost-vectors $\{d(i,l)\}_{l\neq y_i}$, the weak-hypothesis returned
should satisfy the following:
\begin{eqnarray*}
  \sum_{i=1}^m \sum_{l\neq y_i} \enc{\1\enco{h(x_i) = l} -
    \1\enco{h(x_i)=y_i}} d(i,l) &\leq& 
  -\gamma \sum_{i=1}^m\sum_{l\neq y_i} d(i,l) \\
  \mbox { i.e. }
  \sum_{i=1}^m \enct{-\1\enco{h(x_i)=y_i}
    \sum_{l\neq y_i} d(i,l) +
    \sum_{l\neq y_i} \1\enco{h(x_i) = l} d(i,l)} 
  &\leq&
  -\gamma \sum_{i=1}^m\sum_{l\neq y_i} d(i,l).
\end{eqnarray*}
Substituting
\[
C(i,l) = 
\begin{cases}
  d(i,l) & l \neq y_i \\
  -\sum_{l\neq y_i} d(i,l) & l = y_i,
\end{cases}
\]
we may rewrite AdaBoost.MR's weak-learning condition as
\begin{eqnarray}
  \label{multi:mrset:eqn}
  \lefteqn{\forall \C\in\R^{m\times k} \mbox{ satisfying } C(i,l)
      \geq 0 \mbox { for 
      } l\neq y_i, C(i,y_i) = -\sum_{l\neq y_i} C(i,l),} \\
    & & \exists h\in \H: \sum_{i=1}^m C(i,h(x_i))\leq
    -\frac{\gamma}{2}
    \sum_{i=1}^m \enct{ 
    -C(i,y_i) + \sum_{l\neq y_i} C(i,l)} \nonumber.
\end{eqnarray}
Defining $\Cmr$ to be the collection of cost matrices satisfying the
constraints in \eqref{multi:mrset:eqn}, the above condition is the
same as
\[
\forall \C\in\Cmr, \exists h\in\H: \C \bullet \enc{\vh - \MR_\gamma}
\leq 0,
\]
where $\MR_\gamma(i,l) = \bb{l=y_i}\gamma/2$.

%
%

\subsection{A curious equivalence}

We show that the weak learning conditions of AdaBoost.MH and
AdaBoost.M1 are identical in our framework. This is surprising because
the original motivations behind these algorithms were completely
different. AdaBoost.M1 is a direct extension of
binary AdaBoost to the multiclass setting, whereas
AdaBoost.MH is based on the one-against-all
multiclass to binary reduction. This equivalence is a sort of
degeneracy, and arises because the weak classifiers being used predict
single labels per example. With multilabel weak classifiers, for which
AdaBoost.MH was originally designed, the equivalence no longer holds.

The proofs in this and later sections will make use of the following
minimax result, that is a weaker version of Corollary 37.3.2 of
\citep{Rockafellar70}.

\begin{theorem}(Minimax Theorem)
  \label{multi:minmax:thm}
  Let $C,D$ be non-empty closed convex subsets of $\R^m,\R^n$
  respectively, and let $K$ be a\ignore{finite concave-convex} linear  
  function on $C\times D$. If either $C$ or $D$ is bounded, then
  \[
  \min_{v\in D} \max_{u\in C} K(u,v) = \max_{u\in C} \min_{v\in D}
  K(u,v).
  \]
\end{theorem}
 \begin{lemma}
   A weak classifier space $\H$ satisfies $(\Cmone,\Mone_\gamma)$
   if and only if it satisfies $(\Cmh, \MH_\gamma)$.
 \end{lemma}
 \begin{proof}
   We will refer to $\mow$ by M1 and $\mhw$ by MH for brevity.
   The proof is in three steps.

   \emph{Step (i)}: If $\H$ satisfies MH, then it also satisfies
   M1.
   This follows since any constraint \eqref{multi:m1wlbasic:eqn}
   imposed by M1 on $\H$ can be reproduced by MH by plugging
   the following values of $d(i,l)$ in \eqref{multi:mhwlbasic:eqn}
  \[
  d(i,l) = 
  \begin{cases}
    d(i) & \mbox { if } l = y_i \\
    0   & \mbox { if } l \neq y_i.
  \end{cases}
  \]

  \emph{Step (ii)}: If $\H$ satisfies M1, then there is a convex
  combination $\mat{H}_{\vlams}$ of the matrices $\vh \in \H$,
  defined as
  \[
  \mat{H}_{\vlams} \eqdef \sum_{h\in\H} \lambda^*(h) \vh,
  \]
  such that
  \begin{equation}
    \label{multi:m1mh:eqn}
  \forall i: \enc{\mat{H}_{\vlams} - \MH_\gamma}(i,l) 
  \begin{cases}
    \geq 0  & \mbox { if } l = y_i \\
    \leq 0 & \mbox { if } l \neq y_i.
  \end{cases}
  \end{equation}
  Indeed, Theorem~\ref{multi:minmax:thm} yields
  \begin{equation}
    \label{multi:minmaxineq:eqn}
   \min_{\vlambda \in \Delta\enc{\H}} \max_{\C \in \Cmone}
  \C \bullet \enc{\mat{H}_{\vlambda} - \Mone_{\gamma}}
  = \max_{\C \in \Cmone} \min_{h\in \H} \C \bullet \enc{\vh -
    \Mone_{\gamma}} \leq 0,
  \end{equation}
  where the inequality is a restatement of our assumption that $\H$
  satisfies M1.  If $\vlams$ is a minimizer of the minmax
  expression, then $\mat{H}_{\vlams}$ must satisfy
  \begin{equation}
    \label{multi:half:eqn}
  \forall i: \mat{H}_{\vlams}(i,l) 
  \begin{cases}
    \geq \frac{1}{2} + \frac{\gamma}{2} & \mbox { if } l = y_i \\
    \leq \frac{1}{2} - \frac{\gamma}{2} & \mbox { if } l \neq y_i,
  \end{cases}
  \end{equation}
  or else some choice of $\C\in\Cmone$ can cause $\C\bullet
  \enc{\mat{H}_{\vlams} - \Mone}$ to exceed 0.
  In particular, if
  $\mat{H}_{\vlams}(i_0,l) < 1/2 + \gamma/2$, then
  \[
  \enc{\mat{H}_{\vlams} - \Mone_\gamma}(i_0,y_{i_0}) <
  \sum_{l\neq y_{i_0}}  \enc{\mat{H}_{\vlams} - \Mone_\gamma}(i_0,l).
  \]
  Now, if we choose $\C\in\Cmone$ as
  \[
  C(i,l) =
  \begin{cases}
    0 & \mbox { if } i \neq i_0 \\
    1 & \mbox { if } i=i_0, l\neq y_{i_0}\\
    -1 & \mbox { if } i=i_0, l= y_{i_0},
  \end{cases}
  \]
  then,
  \[
  \C\bullet \enc{\mat{H}_{\vlams} - \Mone_\gamma} =
  - \enc{\mat{H}_{\vlams} - \Mone_\gamma}(i_0,y_{i_0}) +
  \sum_{l\neq y_{i_0}}  \enc{\mat{H}_{\vlams} - \Mone_\gamma}(i_0,l)
  > 0,
  \]
  contradicting the inequality in
  \eqref{multi:minmaxineq:eqn}.
  Therefore \eqref{multi:half:eqn} holds.
  Eqn. \eqref{multi:m1mh:eqn}, and thus Step (ii), now follows by
  observing that $\MH_\gamma$, by definition, satisfies
  \[
  \forall i: \MH_\gamma(i,l) = 
  \begin{cases}
    \frac{1}{2} + \frac{\gamma}{2} & \mbox { if } l = y_i \\
    \frac{1}{2} - \frac{\gamma}{2} & \mbox { if } l \neq y_i.
  \end{cases}
  \]

  \emph{Step (iii)} If there is some convex combination $\H_{\vlams}$
  satisfying \eqref{multi:m1mh:eqn}, then $\H$ satisfies MH.
  Recall that $\MH$ consists of entries that are non-positive on the
  correct labels and non-negative for incorrect labels.
  Therefore, \eqref{multi:m1mh:eqn} implies
  \[
  0 \geq \max_{\C \in \Cmh}
  \C \bullet \enc{\mat{H}_{\vlams} - \MH_{\gamma}}
  \geq
  \min_{\vlambda \in \Delta\enc{\H}} \max_{\C \in \Cmh}
  \C \bullet \enc{\mat{H}_{\vlambda} - \MH_{\gamma}}.
  \]
  On the other hand, using Theorem~\ref{multi:minmax:thm} we have
  \[
  \min_{\vlambda \in \Delta\enc{\H}} \max_{\C \in \Cmh}
  \C \bullet \enc{\mat{H}_{\vlambda} - \MH_{\gamma}}
  =
  \max_{\C \in \Cmh} \min_{h\in \H} \C \bullet \enc{\vh - \MH_{\gamma}}.
  \]
  Combining the two, we get
  \[
  0 \geq \max_{\C \in \Cmh} \min_{h\in \H} \C \bullet \enc{\vh - \MH_{\gamma}},
  \]
  which is the same as saying that $\H$ satisfies MH's condition.

  Steps (ii) and (iii) together imply that if $\H$ satisfies M1, then
  it also satisfies MH.
  Along with Step (i), this concludes the proof.
\end{proof}

\section{Necessary and sufficient weak-learning conditions}
\label{multi:necsuf:sec}

The binary weak-learning condition has an appealing form: for any
distribution over the examples, the weak classifier needs to achieve
error not greater than that of a random player who guesses the correct
answer with probability $1/2+\gamma/2$. Further, this is the weakest
condition under which boosting is possible as follows from a
game-theoretic perspective~\citep{FreundSc96b,RatschWa05} . Multiclass
weak-learning conditions with similar properties are missing in the
literature. In this section we show how our framework captures such
conditions.

\subsection{Edge-over-random conditions}
\label{multi:eor:sec}

In the multiclass setting, we model a random player as a baseline
predictor $\B\in\R^{m\times k}$ whose rows are distributions over the
labels, $\B(i)\in\Delta\set{1,\ldots,k}$. The prediction on example
$i$ is a sample from $\B(i)$. We only consider the space of
\emph{edge-over-random} baselines $\Bgam \subseteq \R^{m\times k}$ who
have a faint clue about the correct answer.  More precisely, any
baseline $\B\in\Bgam$ in this space is $\gamma$ more likely to predict
the correct label than an incorrect one on every example $i$: $\forall
l\neq 1, B(i,1) \geq B(i,l) + \gamma$, with equality holding for some
$l$, i.e.:
\[
B(i,1) = \max\set{B(i,l)+\gamma: l\neq 1}.
\]
Notice that the edge-over-random baselines are different from the
baselines used by earlier weak learning conditions discussed in the
previous section.

When $k=2$, the space $\Bgam$ consists of the unique player
$\Ubin_\gamma$, and the binary weak-learning condition is given by
$(\Cbin,\Ubin_\gamma)$. The new conditions generalize this to
$k>2$. In particular, define $\Csig$ to be the multiclass extension of
$\Cbin$: any cost-matrix in $\Csig$ should put the least cost on the
correct label, i.e., the rows of the cost-matrices should come from
the set $\set{\vc\in\R^k: \forall l, c(1) \leq c(l) }$. Then, for
every baseline $\B\in\Bgam$, we introduce the condition $(\Csig,\B)$,
which we call an \emph{edge-over-random} weak-learning
condition. Since $\C\bullet\B$ is the expected cost of the
edge-over-random baseline $\B$ on matrix $\C$, the constraints
\eqref{multi:wl:eqn} imposed by the new condition essentially require better
than random performance.

Also recall that we have assumed that the true label $y_i$ of example
$i$ in our training set is always $1$. Nevertheless, we may
occasionally continue to refer to the true labels as $y_i$.

We now present the central results of this section. The seemingly
mild edge-over-random conditions guarantee boostability, meaning
weak classifiers that satisfy any one such condition can be combined
to form a highly accurate combined classifier. 
\begin{theorem}[Sufficiency]
  \label{multi:suf:thm}
  If a weak classifier space $\H$ satisfies a weak-learning condition
  $(\Csig,\B)$, for some $\B\in\Bgam$, then $\H$ is boostable.
\end{theorem}
\begin{proof}
  The proof is in the spirit of the ones in~\citep{FreundSc96b}.
  Applying Theorem~\ref{multi:minmax:thm} yields
 \[
 0 \geq
 \displaystyle \max_{\C\in\Csig}
 \displaystyle \min_{h\in\H}
 \C \bullet \enc{\vh - \B}
 =
 \displaystyle \min_{\vlambda\in\Delta(\H)}
 \displaystyle \max_{\C\in\Csig}
 \C \bullet \enc{\mat{H}_{\vlambda} - \B},
 \]
 where the first inequality follows from the definition
 \eqref{multi:wl:eqn} of the weak-learning condition. Let $\vlams$ be
 a minimizer of the min-max expression. Unless the first entry of
 each row of $ \enc{\mat{H}_{\vlams} - \B}$ is the largest, the right
 hand side of the min-max expression can be made arbitrarily large by
 choosing $\C\in\Csig$ appropriately. For example, if in some row $i$,
 the $j_0\th$ element is strictly larger than the first element, by
 choosing
 \[
 C(i,j) =
 \begin{cases}
   -1  & \mbox { if } j=1 \\
   1  & \mbox { if } j=j_0 \\
   0  & \mbox { otherwise},
 \end{cases}
 \]
 we get a matrix in $\Csig$ which causes $\C\bullet\enc{\mat{H}_{\vlams} -
   \B}$ to be equal to $C(i,j_0) - C(i,1) > 0$, an impossibility by
 the first inequality.

 Therefore, the convex combination of the weak classifiers, obtained
 by choosing each weak classifier with weight given by $\vlams$,
 perfectly classifies the training data, in fact with a margin
 $\gamma$.
 \end{proof}
 On the other hand, the family of such conditions, taken as a whole, is
 necessary for boostability in the sense that every eligible space of
 weak classifiers satisfies some edge-over-random condition.
\begin{theorem}[Relaxed necessity]
  \label{multi:nec:thm}
  For every boostable weak classifier space $\H$, there exists a
  $\gamma>0$ and $\B\in\Bgam$ such that $\H$ satisfies the
  weak-learning condition $(\Csig,\B)$.
\end{theorem}
\begin{proof}
The proof shows existence through non-constructive averaging
arguments. We will reuse notation from the proof of Theorem~\ref{multi:suf:thm}
above. $\H$ is boostable implies there exists some distribution
$\vlams\in\Delta(\H)$ such that
\[
\forall j\neq 1, i: \mat{H}_{\vlams} (i,1) - \mat{H}_{\vlams} (i,j) > 0.
\]
Let $\gamma > 0$ be the minimum of the above expression over all
possible $(i,j)$, and let $\B = \mat{H}_{\vlams}$. Then $\B\in\Bgam$, and
\[
 \displaystyle \max_{\C\in\Csig}
 \displaystyle \min_{h\in\H}
 \C \bullet \enc{\vh - \B}
 \leq
 \displaystyle \min_{\vlambda\in\Delta(\H)}
 \displaystyle \max_{\C\in\Csig}
 \C \bullet \enc{\mat{H}_{\vlambda} - \B}
 \leq
 \displaystyle \max_{\C\in\Csig}
 \C \bullet \enc{\mat{H}_{\vlams} - \B}
 = 0,
 \]
 where the equality follows since by definition $\mat{H}_{\vlams} - \B =
 \mathbf{0}$. The max-min expression is at most zero is another way of
 saying that $\H$ satisfies the weak-learning condition
 $(\Csig,\B)$ as in \eqref{multi:wl:eqn}.
\end{proof}
Theorem~\ref{multi:nec:thm} states that any boostable weak classifier
space will satisfy some condition in our family, but it does not help
us choose the right condition. Experiments in
Section~\ref{multi:expts:section} suggest $\enc{\Csig,\Ugam}$ is effective
with very simple weak-learners compared to popular boosting
algorithms. (Recall $\Ugam\in\Bgam$ is the edge-over-random baseline
closest to uniform; it has weight $(1-\gamma)/k$ on incorrect labels
and $(1-\gamma)/k + \gamma$ on the correct label.) However, there are
theoretical examples showing each condition in our family is too
strong.
\begin{theorem}
  \label{multi:eorstrong:thm}
  For any $\B\in\Bgam$, there exists a boostable space $\H$ that fails
  to satisfy the condition $(\Csig,\B)$.
\end{theorem}
\begin{proof}
 We provide, for any
 $\gamma > 0$ and edge-over-random baseline $\B\in\Bgam$, a dataset and
 weak classifier space that is boostable but fails to satisfy the
 condition $(\Csig,\B)$.
 
 Pick $\gamma' = \gamma/k$ and set $m>1/\gamma'$ so that $\floor{m(1/2
   + \gamma')} > m/2$.  Our 
 dataset will have $m$ labeled examples $\set{(0,y_0), \ldots,
   (m-1,y_{m-1})}$, and $m$ weak classifiers. We want the following
 symmetries in our weak classifiers:
\begin{itemize}
\item Each weak classifier correctly classifies $\floor{m(1/2 + \gamma')}$
  examples and misclassifies the rest.
\item On each example, $\floor{m(1/2 + \gamma')}$  weak classifiers
  predict correctly.
\end{itemize}
Note the second property implies boostability, since the uniform
convex combination of all the weak classifiers is a perfect
predictor.

The two properties can be satisfied by the following design. A window
is a contiguous sequence of examples that may wrap around; for example
\[
\set{i, (i+1)\mod m, \ldots, (i+k)\mod m}
\]
is a window containing $k$
elements, which may wrap around if $i+k \geq m$.  For each window of
length $\floor{m(1/2 + \gamma')}$ create a hypothesis that correctly
classifies within the window, and misclassifies outside. This
weak-hypothesis space has size $m$, and has the required properties.

We still have flexibility as to how the misclassifications occur, and
which cost-matrix to use, which brings us to the next two choices:
\begin{itemize}
\item Whenever a hypothesis misclassifies on example $i$, it predicts label
  \begin{equation}
    \label{multi:yhat:eqn}
  \hat{y}_i \eqdef \argmin \set{B(i,l) : l \neq y_i}.
  \end{equation}
\item A cost-matrix is chosen so that the cost of predicting
  $\hat{y}_i$ on example $i$ is 1, but for any other prediction the
  cost is zero. Observe this cost-matrix belongs to $\Csig$.
\end{itemize}
Therefore, every time a weak classifier predicts incorrectly, it also
suffers cost 1. Since each weak classifier predicts correctly only
within a window of length $\floor{m(1/2+\gamma')}$, it suffers cost
$\ceil{m(1/2 - \gamma')}$. On the other hand, by the choice of
$\hat{y}_i$ in \eqref{multi:yhat:eqn},
\begin{eqnarray*}
B(i,\hat{y}_i)
&=& \min\set{B(i,1)-\gamma,B(i,2),\ldots,B(i,k)}\\
&\leq&
\frac{1}{k}\enct{B(i,1)-\gamma + B(i,2) + B(i,3) + \ldots + B(i,k)}\\
&=& 1/k-\gamma/k.
\end{eqnarray*}
So the cost of $\B$ on the
chosen cost-matrix is at most $m(1/k-\gamma/k)$, which is less than the cost
$\ceil{m(1/2 - \gamma')} \geq m(1/2-\gamma/k)$ of any weak classifier
whenever the number 
of labels $k$ is more than two.
Hence our boostable space of weak
classifiers fails to satisfy $(\Csig,\B)$. 
\end{proof}
Theorems~\ref{multi:nec:thm} and \ref{multi:eorstrong:thm} can be
interpreted as follows. While a boostable space will satisfy
\emph{some} edge-over-random condition, without further information
about the dataset it is not possible to know \emph{which} particular
condition will be satisfied. The kind of prior knowledge required to
make this guess correctly is provided by
Theorem~\ref{multi:suf:thm}: the appropriate weak learning condition
is determined by the distribution of votes on the labels for each
example that a target weak classifier combination might be able to
get. Even with domain expertise, such knowledge may or may not be
obtainable in practice before running boosting. We therefore need
conditions that assume less.

\subsection{The minimal weak learning condition}
\label{multi:minwl:sec}
A perhaps extreme way of weakening the condition is by requiring the
performance on a cost matrix to be competitive not with a {\em
  fixed\/} baseline $\B\in\Bgam$, but with the {\em worst} of them:
\begin{equation}
  \label{multi:minwl:eqn}
\forall\C\in\mathcal{\Csig},\exists h\in\H: \C\bullet\vh \leq
\max_{\B\in\Bgam} \C\bullet\B.
\end{equation}
Condition \eqref{multi:minwl:eqn} states that during the course of the same
boosting game, Weak-Learner may choose to beat {\em any}
edge-over-random baseline $\B\in\Bgam$, possibly a different one for
every round and every cost-matrix.  This may superficially seem much
too weak.  On the contrary, this condition turns out to be equivalent
to boostability.  In other words, according to our criterion, it is
neither too weak nor too strong as a weak-learning condition.
However, unlike the edge-over-random conditions, it also turns out to
be more difficult to work with algorithmically.

Furthermore, this condition can be shown to be equivalent to the one
used by AdaBoost.MR~\citep{SchapireSi99,FreundSc96}. This is perhaps
remarkable since the latter is based on the apparently completely
unrelated all-pairs multiclass to binary reduction. In
Section~\ref{multi:old:sec} we saw that the MR condition is given by
$(\Cmr,\MR_\gamma)$, where $\Cmr$ consists of cost-matrices that put
non-negative costs on incorrect labels and whose rows sum up to zero,
while $\MR_\gamma\in\R^{m\times k}$ is the matrix that has $\gamma$ on
the first column and $-\gamma$ on all other
columns. Further, the MR condition, and hence
\eqref{multi:minwl:eqn}, can be shown to be neither too weak nor too
strong.
\begin{theorem}[MR]
\label{multi:mrminwl:thm}
  A weak classifier space $\H$ satisfies AdaBoost.MR's weak-learning
  condition $(\Cmr,\MR_\gamma)$ if and only if it satisfies
  \eqref{multi:minwl:eqn}. Moreover, this condition is equivalent to being
  boostable.
\end{theorem}
\begin{proof}
We will show the following three conditions are equivalent:
\begin{enumerate}[(A)]
\item $\H$  is boostable
\item $\exists \gamma > 0 \mbox{ such that } 
   \forall \C \in \Csig, \exists
  h\in \H: \C\bullet\vh \leq \displaystyle \max_{\B\in\Bgam} \C\bullet \B$
\item $\exists \gamma > 0 \mbox{ such that }
  \forall \C \in \Cmr, \exists
  h\in \H: \C\bullet\vh \leq \C\bullet \MR$.
\end{enumerate}

We will show (A) implies (B), (B) implies (C), and (C) implies (A) to
achieve the above.

{\it (A) implies (B)}: Immediate from Theorem~2.

{\it (B) implies (C)}: Suppose (B) is satisfied with $2\gamma$. We
will show that this implies $\H$ satisfies $(\Cmr,
\MR_\gamma)$. Notice $\Cmr \subset \Csig$. Therefore it suffices to
show that
\[
\forall \C\in \Cmr, \B\in \cB_{2\gamma}: \C\bullet \enc{\B -
  \MR_\gamma} \leq 0.
\]
Notice that $\B\in\cB_{2\gamma}$ implies $\B' = \B - \MR_\gamma$ is a
matrix whose largest entry in each row is in the first column of that
row. Then, for any $\C\in\Cmr$, $\C\bullet \B'$ can be written as
\[
\C\bullet \B' = \sum_{i=1}^{m}\sum_{j=2}^k C(i,j) \enc{B'(i,j) - B'(i,1)}.
\]
Since $C(i,j) \geq 0$ for $j>1$, and $B'(i,j) - B'(i,1) \leq 0$, we
have our result.

{\it (C) implies (A)}:  Applying Theorem~\ref{multi:minmax:thm}
  \[
  0 \geq \max_{\C \in \Cmr} \min_{h \in \H} \C \bullet \enc{ \vh -
    \MR_\gamma} = \min_{\vlambda \in \Delta(\H)} \max_{\C \in \Cmr} \C
  \bullet \enc{ \mat{H}_{\vlambda} - \MR_\gamma}.
  \]
  For any $i_0$ and $l_0\neq 1$, the
  following cost-matrix $\C$ satisfies $\C \in \Cmr$,
  \[
  C(i,l) = 
  \begin{cases}
    0  & \mbox { if } i \neq i_0 \mbox { or } l \not \in \set{1,l_0}\\
    1 & \mbox { if } i = i_0, l = l_0 \\
    -1 & \mbox { if } i = i_0, l = 1.
  \end{cases}
  \]
  Let $\vlambda$ belong to the $\argmin$ of the $\min\max$
  expression. Then $\C \bullet \enc{ \mat{H}_{\vlambda} -
    \MR_\gamma} \leq 0$ implies $\mat{H}_{\vlambda}(i_0,1) -
  \mat{H}_{\vlambda}(i_0,l_0) \geq 2\gamma$. Since this is true for all $i_0$
  and $l_0\neq 1$, we conclude that the $(\Cmr, \MR_\gamma)$ condition
  implies boostability.

  This concludes the proof of equivalence.
\end{proof}
  Next, we illustrate the strengths of our \ignore{random-over-edge}
  minimal weak-learning condition through concrete comparisons with previous
algorithms.

\paragraph{Comparison with SAMME.}
The SAMME algorithm of \citet{ZhuZoRoHa09} requires the weak classifiers
to achieve less error than uniform random guessing for multiple
labels; in our language, their weak-learning condition is
$(\Csam,\Ugam)$, as shown in Section~\ref{multi:old:sec}, where
$\Csam$ consists of cost matrices whose rows are of the form
$(0,t,t,\ldots)$ for some non-negative $t$.
As is well-known, this
condition is not sufficient for boosting to be possible. In
particular, consider the dataset $\set{(a,1),(b,2)}$ with $k=3,m=2$,
and a weak classifier space consisting of $h_1,h_2$ which always
predict $1,2$, respectively (Figure~\ref{multi:samme:fig}).
\begin{figure}
\begin{center}
  \begin{tabular}{c|cc}
      & $h_1$ & $h_2$ \\
    \hline
    $a$ & $1$  & $2$ \\
    $b$ & $1$  & $2$ \\
  \end{tabular}
\end{center}
\caption[Dataset demonstrating SAMME's condition is too weak]{A weak
  classifier space which satisfies SAMME's weak learning condition but
  is not boostable.} 
\label{multi:samme:fig}
\end{figure}
Since neither classifier distinguishes between $a,b$ we cannot achieve
perfect accuracy by combining them in any way. Yet, due to the
constraints on the cost-matrix, one of $h_1,h_2$ will always manage
non-positive cost while random always suffers positive cost. On the
other hand our weak-learning condition allows the Booster to choose
far richer cost matrices. In particular, when the cost matrix
$\C\in\Csig$ is given by
\begin{center}
  \begin{tabular}{l|rrr}
    & $1$ & $2$ & $3$ \\
    \hline
  $a$ & $-1$ & $+1$ & $0$ \\
  $b$ & $+1$ & $-1$ & $0$,
\end{tabular}
\end{center}
both classifiers in the above example suffer more loss than the random
player $\Ugam$, and fail to satisfy our condition.

\paragraph{Comparison with AdaBoost.MH.}
AdaBoost.MH~\citep{SchapireSi99} was designed for use with weak
hypotheses that on each example return a prediction for every label.
When used in our framework, where the weak classifiers return only a
single multiclass prediction per example,
the implicit demands made by AdaBoost.MH on the weak classifier space
turn out to be too strong.
To demonstrate this, we construct a classifier space that satisfies
the condition $(\Csig,\Ugam)$ in our family, but cannot satisfy
AdaBoost.MH's weak-learning condition.
Note that this does not imply that the conditions are too strong when
used with more powerful weak classifiers that return multilabel
multiclass predictions.

Consider a space $\H$ that has, for every $(1/k+\gamma)m$ element
subset of the examples, a classifier that predicts correctly on
exactly those elements. The expected loss of a randomly chosen
classifier from this space is the same as that of the random player
$\Ugam$. Hence $\H$ satisfies this weak-learning condition. On the
other hand, it was shown in Section~\ref{multi:old:sec} that
AdaBoost.MH's weak-learning condition is the pair $(\Cmh,\MH_\gamma)$,
where $\Cmh$ consists of cost matrices with non-negative entries on
incorrect labels and non-positive entries on real labels, 
and where each row of the matrix $\MH_\gamma$ is the vector
$(1/2+\gamma/2,1/2-\gamma/2,\ldots,1/2-\gamma/2)$. A quick calculation
shows that for any $h\in\H$, and $\C\in\Cmh$ with $-1$ in the first
column and zeroes elsewhere, $\C\bullet\enc{\vh - \MH_\gamma} = 1/2 -
1/k$.  This is positive when $k>2$, so that $\H$ fails to satisfy
AdaBoost.MH's condition.

We have seen how our framework allows us to capture the strengths and
weaknesses of old conditions, describe a whole new family of
conditions and also identify the condition making minimal assumptions.
In the next few sections, we show how to design boosting algorithms
that employ these new conditions and enjoy strong theoretical
guarantees.

\section{Algorithms}
\label{multi:algo:sec}

In this section we devise algorithms by analyzing the boosting games
that employ weak-learning conditions in our framework. We compute
the optimum Booster strategy against a completely adversarial
Weak-Learner, which here is permitted to choose weak classifiers
without restriction, i.e. the entire space $\Hall$ of all possible
functions mapping examples to labels. By modeling Weak-Learner
adversarially, we make absolutely no assumptions on the algorithm it
might use. Hence, error guarantees enjoyed in this situation will be
universally applicable. Our algorithms are derived from the very
general drifting games framework~\citep{Schapire01} for solving
boosting games, which in turn was inspired by Freund's Boost-by-majority
algorithm~\citep{Freund95}, which we review next.

\paragraph{The OS Algorithm.}
Fix the number of rounds $T$ and a
weak-learning condition $(\Csp,\B)$.
We will only consider conditions that are not \emph{vacuous}, i.e., at
least some classifier space satisfies the condition, or equivalently,
the space $\Hall$ satisfies $(\Csp,\B)$.
Additionally, we assume the
constraints placed by $\Csp$ are on individual rows. In other
words, there is some subset $\Crow \subseteq \R^k$ of all possible
rows, such that a cost matrix $\C$ belongs to the collection
$\Csp$ if and only if each of its rows belongs to this subset:
\begin{equation}
  \label{multi:crow:eqn}
\C\in\Csp \iff
\forall i: \C(i) \in \Crow.
\end{equation}
Further, we assume
$\Crow$ forms a convex cone i.e $\vc,\vc' \in \Crow$ implies $t\vc + t'\vc' \in
\Crow$ for any non-negative $t,t'$. This also implies that $\Csp$ is a convex
cone.
This is a very natural restriction, and is satisfied by the space $\C$
used by the weak learning conditions
of AdaBoost.MH, AdaBoost.M1, 
AdaBoost.MR, SAMME as well as every edge-over-random condition.
\footnote{All our results hold under the weaker restriction on
  the space $\Csp$, where the set of possible cost vectors $\Crow$ for
  a row $i$ could depend on $i$.
  For simplicity of exposition, we stick to the more restrictive
  assumption that $\Crow$ is common across all rows.}
For simplicity of presentation we fix the weights
$\alpha_t=1$ in each round. With $\f_T$ defined as in
\eqref{multi:fdef:eqn}, whether the final hypotheses output by Booster
makes a prediction error on an example $i$ is decided by whether an
incorrect label received the maximum number of votes, $f_T(i,1) \leq
\max_{l=2}^kf_T(i,l)$. Therefore, the optimum Booster payoff can be written as
\begin{equation}
  \label{multi:payoff:eqn}
\min_{\C_1\in \Csp} \max_{\substack{h_1\in\Hall:\\
    \C_1\bullet\enc{\mat{1}_{h_1} - \B}\leq \vzero}} \ldots
\min_{\C_T\in\Csp}\max_{\substack{h_T\in\Hall:\\
    \C_T\bullet\enc{\mat{1}_{h_T} - \B} \leq \vzero}}
\frac{1}{m}\sum_{i=1}^m
\Lzero(f_T(x_i,1),\ldots,f_T(x_i,k)).
\end{equation}
where the function $\Lzero:\R^k\to\R$ encodes 0-1 error 
\begin{equation}
  \label{multi:zero_one:eqn}
  \Lzero(\vs) =
  \1\enco{s(1) \leq \max_{l>1} s(l)}.
\end{equation}
In general, we will also consider
other loss functions $L:\R^k\to\R$ such as exponential loss, hinge
loss, etc. that upper-bound error and are \emph{proper}: i.e. $L(\vs)$
is increasing 
in the weight of the correct label $s(1)$, and decreasing in the
weights of the incorrect labels $s(l),l\neq 1$.

Directly analyzing the optimal payoff is hard. However,
\cite{Schapire01} observed that the payoffs can be very well
approximated by certain potential functions.  Indeed, for any
$\vb\in\R^k$ define the \emph{potential function} $\phi^{\vb}_t:\R^k
\To \R$ by the following recurrence:
\begin{eqnarray}
  \phi^{\vb}_0 &=& L \nonumber \\
    \label{multi:dgrec:eqn}
    \phi^{\vb}_t (\vs) &=&
    \begin{array}{ccc}
      \displaystyle \min_{\vc\in\Crow}
      &
      \displaystyle
      \max_{\vp\in\Delta{\set{1,\ldots,k}}}
      &
      \E_{l\sim \vp}\enco{\phi^{\vb}_{t-1}\enc{\vs + \ve_l}} \\
      & \mbox { s.t. } &
      \E_{l\sim \vp} \enco{c(l)} \leq \dotp{\vb}{\vc},
    \end{array}
  \end{eqnarray}
  where $l\sim \vp$ denotes that label $l$ is sampled from the
  distribution $\vp$, and $\ve_l\in\R^k$ is the unit-vector whose $l$th coordinate is
  $1$ and the remaining coordinates zero.
  Notice the recurrence uses
  the collection of rows $\Crow$ instead of the collection of cost
  matrices $\Csp$.
When there are $T-t$ rounds
remaining (that is, after $t$ rounds of boosting), these potential
functions 
compute an estimate $\phi_{T-t}^{\vb}(\vs_{t})$ of whether an example $x$
will be misclassified, based on its current state $\vs_{t}$ consisting
of counts of votes received so far on various classes:
\begin{equation}
    \label{multi:state:eqn}
  s_{t}(l) =
  \sum_{t'=1}^{t-1} \1\enco{h_{t'}(x) = l}.
\end{equation}
Notice this definition of state assumes that $\alpha_t=1$ in each
round. Sometimes, we will choose the weights differently. In such
cases, a more appropriate definition is the weighted state
$\f_t\in\R^k$, tracking the weighted counts of votes received so far:
\begin{equation}
  \label{multi:var_state:eqn}
  f_t(l) = \sum_{t'=1}^{t-1} \alpha_{t'}\1\enco{h_{t'}(x) = l}.
\end{equation}
However, unless otherwise noted, we will assume $\alpha_t=1$, and so
the definition in \eqref{multi:state:eqn} will suffice.

\old{
The recurrence in \eqref{multi:dgrec:eqn} requires the $\max$ player's
response $\vp$ to satisfy the constraint that the expected cost under
the distribution $\vp$ is at most the inner-product $\dotp{\vc}{\vb}$.
If there is no distribution satisfying this requirement, then
the value of the $\max$ expression is $-\infty$.
The existence of a valid distribution depends on both $\vb$ and $\vc$
and is captured by the following:
\begin{equation}
  \label{multi:pexists:eqn}
  \exists \vp\in\Delta\set{1,\ldots,k}:
  \E_{l\sim\vp}\enco{c(l)} \leq \dotp{\vc}{\vb}
  \iff
  \min_{l}c(l) \leq \dotp{\vb}{\vc}.
\end{equation}
In this \thesis{chapter}\notthesis{paper}, the vector $\vb$ will
always correspond to some row $\B(i)$ of the baseline used in the
weak learning condition.
In such a situation, the next lemma shows that a distribution
satisfying the required constraints will always exist.
\begin{lemma}
  \label{multi:pexists:lem}
  If $\Crow$ is a cone and \eqref{multi:crow:eqn} holds, then for any row
  $\vb=\B(i)$ of the baseline 
  and any cost vector $\vc\in\Crow$,
  \eqref{multi:pexists:eqn} holds unless the condition $(\Csp,\B)$ is
  vacuous. 
\end{lemma}
\begin{proof}
  We show that if \eqref{multi:pexists:eqn} does not hold, then the
  condition is vacuous.
  Assume that for row $\vb=\B(i_0)$ of the baseline, and
  some choice of cost vector $\vc\in\Crow$, 
  \eqref{multi:pexists:eqn} does not hold.
  We pick a cost-matrix
  $\C\in\Csp$, such that no weak classifier $h$ can satisfy the requirement
  \eqref{multi:wl:eqn}, implying the condition must be vacuous.
  The $i_0^{\rm th}$ row of the cost matrix is $\vc$, and the remaining
  rows are $\vzero$.
  Since $\Crow$ is a cone, $\vzero\in\Crow$ and hence the cost matrix
  lies in $\Csp$.
  With this choice for $\C$, the condition
  \eqref{multi:wl:eqn} becomes
  \[
  c(h(x_i)) = C\enc{i,h(x_i)} \leq \dotp{\C(i)}{\B(i)} = \dotp{\vc}{\vb}
  < \min_{l}c(l),
  \]
  where the last inequality holds since, by assumption, \eqref{multi:pexists:eqn} is
  not true for this choice of $\vc,\vb$.
  The previous equation is an impossibility, and hence no such weak
  classifier $h$ exists, showing the condition is vacuous.
\end{proof}
  Lemma~\ref{multi:pexists:lem} shows that the expression
  in \eqref{multi:dgrec:eqn} is well defined, and takes on
  finite values.
  We next record an alternate dual form for the same recurrence which
  will be useful later.
\begin{lemma}
  \label{multi:dgrec_dual:lem}
  The recurrence in \eqref{multi:dgrec:eqn} is equivalent to
  \begin{equation}
    \label{multi:dgrec_dual:eqn}
    \phi^{\vb}_t (\vs)
    =
    \min_{\vc\in\Crow}
    \max_{l=1}^k
    \enct{\phi^{\vb}_{t-1}\enc{\vs+\ve_l}
    - \enc{c(l)-\dotp{\vc}{\vb}}}.
  \end{equation}
\end{lemma}
\begin{proof}
  Using Lagrangean multipliers, we may convert \eqref{multi:dgrec:eqn}
  to an unconstrained expression as follows:
  \[
  \phi^{\vb}_t(\vs)
  =
  \min_{\vc\in\Crow}
  \max_{\vp\in\Delta\set{1,\ldots,k}}
  \min_{\lambda \geq 0}
  \enct{
    \E_{l\sim\vp}\enco{\phi^{\vb}_{t-1}\enc{\vs+\ve_l}}
    - \lambda
    \enc{\E_{l\sim\vp}\enco{c(l)} - \dotp{\vc}{\vb}}}.
  \]
  Applying Theorem~\ref{multi:minmax:thm} to the inner min-max
  expression we get
  \[
    \phi^{\vb}_t(\vs)
  =
  \min_{\vc\in\Crow}
  \min_{\lambda \geq 0}
  \max_{\vp\in\Delta\set{1,\ldots,k}}
  \enct{
    \E_{l\sim\vp}\enco{\phi^{\vb}_{t-1}\enc{\vs+\ve_l}}
    - 
    \enc{\E_{l\sim\vp}\enco{\lambda c(l)} - \dotp{\lambda\vc}{\vb}}}.
  \]
  Since $\Crow$ is a cone, $\vc\in\Crow$ implies
  $\lambda\vc\in\Crow$. Therefore we may absorb the Lagrange
  multiplier into the cost vector:
  \[
  \phi^{\vb}_t(\vs)
  =
  \min_{\vc\in\Crow}
  \max_{\vp\in\Delta\set{1,\ldots,k}}
    \E_{l\sim\vp}\enco{\phi^{\vb}_{t-1}\enc{\vs+\ve_l}
    - \enc{c(l) - \dotp{\vc}{\vb}}}.
  \]
  For a fixed choice of $\vc$, the expectation is maximized when the
  distribution $\vp$ is concentrated on a single label that maximizes
  the inner expression, which completes our proof. 
\end{proof}
The dual form of the recurrence is useful
for optimally choosing the cost matrix in each round.
  When the weak learning condition being used is $(\Csp,\B)$,
  \cite{Schapire01}
  proposed a Booster strategy, called the OS strategy, which always
  chooses the weight $\alpha_t=1$, and uses the potential
  functions to construct a cost matrix $\C_{t}$ as follows.
  Each row $\C_{t}(i)$ of the matrix achieves the minimum of the right
  hand side of 
  \eqref{multi:dgrec_dual:eqn} with $\vb$ replaced by $\B(i)$, $t$ replaced by
  $T-t$, and $\vs$ replaced by current state $\vs_t(i)$:
  \begin{equation}
    \label{multi:os_ct:eqn}
  \C_{t}(i) =
  \argmin_{\vc\in\Crow}
  \max_{l=1}^k
    \enct{\phi^{\B(i)}_{T-t-1}\enc{\vs+\ve_l}
    - \enc{c(l) - \dotp{\vc}{\B(i)}}}.
\end{equation}
}
The
  following theorem, proved in the appendix, provides a guarantee for
  the loss suffered by the 
  OS algorithm, and also shows that it is the game-theoretically
  optimum strategy when the number of examples is large.
  Similar results have been proved by \citet{Schapire01}, but our
  theorem holds much more generally, and also achieves tighter lower
  bounds. 
  \old{
\begin{theorem}[Extension of results in \citep{Schapire01}]
  \label{multi:dgstrat:thm}
 Suppose the weak-learning condition is not vacuous and is given by
 $(\Csp,\B)$, where 
 $\Csp$ is such that, for some convex cone $\Crow\subseteq\R^k$, the condition
 \eqref{multi:crow:eqn} holds.
 Let the potential functions $\phi_t^{\vb}$ be defined as in
 \eqref{multi:dgrec:eqn}, 
 and assume the Booster employs the OS algorithm, choosing
 $\alpha_t=1$ and 
 $\C_t$ as in \eqref{multi:os_ct:eqn} in each round $t$.
 Then the average potential of
 the states,
 \[
 \frac{1}{m}\sum_{i=1}^m \phi_{T-t}^{\B(i)}\enc{\vs_t(i)},
 \]
  never increases in any round. In particular, the loss suffered after $T$
  rounds of play is at most
  \begin{equation}
    \label{multi:dgstratbnd:eqn}
  \frac{1}{m}\sum_{i=1}^m \phi_T^{\B(i)}(\vzero).
\end{equation}

Further, under certain conditions, this bound is nearly tight.
In particular, assume the loss function does not vary too much but
satisfies
\begin{equation}
  \label{multi:lossvar:eqn}
  \sup_{\vs,\vs'\in \mathcal{S}_T} \abs{L(\vs)-L(\vs')} \leq \diameter(L,T),
\end{equation}
where $\mathcal{S}_T$, a subset of
$\set{\vs\in\R^k:\norm{\vs}_\infty \leq  T}$,
is the set of 
all states reachable in $T$ iterations, and $\diameter(L,T)$ is an
upper bound on the discrepancy of losses between any two reachable
states when the loss function is $L$ and the total number of
iterations is $T$.
Then, for any $\eps > 0$, when the number of examples $m$ is
sufficiently large,
\begin{equation}
  \label{multi:mlarge:eqn}
m \geq \frac{T\diameter(L,T)}{\eps},
\end{equation}
no Booster strategy can guarantee to achieve in $T$ rounds a loss that
is $\eps$ less than the bound \eqref{multi:dgstratbnd:eqn}.
\end{theorem}
}
In order to implement the near optimal OS strategy, we need to solve
\eqref{multi:os_ct:eqn}. 
This is computationally only as hard as
evaluating the potentials, which in turn reduces to computing the
recurrences in \eqref{multi:dgrec:eqn}. 
In the next few sections, we study how to do this when
using various losses and  weak learning conditions.

\section{Solving for any fixed edge-over-random condition}

In this section we show how to implement the OS strategy when the weak
learning condition is any fixed edge-over-random condition:
$(\Csp,\B)$ for some $\B\in\Bgam$.
By our previous discussions, this is equivalent to
computing the potential $\phi^{\vb}_t$ by solving the recurrence in
\eqref{multi:dgrec:eqn}, where the vector $\vb$ corresponds to some
row of the baseline $\B$.
Let $\dgam\subseteq
\Delta\set{1,\ldots,k}$ denote the set of all edge-over-random
distributions on 
$\set{1,\ldots,k}$ with $\gamma$ more weight on the first coordinate:
\begin{equation}
  \label{multi:dgam:eqn}
  \dgam = \set{\vb\in \Delta\set{1,\ldots,k}: b(1) - \gamma
    = \max\set{b(2), \ldots, b(k)}}.
\end{equation}
 Note, that $\Bgam$ consists of all matrices whose rows belong to the
 set $\dgam$.
 Therefore we are interested in computing $\phi^{\vb}$, where $\vb$ is
 an arbitrary edge-over-random distribution: $\vb\in\dgam$.
 We begin by simplifying the recurrence \eqref{multi:dgrec:eqn}
 satisfied by such potentials, and
 show how to compute the optimal cost matrix in terms of the
 potentials.
\begin{lemma}
  \label{multi:homogsol:lem}
  Assume $L$ is proper, and $\vb\in\dgam$ is an edge-over-random
  distribution. 
  Then the recurrence \eqref{multi:dgrec:eqn} may be simplified as
  \begin{eqnarray}
    \label{multi:fixed_dgsimp:eqn}
    \phi_t^{\vb}(\vs) 
    &=&
    \E_{l\sim \vb}\enco{\phi_{t-1}\enc{\vs + \ve_l}}.
  \end{eqnarray}
  Further, if the cost matrix $\C_t$ is chosen as follows
  \begin{equation}
    \label{multi:fixed_ctopt:eqn}
    C_t(i,l) = \phi^{\vb}_{T-t-1}(\vs_{t}(i)+\ve_l),    
  \end{equation}
  then $\C_t$ satisfies the condition in \eqref{multi:os_ct:eqn}, and hence
  is the optimal choice.
\end{lemma}
\begin{proof}
Let $\csig \subseteq \R^k$ denote all vectors $\vc$ satisfying
$\forall l:c(1)\leq c(l)$. Then, we have
  \begin{eqnarray*}
    \phi^{\vb}_t (\vs) &=&
    \begin{array}{ccc}
      \displaystyle \min_{\vc\in\csig} &
      \displaystyle \max_{\vp\in\Delta{\set{1,\ldots,k}}} &
      \E_{l\sim \vp}\enco{\phi_{t-1}\enc{\vs + \ve_l}} \\
      &  \mbox { s.t. } & \E_{l\sim\vp}[c(l)] \leq \E_{l\sim \vb}\enco{c(l)},
    \end{array} \enc{ \mbox { by \eqref{multi:dgrec:eqn} } }\\
    &=&  \displaystyle \min_{\vc\in\csig}
    \displaystyle \max_{\vp\in\Delta}
    \displaystyle \min_{\lambda\geq 0}
    \enct{
    \E_{l\sim \vp}\enco{\phi_{t-1}^{\vb}\enc{\vs + \ve_l}} +
    \lambda \enc{\E_{l\sim \vb}\enco{c(l)}-\E_{l\sim\vp}[c(l)]}}
  \enc{\mbox{Lagrangean}}\\
  &=&  \displaystyle \min_{\vc\in\csig}
  \displaystyle \min_{\lambda\geq 0}
    \displaystyle \max_{\vp\in\Delta}
    \E_{l\sim \vp}\enco{\phi_{t-1}^{\vb}\enc{\vs + \ve_l}} +
    \lambda \dotp{\vb-\vp}{\vc}
    \enc{\mbox{Theorem~\ref{multi:minmax:thm}}}\\
  &=&  \displaystyle \min_{\vc\in\csig}
  \displaystyle \max_{\vp\in\Delta}
   \E_{l\sim \vp}\enco{\phi_{t-1}^{\vb}\enc{\vs + \ve_l}} +
    \dotp{\vb-\vp}{\vc} \enc{\mbox{absorb }\lambda\mbox{ into } \vc }\\
  &=& \displaystyle \max_{\vp\in\Delta}
  \displaystyle \min_{\vc\in\csig}
  \E_{l\sim \vp}\enco{\phi_{t-1}^{\vb}\enc{\vs + \ve_l}} +
  \dotp{\vb-\vp}{\vc} \enc{\mbox{Theorem~\ref{multi:minmax:thm}}}.
\end{eqnarray*}
Unless $b(1) - p(1) \leq 0$ and $b(l) - p(l) \geq 0$ for each $l>1$,
the quantity $\dotp{\vb-\vp}{\vc}$ can be made arbitrarily small for
appropriate choices of $\vc\in\csig$. The max-player is therefore
forced to constrain its choices of $\vp$, and the above expression
becomes 
\[
\begin{array}{cl}
  \displaystyle \max_{\vp\in\Delta}  &
\E_{l\sim \vp}\enco{\phi_{t-1}^{\vb}\enc{\vs + \ve_l}} \\
{\rm s.t.} &
b(l) - q(l) 
\begin{cases}
  \geq 0 & \mbox{ if } l=1, \\
  \leq 0 & \mbox{ if } l>1.
\end{cases}
\end{array}
\]
Lemma~6 of \citep{Schapire01} states that if $L$ is \emph{proper} (as
defined here), so is
$\phi_t^{\vb}$; the same result can be extended to our drifting
games. This implies the optimal choice of $\vp$ in the above
expression is in fact the distribution that puts as small weight as
possible in the first coordinate, namely $\vb$.
Therefore the optimum choice of $\vp$ is $\vb$, and
the potential is the same as in \eqref{multi:fixed_dgsimp:eqn}.

\old{
We end the proof by showing that the choice of cost matrix in
\eqref{multi:fixed_ctopt:eqn} is optimum.
Theorem~\ref{multi:dgstrat:thm} states that a cost matrix $\C_t$ is
the optimum choice if it satisfies \eqref{multi:os_ct:eqn}, that is, if
the expression
\begin{equation}
  \label{multi:ctopt_one:eqn}  
  \max_{l=1}^k
  \enct{\phi^{\B(i)}_{T-t-1}\enc{\vs+\ve_l}
    - \enc{C_t(i,l) - \dotp{\C_t(i)}{\B(i)}}}
\end{equation}
is equal to
\begin{equation}
  \label{multi:ctopt_two:eqn}    
  \min_{\vc\in\Crow}
  \max_{l=1}^k
  \enct{\phi^{\B(i)}_{T-t-1}\enc{\vs+\ve_l}
    - \enc{c(l) - \dotp{\vc}{\B(i)}}}
  =
  \phi^{\B(i)}_{T-t}\enc{\vs},
\end{equation}
where the equality in \eqref{multi:ctopt_two:eqn} follows from
\eqref{multi:dgrec_dual:eqn}. 
If $\C_t(i)$ is chosen as in \eqref{multi:fixed_ctopt:eqn}, then, for
any label $l$, the expression within $\max$ in
\eqref{multi:ctopt_one:eqn} evaluates to 
\begin{eqnarray*}
  \phi^{\B(i)}_{T-t-1}\enc{\vs+\ve_l}
  &-& \enc{\phi^{\B(i)}_{T-t-1}\enc{\vs+\ve_l} -
      \dotp{\C_t(i)}{\B(i)}}\\
&=&
\dotp{\B(i)}{\C_t(i)}\\
&=&
\E_{l\sim\B(i)}\enco{C_t(i,l)}\\
&=&
\E_{l\sim\B(i)}\enco{\phi^{\B(i)}_{T-t-1}\enc{\vs+\ve_l}}\\
&=&
\phi^{\B(i)}_{T-t}(\vs),
\end{eqnarray*}
where the last equality follows from \eqref{multi:fixed_dgsimp:eqn}.
Therefore the $\max$ expression in \eqref{multi:ctopt_one:eqn} is also
equal to $\phi^{\B(i)}_{T-t}(\vs)$, which is what we needed to show.
}
\end{proof}
Eq. \eqref{multi:fixed_ctopt:eqn} in Lemma~\ref{multi:homogsol:lem}
implies the cost matrix chosen by the 
OS strategy can be expressed in terms of the potentials,
which is the only thing left to calculate.
Fortunately, the simplification \eqref{multi:fixed_dgsimp:eqn} of the
drifting games recurrence, allows the potentials to be
solved completely in terms of a random-walk 
$\rw^t_\vb(\vx)$.
This random variable denotes the
position of a particle after 
$t$ time steps, that starts at location $\vx\in \R^k$, and in each
step moves in direction $\ve_l$ with probability $b(l)$.
\begin{corollary}
  \label{multi:fixed_rw:cor}
  The recurrence in \eqref{multi:fixed_dgsimp:eqn} can be solved as
  follows: 
  \begin{equation}
    \label{multi:fixed_rw:eqn}
  \phi^{\vb}_t(\vs) = \E \enco{L\enc{\rw^t_\vb(\vs)}}.
  \end{equation}
\end{corollary}
\begin{proof}
Inductively assuming $\phi_{t-1}^{\vb}(\vx) =
\E\enco{L(\rw^{t-1}_\vb(\vx))}$,
\begin{eqnarray*}
  \phi_t(\vs) = \E_{l\sim \vb}\enco{L(\rw^{t-1}_\vb(\vs) + \ve_l)}
  =  \E\enco{L(\rw^t_\vb(\vs))}.
\end{eqnarray*}
The last equality follows by observing that the random position
$\rw^{t-1}_\vb(\vs) + \ve_l$ is distributed as $\rw^t_\vb(\vs)$ when
$l$ is sampled from $\vb$.
\end{proof}
Lemma~\ref{multi:homogsol:lem} and Corollary~\ref{multi:fixed_rw:cor}
together imply:
\begin{theorem}
  \label{multi:fixedsol:thm}
  Assume $L$ is proper and $\vb\in\dgam$ is an edge-over-random
  distribution. Then the potential $\phi_t^{\vb}$, defined by the
  recurrence in \eqref{multi:dgrec:eqn}, has the solution given in
  \eqref{multi:fixed_rw:eqn} in terms of random walks.
\end{theorem}
Before we can compute \eqref{multi:fixed_rw:eqn}, we need to choose a loss 
function $L$. We next consider two options for the loss
 --- the non-convex 0-1 error, and exponential loss.

\paragraph{Exponential Loss.}
The exponential loss serves as a smooth convex proxy for
discontinuous non-convex 0-1 error \eqref{multi:zero_one:eqn} that we
would ultimately like to bound, and is given by
\begin{equation}
  \label{multi:exploss:eqn}
  \Lexpe(\vs) = \sum_{l=2}^ke^{\eta(s_l-s_1)}.
\end{equation}
The parameter $\eta$ can be thought of as the weight in each round, that is,
$\alpha_t = \eta$ in each round. Then notice that the weighted state $\f_t$
of the examples, defined in \eqref{multi:var_state:eqn}, is related to
the unweighted states $\vs_t$ as $f_t(l) = \eta s_t(l)$. Therefore the
exponential loss function in \eqref{multi:exploss:eqn} directly
measures the loss of the weighted state as
\begin{equation}
  \label{multi:exploss_varstate:eqn}
  \Lexp(\f_t) = \sum_{l=2}^k e^{f_t(l) - f_t(1)}.
\end{equation}
Because of this correspondence, the optimal strategy with the loss
function $\Lexp$ and $\alpha_t=\eta$ is the same as that using loss
$\Lexpe$ and $\alpha_t=1$. We study the latter setting so that we may
use the results derived earlier. With the choice of the exponential
loss $\Lexpe$,  the potentials are easily
computed, and in fact have a 
closed form solution.
\begin{theorem}
  \label{multi:expdgsol:thm}
  If $\Lexpe$ is as in \eqref{multi:exploss:eqn}, where $\eta$ is
  non-negative, then the 
  solution in Theorem~\ref{multi:fixedsol:thm} evaluates to
  $\phi^{\vb}_t(\vs) = \sum_{l=2}^k (a_l)^te^{\eta_l\enc{s_l-s_1}}$,
  where $a_l = 1 - (b_1 + b_l) + e^{\eta}b_l + e^{-\eta}b_1$.
\end{theorem}
The proof by induction is straightforward.
 By tuning the weight $\eta$,
each $a_l$ can be always made less than $1$. 
This ensures the exponential loss decays exponentially with rounds.  
In particular, when $\B=\Ugam$ (so that the 
condition is $(\Csig,\Ugam)$), the
relevant potential $\phi_t(\vs)$ or $\phi_t(\f)$ is given by
\begin{equation}
  \label{multi:unifpot:eqn}
  \phi_t(\vs) = \phi_t(f) =
  \kappa(\gamma,\eta)^t\sum_{l=2}^ke^{\eta\enc{s_l-s_1}}
  =   \kappa(\gamma,\eta)^t\sum_{l=2}^ke^{f_l-f_1}
\end{equation}
where
\begin{equation}
  \label{multi:kappa:eqn}
\kappa(\gamma,\eta) = 1 +
\frac{\enc{1-\gamma}}{k}\enc{e^\eta + e^{-\eta}-2} -
\enc{1-e^{-\eta}}\gamma.
\end{equation}
 \ignore{
\begin{equation}
  \label{multi:unifsol:eqn}
C(i,l) =
\begin{cases}
  \kappa(\gamma,\eta)^{T-t}\enct{\sum_{j=2}^k e^{\eta(s_j-s_1)} +
    \enc{e^\eta-1}e^{\eta(s_l-s_1)}}& \mbox { if } l>1, \\
  \kappa(\gamma,\eta)^{T-t}\enct{\sum_{j=2}^k e^{\eta(s_j-s_1)}\cdot
    e^{-\eta}} & \mbox { if } l=1.
\end{cases}
\end{equation}}
The cost-matrix output by the OS algorithm can be simplified by
rescaling, or adding the same number to each coordinate
of a cost vector, without affecting the constraints it imposes on a
weak classifier, to the following form
\[
C(i,l) =
\begin{cases}
  \enc{e^\eta-1}e^{\eta(s_l-s_1)}& \mbox { if } l>1, \\
  \enc{e^{-\eta}-1}\sum_{l=2}^k e^{\eta(s_l-s_1)} & \mbox { if } l=1.
\end{cases}
\]
Using the correspondence between unweighted and weighted states, the
above may also be rewritten as:
\begin{equation}
  \label{multi:unifsol2:eqn}
C(i,l) =
\begin{cases}
  \enc{e^\eta-1}e^{f_l-f_1}& \mbox { if } l>1, \\
  \enc{e^{-\eta}-1}\sum_{l=2}^k e^{f_l-f_1} & \mbox { if } l=1.
\end{cases}
\end{equation}
With such a choice, Theorem~\ref{multi:dgstrat:thm} and the form of the
potential guarantee that the average loss
\begin{equation}
  \label{multi:avgexploss:eqn}
\frac{1}{m}\sum_{i=1}^m\Lexpe(\vs_t(i))
=
\frac{1}{m}\sum_{i=1}^m\Lexp(\f_t(i))
\end{equation}
of the states changes by a
factor of at most $\kappa\enc{\gamma,\eta}$ every round. Therefore the
final loss, which upper bounds the error, i.e., the fraction of
misclassified training examples, 
is at most $(k-1)\kappa\enc{\gamma,\eta}^T$.
Since this upper bound holds for any value of $\eta$, we may tune it to
optimize the bound.
Setting $\eta=\ln\enc{1+\gamma}$, the error can be upper bounded by
$(k-1)e^{-T\gamma^2/2}$.

\ignore{When the loss function is $L(\vx) = \exp\enc{\eta\enc{-x_1 +
      x_2 + \ldots + x_k}}$, we recover the Multiclass AdaBoost
  algorithm. However, since this loss does not upper-bound the
  training error, and can go to zero even when training error remains
  high, we will not discuss this due to lack of space.}

\paragraph{Zero-one Loss.} There is no simple closed form solution for
the potential when using the zero-one loss $\Lzero$
\eqref{multi:zero_one:eqn}.
However, we may compute the 
potentials efficiently as follows. To compute $\phi_t^{\vb}(\vs)$, we
need to find the probability that a random walk (making steps
according to $\vb$) of length $t$ in $\Z^k$,
starting at $\vs$ will end up in a region where the loss function is
$1$. Any such random walk will consist of $x_l$ steps in direction
$\ve_l$ where the non-negative $\sum_l x_l = t$.
The probability of each such path is
$\prod_{l} b_l^{x_l}$. Further, there are exactly
$\binom{t}{x_1,\ldots,x_{k}}$ such paths. Starting at state $\vs$,
such a path will lead to a correct answer only if $s_1 + x_1 > s_l +
x_l$ for each $l>1$. Hence we may write the potential $\phi_t^{\vb}(\vs)$ as
\begin{eqnarray*}
  \phi_t^{\vb}(\vs) = 1 - \sum_{x_1,\ldots,x_k}^t&
  \binom{t}{x_1,\ldots,x_{k}}\prod_{l=1}^kb_l^{x_l} & \\
  \text { s.t. } & x_1 + \ldots + x_{k} &= t \\
  \forall l:& x_l &\geq 0\\
  \forall l:& x_l+s_l &\leq x_1 + s_1.
\end{eqnarray*}
Since the $x_l$'s are restricted to be integers, this problem is
presumably hard. In particular, the only algorithms known to the
authors that take time logarithmic in $t$ is also exponential in
$k$. However, by using dynamic programming, we can compute the
summation in time polynomial in $|s_l|$, $t$ and $k$. In fact, the
runtime is always $O(t^3k)$, and at least $\Omega(tk)$.

The bounds on error we achieve, although not in closed form, are much
tighter than those obtainable using 
exponential loss. The exponential loss analysis yields an error upper
bound of $(k-1)e^{-T\gamma^2/2}$. Using a different initial
distribution, \citet{SchapireSi99} achieve the slightly better bound
$\sqrt{(k-1)}e^{-T\gamma^2/2}$. However, when the edge $\gamma$ is small and
the number of 
rounds are few, each bound is greater than 1 and hence trivial. On the
other hand, the bounds computed by the above dynamic program are sensible
for all values of $k$, $\gamma$ and $T$. When $\vb$ is the
$\gamma$-biased uniform distribution
 $\vb=(\frac{1-\gamma}{k} + \gamma,
 \frac{1-\gamma}{k},
 \frac{1-\gamma}{k},
 \ldots,
 \frac{1-\gamma}{k})$
 a table containing the error upper bound
$\phi_T^{\vb}(0)$ for $k=6$, $\gamma=0$ and small values for the
number of rounds $T$ is shown in
Figure~\ref{multi:phitab:fig}; note that with the exponential loss,
the bound is always 1 if the edge $\gamma$ is 0.
Further, the bounds due to the exponential
loss analyses seem to imply that the dependence of the error on the
number of labels is monotonic. However, a plot of the tighter bounds
with edge $\gamma=0.1$, number of rounds $T=10$ 
against various values of $k$, shown in Figure~\ref{multi:kcurve:fig},
indicates that the true dependence is more complicated. Therefore the
tighter analysis also provides qualitative insights not obtainable via
the exponential loss bound.
\begin{figure}[ht]
  \begin{center}
    \subfigure[]{
      \centering
    \begin{tabular}{|r|r||r|r|}
      \hline
      $T$ &  $\phi^{\vb}_T(\vzero)$ & $T$ &  $\phi^{\vb}_T(\vzero)$ \\
      \hline
      \hline
      0 & 1.00 & 6 & 0.90 \\
      1 & 0.83 & 7 & 0.91 \\
      2 & 0.97 & 8 & 0.90 \\
      3 & 0.93 & 9 & 0.89 \\
      4 & 0.89 & 10 & 0.89 \\
      5 & 0.89 & & \\
      \hline
    \end{tabular}
    \label{multi:phitab:fig}
  }\;\;\;\;
  \subfigure[]{
    \centering
    \begin{tabular}{c}
      \includegraphics[width=0.4\textwidth]{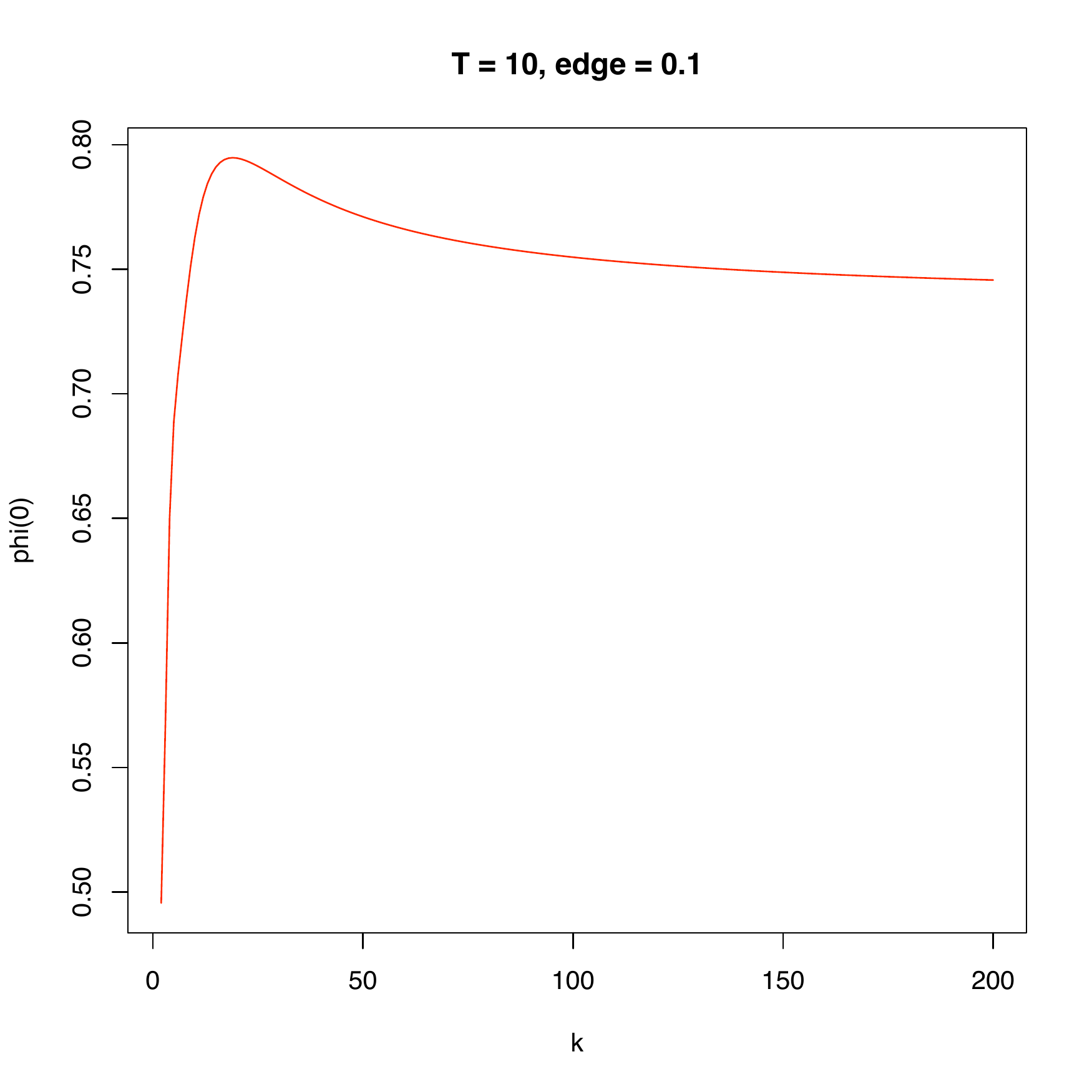}
    \end{tabular}
    \label{multi:kcurve:fig}
  }
\end{center}
  \caption[Potential plots when loss is 0-1 error.]{Plot of potential
    value $\phi_T^{\vb}(\vzero)$ where $\vb$ 
    is the $\gamma$-biased uniform distribution:
    $\vb=(\frac{1-\gamma}{k} + \gamma,
    \frac{1-\gamma}{k},
    \frac{1-\gamma}{k},
    \ldots,
    \frac{1-\gamma}{k})$.
    {\bf \subref{multi:phitab:fig}:} Potential values (rounded to two
    decimal places) for different number of rounds $T$ using
    $\gamma=0$ and  $k=6$. These are bounds on the error, and less
    than 1 even when the edge and number of rounds are small.
    {\bf \subref{multi:kcurve:fig}:} Potential values for different
    number of  classes $k$, with $\gamma=0.1$, and $T=10$. These are
    tight estimates for the optimal error, and yet not monotonic in
    the number of classes.
  }
  \label{multi:zo_pot:fig}  
\end{figure}
\section{Solving for the minimal weak learning condition}

In the previous section we saw how to find the optimal boosting
strategy when using any fixed edge-over-random condition. However
as we have seen before, these conditions can be stronger than
necessary, and therefore lead to boosting algorithms that require
additional assumptions. Here we show how to compute the optimal
algorithm while using the weakest weak learning condition, provided by
\eqref{multi:minwl:eqn}, or equivalently the condition used by AdaBoost.MR,
$(\Cmr,\MR_{\gamma})$. Since there are two possible formulations for
the minimal condition, it is not immediately clear which to use to
compute the optimal boosting strategy. To resolve this, we first show
that the optimal boosting strategy based on any formulation of a
necessary and sufficient weak learning condition is the same. Having
resolved this ambiguity, we show how to compute this strategy for the
exponential loss and 0-1 error using the first formulation.

\subsection{Game-theoretic equivalence of necessary and sufficient
  weak-learning conditions}

In this section we study the effect of the weak learning condition on
the game-theoretically optimal boosting strategy. We introduce the
notion of \emph{game-theoretic equivalence} between two weak learning
conditions, that determines if the payoffs \eqref{multi:payoff:eqn} of
the optimal boosting strategies based on the two conditions are
identical. This is different from the usual notion of equivalence
between two conditions, which holds if any weak classifier space
satisfies both conditions or neither condition. In fact we prove that
game-theoretic equivalence is a broader notion; in other words,
equivalence implies game-theoretic equivalence. A special case of
this general result is that any two weak
learning conditions that are necessary and sufficient, and hence
equivalent to boostability, are also game-theoretically equivalent.
In particular, so are the conditions of AdaBoost.MR and
\eqref{multi:minwl:eqn}, and the resulting optimal Booster strategies
enjoy equally good payoffs. We conclude that in order to derive
the optimal boosting strategy that uses the minimal weak-learning
condition, it is sound to use either of these two formulations.

The purpose of a weak learning condition $(\Csp,\B)$ is to impose
restrictions on 
the Weak-Learner's responses in each round. These restrictions are
captured by subsets of the weak classifier space as follows.
If Booster chooses
cost-matrix $\C\in\Csp$ in a round, the Weak-Learner's response
$h$ is restricted to the subset $S_{\C}\subseteq\Hall$ defined
as
\[
S_{\C} = \set{h\in\Hall: \C\bullet\vh \leq \C\bullet\B}.
\]
Thus, a weak
learning condition is essentially a family of subsets of the weak
classifier space. Further, smaller
subsets mean fewer options for Weak-Learner, and hence better payoffs
for the optimal boosting strategy. Based on this idea, we may define
when a weak learning condition  $(\Csp_1,\B_1)$ is \emph{game-theoretically
  stronger} than another condition $(\Csp_2,\B_2)$ if the following holds:
For every subset $S_{\C_2}$ in the second 
condition (that is $\C_2\in\Csp_2$), there is a subset $S_{\C_1}$ in
the  first condition (that is $\C_1\in\Csp_1$), such that $S_{\C_1}
\subseteq S_{\C_2}$. Mathematically, this may be written as follows:
\[
\forall \C_1 \in \Csp_1,
\exists \C_2 \in \Csp_2
:
S_{\C_1} \subseteq S_{\C_2}.
\]
Intuitively, a game theoretically stronger
condition will allow Booster to place similar or stricter
restrictions on Weak-Learner in each round. Therefore, the optimal
Booster payoff using a game-theoretically stronger condition is at least
equally good, if not better. 
It therefore follows that if two conditions are both
game-theoretically stronger than each other, the corresponding Booster
payoffs must be equal, that is they must be \emph{game-theoretically
  equivalent}. 

Note that game-theoretic equivalence of two conditions does not mean
that they are identical as families of subsets, for we may arbitrarily
add large and ``useless'' subsets to the two conditions without affecting
the Booster payoffs, since these subsets will never be used by an
optimal Booster strategy. In fact we next show that game-theoretic
equivalence is a broader notion than just equivalence.
\begin{theorem}
  \label{multi:eq_gteq:thm}
  Suppose $(\Csp_1,\B_1)$ and $(\Csp_2,\B_2)$ are two equivalent weak
  learning conditions, that is, every space $\H$ satisfies both or
  neither condition. Then each condition is game-theoretically
  stronger than the other, and hence game-theoretically equivalent.
\end{theorem}
\begin{proof}
  We argue by contradiction. Assume that despite
equivalence, the first condition (without loss of generality) includes a
particularly hard subset $S_{\C_1}\subseteq 
\Hall, \C_1\in\Csp_1$
which is not smaller than any subset in the second condition. In
particular, for every  subset $S_{\C_2},\C_2\in\Csp_2$ in the second
condition is satisfied by some weak classifier $h_{\C_2}$ not
satisfying the hard subset in the first condition: $h_{\C_2}\in
S_{\C_2}\setminus S_{\C_1}$. Therefore, the space
\[
\H = \set{h_{\C_2}:\C_2\in\Csp_2},
\]
formed by just these classifiers 
satisfies the second condition, but has an empty intersection with
$S_{\C_1}$. In other words, $\H$ satisfies the second but not the
first condition, a contradiction to their equivalence.
\end{proof}
An immediate corollary is the game theoretic equivalence of necessary
and equivalent conditions.
\begin{corollary}
  \label{multi:necsuf_gteq:cor}
  Any two necessary and sufficient weak learning conditions are
  game-theoretically equivalent. In particular the optimum Booster
  strategies based on AdaBoost.MR's condition $(\Cmr,\MR_\gamma)$ and
  \eqref{multi:minwl:eqn} have equal payoffs.
\end{corollary}
Therefore, in deriving the optimal Booster strategy, it is sound to
work with either AdaBoost.MR's condition or
\eqref{multi:minwl:eqn}. In the next section, we actually compute 
the optimal strategy using the latter formulation.

\subsection{Optimal strategy with the minimal conditions}

In this section we compute the optimal Booster strategy that uses the
minimum weak learning condition provided in 
\eqref{multi:minwl:eqn}. We choose this instead of AdaBoost.MR's
condition because this description is more closely related to the
edge-over-random conditions, and the resulting algorithm has a close
relationship to the ones derived for fixed edge-over-random
conditions, and therefore more insightful. However, this formulation
does not state the condition as a single pair $(\C,\B)$, and therefore
we cannot directly use the result of
Theorem~\ref{multi:dgstrat:thm}. Instead, we define new potentials and
a modified OS strategy that is still nearly optimal, and this
constitutes the first part of this section. 
In the second part, we
show how to compute these new potentials and the resulting OS
strategy. 
 
\subsubsection{Modified potentials and OS strategy}
The condition in \eqref{multi:minwl:eqn} is not stated as
a single pair $(\Csig,\B)$, but uses all possible edge-over-random
baselines $\B\in\Bgam$. Therefore, we modify the definitions
\eqref{multi:dgrec:eqn} of the potentials accordingly to
extract an optimal Booster strategy.
Recall that $\dgam$ is defined in \eqref{multi:dgam:eqn} as the set of
all edge-over-random distributions which constitute the rows of
edge-over-random baselines $\B\in\Bgam$.
Using these, define new potentials $\phi_t(\vs)$ as follows:
 \begin{equation}
   \label{multi:dgnecsuf:eqn}
   \phi_t (\vs) =
   \begin{array}{ccc}
     \displaystyle \min_{\vc\in\csig} &
     \displaystyle
     \max_{\vb\in\dgam}
     \max_{\vp\in\Delta\set{1,\ldots,k}} &
     \E_{l\sim \vp}\enco{\phi_{t-1}\enc{\vs + \ve_l}} \\
     &  \mbox { s.t. } & \E_{l\sim\vp}[c(l)] \leq \dotp{\vb}{\vc}.
   \end{array}
 \end{equation}
 The main difference between \eqref{multi:dgnecsuf:eqn} and
 \eqref{multi:dgrec:eqn} is that while the older potentials 
 were defined using a fixed vector $\vb$ corresponding to some row in
 the fixed baseline $\B$, the new definition takes the maximum over all
 possible rows $\vb\in\dgam$ that an edge-over-random baseline $\B\in\Bgam$
 may have.
 \old{
 As before, we may write the recurrence in \eqref{multi:dgnecsuf:eqn}
 in its dual form
 \begin{equation}
   \label{multi:dgnecsuf_dual:eqn}
   \phi_t(\vs) =
   \min_{\vc\in\csig} 
   \max_{\vb\in\dgam}
   \max_{l=1}^k
   \enct{\phi_{t-1}\enc{\vs + \ve_l}
     -\enc{c(l)-\dotp{\vc}{\vb}}}.
 \end{equation}
 The proof is very similar to that of Lemma~\ref{multi:dgrec_dual:lem}
 and is omitted.
 We may now define a new OS strategy that chooses
 a cost-matrix in round $t$ analogously:
 \begin{equation}
    \label{multi:nos_ct:eqn}
  \C_{t}(i) \in
   \argmin_{\vc\in\csig} 
   \max_{\vb\in\dgam}
   \max_{l=1}^k
   \enct{\phi_{t-1}\enc{\vs + \ve_l}
     -\enc{c(l)-\dotp{\vc}{\vb}}}.
\end{equation}
where recall that $\vs_t(i)$ denotes the state vector (defined in
\eqref{multi:state:eqn}) of example $i$.
With this strategy, we can show an optimality result very similar to
Theorem~\ref{multi:dgstrat:thm}. 
\begin{theorem}
  \label{multi:mindg:thm}
  Suppose the weak-learning condition is given by
  \eqref{multi:minwl:eqn}. Let the potential functions $\phi_t^{\vb}$
  be defined as in \eqref{multi:dgnecsuf:eqn}, and assume the Booster
  employs the modified OS strategy, choosing $\alpha_t=1$ and $\C_t$
  as in \eqref{multi:nos_ct:eqn} in each round $t$. Then the average potential of
  the states,
  \[
  \frac{1}{m}\sum_{i=1}^m \phi_{T-t}\enc{\vs_t(i)},
  \]
  never increases in any round.
  In particular, the loss suffered after $T$
  rounds of play is at most $\phi_T(\vzero)$.

  Further, for any $\eps > 0$, when the loss function satisfies
  \eqref{multi:lossvar:eqn} and the number of examples $m$ is as large
  as in \eqref{multi:mlarge:eqn},
  no Booster strategy can guarantee to 
  achieve less than $\phi_T(\vzero)-\eps$ loss in $T$ rounds.
\end{theorem}
The proof is very similar to that of
Theorem~\ref{multi:dgstrat:thm} and is omitted.
}

\subsubsection{Computing the new potentials.} Here we show how to
compute the new potentials. The resulting algorithms will require
exponential time, and we provide some empirical evidence showing that
this might be necessary. Finally, we show how to carry out the
computations efficiently in certain special situations.

\paragraph{An exponential time algorithm.}
Here we show how the potentials may be computed as the expected loss
of some random walk, just as we did for the potentials arising with
fixed edge-over-random conditions. The main difference is there will
be several random walks to choose from.

We first begin by simplifying the recurrence
\eqref{multi:dgnecsuf:eqn}, and expressing the optimal cost matrix in
\eqref{multi:nos_ct:eqn} in terms of the potentials, just as we did
in Lemma~\ref{multi:homogsol:lem} for the case of fixed
edge-over-random conditions. 
\begin{lemma}
  \label{multi:rec_simp:lem}
  Assume $L$ is proper.
  Then the recurrence \eqref{multi:dgnecsuf:eqn} may be simplified as
   \begin{eqnarray}
   \label{multi:homognecsufsol:eqn}
    \phi_t(\vs) 
   &=& \max_{\vb\in\dgam}
   \E_{l\sim \vb}\enco{\phi_{t-1}\enc{\vs + \ve_l}}.
 \end{eqnarray}
 Further, if the cost matrix $\C_t$ is chosen as follows:
 \begin{equation}
   \label{multi:ctopt:eqn}
   C_t(i,l) = \phi_{T-t-1}(\vs_{t}(i)+\ve_l),
 \end{equation}
 then $\C_t$ satisfies the condition in \eqref{multi:nos_ct:eqn}.
\end{lemma}
The proof is very similar to that of Lemma~\ref{multi:homogsol:lem}
and is omitted.
Eq. \eqref{multi:ctopt:eqn} implies that, as before, computing the
optimal Booster strategy reduces to computing the new potentials. One
computational difficulty created by the new definitions
\eqref{multi:dgnecsuf:eqn} or \eqref{multi:homognecsufsol:eqn} is that
they require infinitely many possible distributions $\vb\in\dgam$ to
be considered. We show that we may in fact restrict our
attention to only finitely many of such distributions described next.

At any state $\vs$ and number of remaining iterations $t$, let $\pi$
be a permutation of the coordinates $\set{2,\ldots,k}$ that sorts the
potential values:
  \begin{equation}
    \label{multi:pi:eqn}
    \phi_{t-1}\enc{\vs + \ve_{\pi(k)}}
    \geq   \phi_{t-1}\enc{\vs + \ve_{\pi(k-1)}}
    \geq \ldots
    \geq \phi_{t-1}\enc{\vs + \ve_{\pi(2)}}.
  \end{equation}
  For any permutation $\pi$ of the coordinates $\set{2,\ldots,k}$, let
  $\vb^{\pi}_a$ denote the $\gamma$-biased uniform distribution on the
  $a$ coordinates $\set{1,\pi_k,\pi_{k-1},\ldots,\pi_{k-a+2}}$:
  \begin{equation}
    \label{multi:bpi:eqn}
    b^{\pi}_a(l) =
    \begin{cases}
      \frac{1-\gamma}{a} + \gamma & \mbox{ if } l = 1 \\
     \frac{1-\gamma}{a} & \mbox{ if } l \in \set{\pi_k,\ldots,
       \pi_{k-a+2}} \\
     0 & \mbox { otherwise. }
   \end{cases}
 \end{equation}
 Then, the next lemma shows that we may restrict our attention to only
 the distributions $\set{\vb^{\pi}_2, \ldots, \vb^{\pi}_k}$ when
 evaluating the recurrence in \eqref{multi:homognecsufsol:eqn}. 
\begin{lemma}
  \label{multi:homognecsufsol:lem}
  Fix a state $\vs$ and remaining rounds of boosting $t$.  Let
  $\pi$ be a permutation of the coordinates $\set{2,\ldots,k}$
  satisfying \eqref{multi:pi:eqn}, and define $\vb^{\pi}_a$ as in
  \eqref{multi:bpi:eqn}.  Then the recurrence
  \eqref{multi:homognecsufsol:eqn} may be simplified as follows:
  \begin{eqnarray}
   \label{multi:pot_simp:eqn}    
   \phi_t(\vs) 
   = \max_{\vb\in\dgam}
   \E_{l\sim \vb}\enco{\phi_{t-1}\enc{\vs + \ve_l}} 
   = \max_{2\leq a \leq k}
   \E_{l\sim \vb^{\pi}_a}\enco{\phi_{t-1}\enc{\vs + \ve_l}}.
 \end{eqnarray}
\end{lemma}
\begin{proof}
  Assume (by relabeling the coordinates if necessary) that $\pi$ is
  the identity permutation, that is, $\pi(2) = 2, \ldots, \pi(k) =
  k$. Observe that
  the right hand side of \eqref{multi:homognecsufsol:eqn} is at least
  as much the right hand side of \eqref{multi:pot_simp:eqn}
  since the former considers more distributions. We complete the proof
  by showing that the former is also at most the latter.
  
  By \eqref{multi:homognecsufsol:eqn}, we may assume that some optimal $\vb$
  satisfies
  \begin{eqnarray*}
    b(k) = \cdots = b(k-a+2) &=& b(1) - \gamma,\\
    b(k-a+1)  &\leq& b(1)-\gamma,\\
    b(k-a) = \cdots = b(2) &=& 0.
  \end{eqnarray*}
  Therefore, $\vb$ is a distribution supported on $a+1$ elements, with
  the minimum weight placed on element $k-a+1$.
  This implies $b(k-a+1) \in [0,1/(a+1)]$.

  Now, $\E_{l\sim \vb}\enco{\phi_{t-1}(\vs+\ve_l)}$ may be written as
  \begin{eqnarray*}
    && \gamma\cdot\phi_{t-1}(\vs + \ve_1)  +  b(k-a+1)\phi_{t-1}(\vs + \ve_{k-a+1})\\
    &+& (1-\gamma - b(k-a+1))
    \frac{\phi_{t-1}(\vs+\ve_1)+\phi_{t-1}(\vs+\ve_{k-a+2})      + 
      \ldots \phi_{t-1}(\vs + \ve_k)}{a} \\
    &=&  \gamma\cdot\phi_{t-1}(\vs + \ve_1)  +
    \frac{b(k-a+1)}{1-\gamma}\phi_{t-1}(\vs + \ve_{k-a+1}) \\
    &+& (1-\gamma)\Big\{ 
     \enc{1 - \frac{b(k-a+1)}{1-\gamma}}
      \frac{\phi_{t-1}(\vs+\ve_1) + \phi_{t-1}(\vs+\ve_{k-a+2}) +
        \ldots \phi_{t-1}(\vs + \ve_k)}{a}  \Big\}
  \end{eqnarray*}
  Replacing $b(k-a+1)$ by $x$ in the above expression, we get a linear
  function of $x$. When restricted to $[0,1/(a+1)]$ the maximum value
  is attained at a boundary point. For $x=0$, the expression becomes
  \[
  \gamma\cdot\phi_{t-1}(\vs + \ve_1) + (1-\gamma)
  \frac{\phi_{t-1}(\vs+\ve_1) + \phi_{t-1}(\vs+\ve_{k-a+2}) + \ldots
    \phi_{t-1}(\vs + \ve_k)}{a}.
  \]
  For $x=1/(a+1)$, the expression becomes
  \[
  \gamma\cdot\phi_{t-1}(\vs + \ve_1) + (1-\gamma)
  \frac{\phi_{t-1}(\vs+\ve_1) + \phi_{t-1}(\vs+\ve_{k-a+1}) + \ldots
    \phi_{t-1}(\vs + \ve_k)}{a+1}.
  \]
  Since $b(k-a+1)$ lies in $[0,1/(a+1)]$, the optimal value is at most
  the maximum of the two. However each of these last two expressions
  is at most the right hand side of
  \eqref{multi:pot_simp:eqn},
  completing the proof.
\end{proof}
 Unraveling \eqref{multi:pot_simp:eqn}, we find that $\phi_t(\vs)$ is
 the expected loss of the final state reached by some random walk of
 $t$ steps starting at state $\vs$. However, the number of
 possibilities for the random-walk is huge; indeed, the distribution
 at each step can be any of the $k-1$ possibilities $\vb^{\pi}_a$ for
 $a\in\set{2,\ldots,k}$, where the parameter $a$ denotes the size of
 the support of the $\gamma$-biased uniform distribution chosen at
 each step. In other words, at a given state $\vs$ with $t$ rounds of
 boosting remaining, the parameter $a$ determines the number of
 directions the optimal random walk will consider taking; we will
 therefore refer to $a$ as the \emph{degree} of the random walk given
 $(\vs,t)$. Now, the total number of states reachable in $T$ steps is
 $O\enc{T^{k-1}}$.
 If the degree assignment  every such state, for every value of $t\leq
 T$ is fixed in advance, $\va = \set{a(\vs,t): t \leq T, \vs \mbox{
     reachable}} $, we may identify a
 unique random walk $\rw^{\va,t}(\vs)$ of length $t$ starting at step
 $\vs$. Therefore the potential may be computed as
\begin{equation}
  \label{multi:rws:eqn}
\phi_t(\vs) =
\max_{\va}
\E\enco{\rw^{\va,t}(\vs)}.
\end{equation}
A dynamic programming approach for
computing \eqref{multi:rws:eqn} requires time and memory linear in the
number of different states reachable by a random walk that takes $T$
coordinate steps: $O(T^{k-1})$. This is exponential in the dataset size,
and hence impractical.
In the next two sections we show that perhaps there may not be any way
of computing these efficiently in general, but provide efficient
algorithms in certain special cases.

\paragraph{Hardness of evaluating the potentials.} 

Here we provide empirical evidence for the hardness of computing the
new potentials. We first identify a computationally easier problem,
and show that even that is probably hard to compute.
Eq. \eqref{multi:pot_simp:eqn}
implies that if the potentials were 
efficiently computable, the correct value of the degree $a$ could be determined
efficiently. The problem of determining the degree $a$ given the state $\vs$ and
remaining rounds $t$ is therefore  easier than
evaluating the potentials. However, a plot of the degrees against
states and remaining rounds, henceforth called a \emph{degree map},
reveals very little structure that might be captured by a
computationally efficient function. 

We include three such degree maps in Figure~\ref{multi:deg:fig}.
\begin{figure}
  \begin{tabular}{ccc}
    \includegraphics[width=0.3\textwidth]{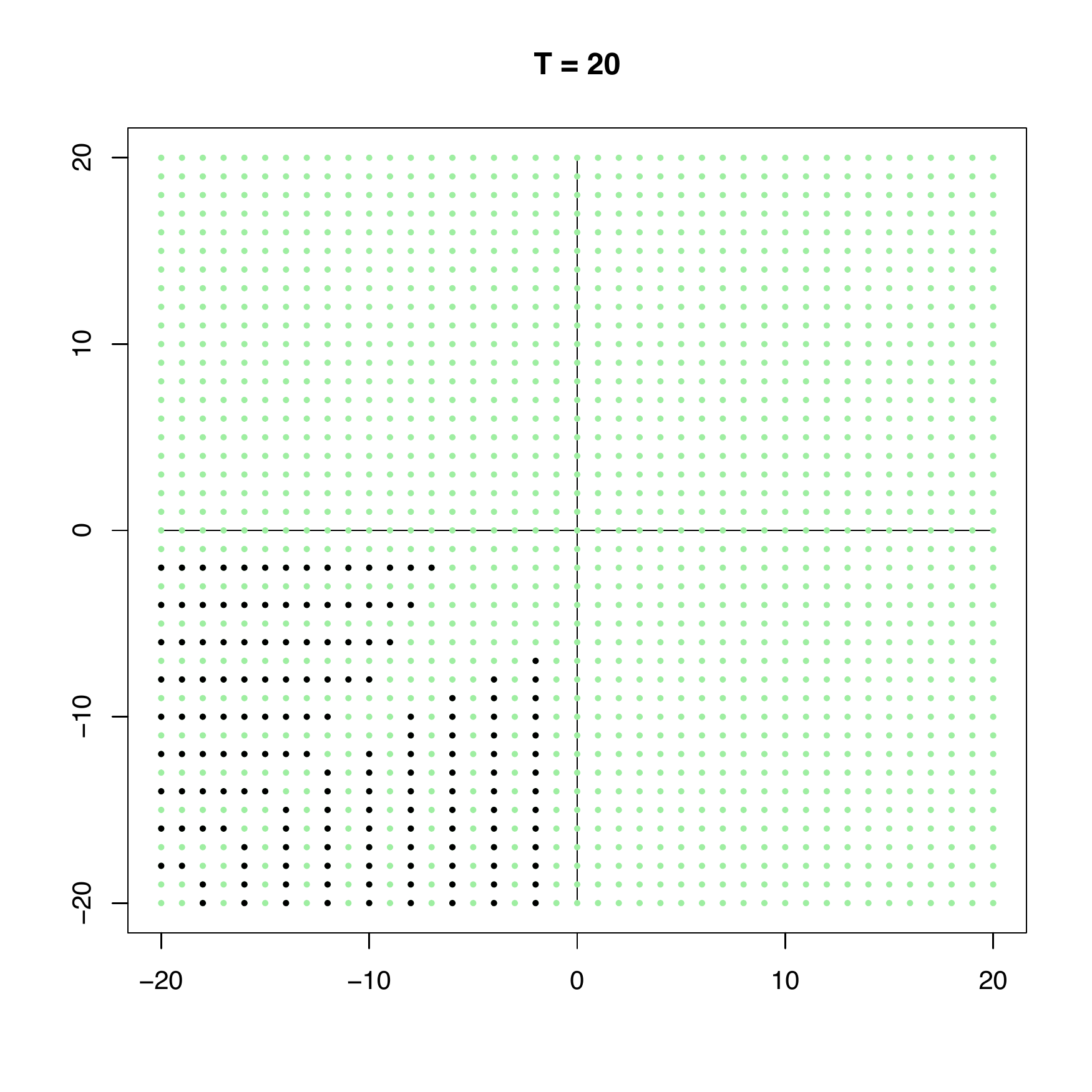} &
    \includegraphics[width=0.3\textwidth]{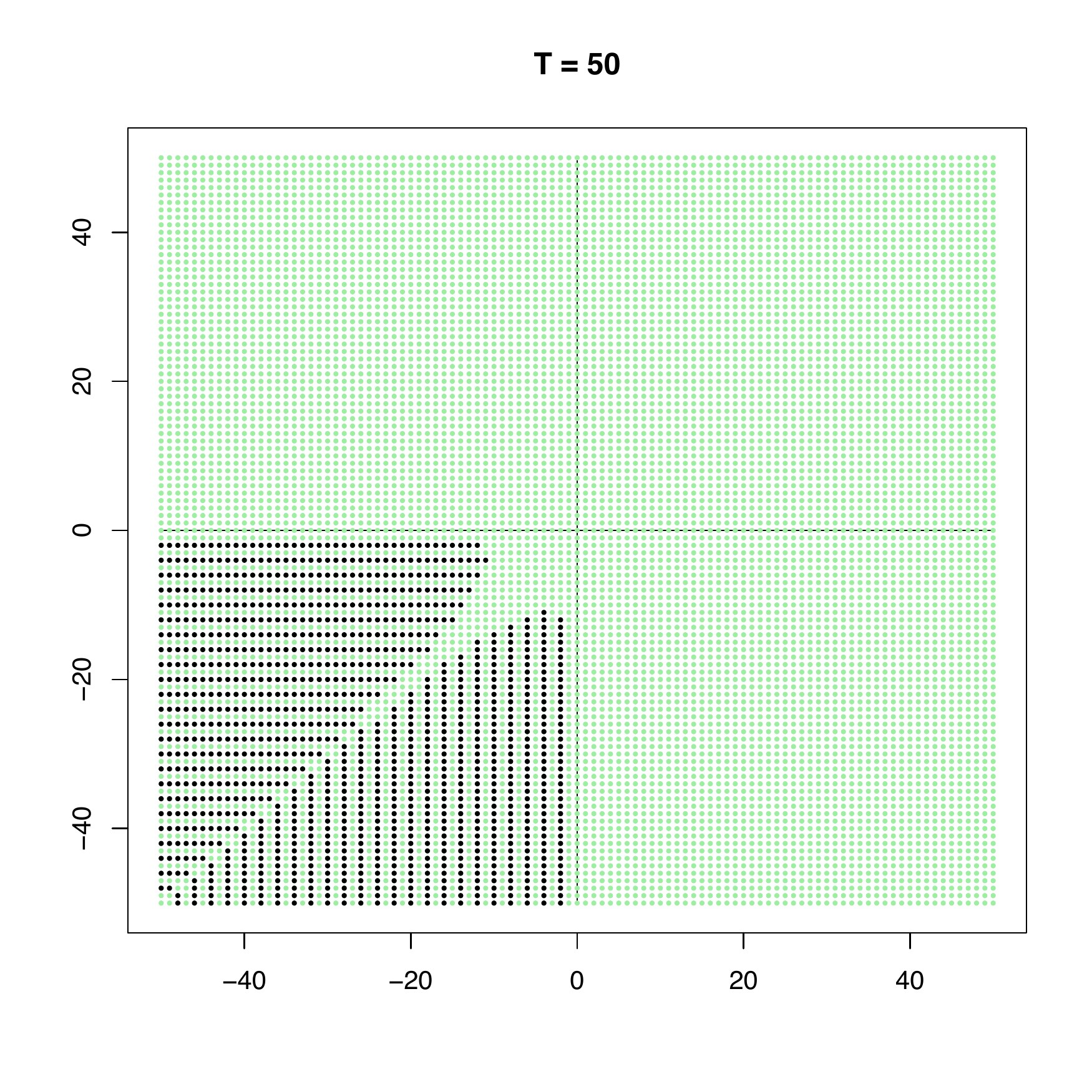} &
    \includegraphics[width=0.3\textwidth]{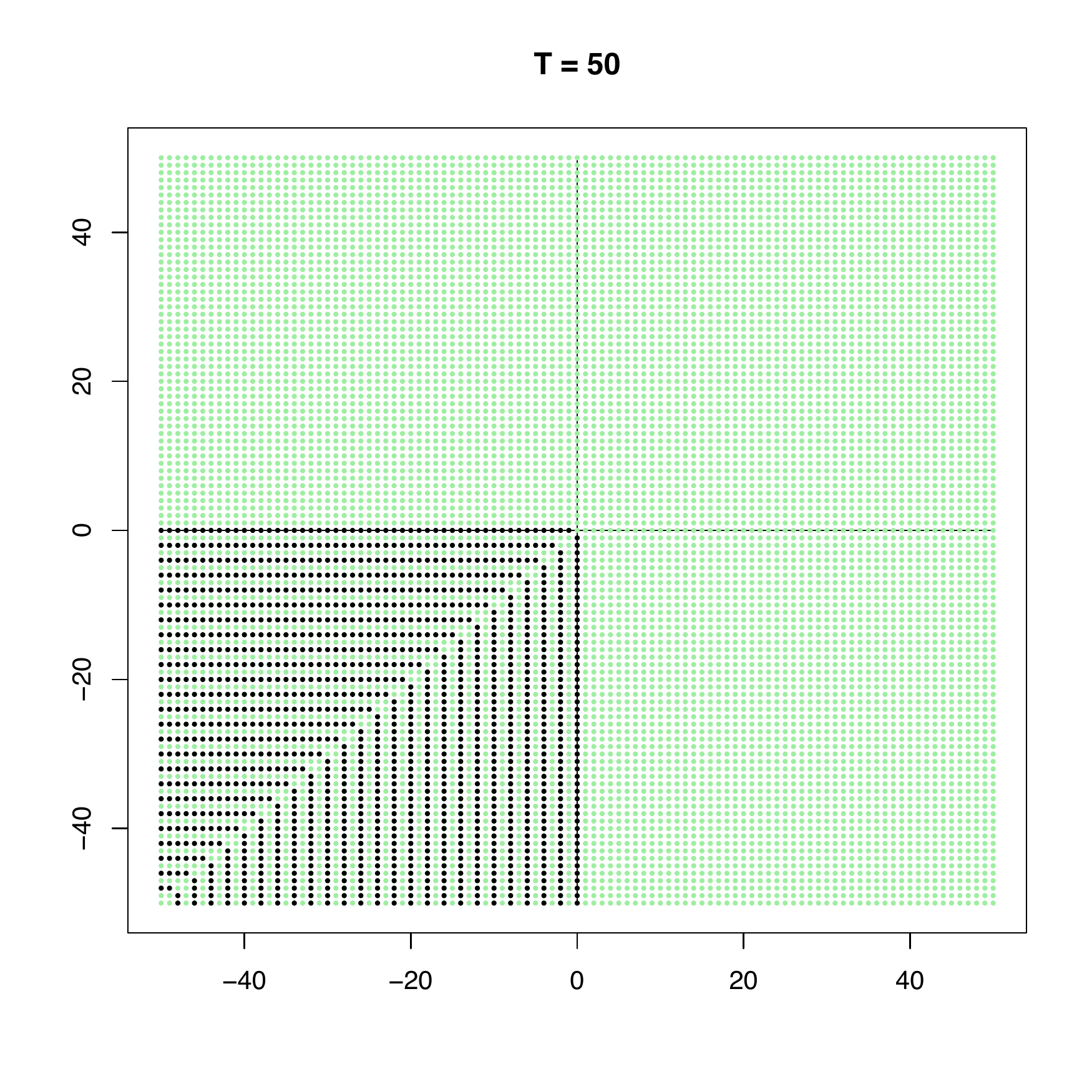}    
  \end{tabular}
  \caption[Degree plots with error as loss function]{Green pixels have
    degree 3, black pixels have degree 2. Each 
    step is diagonally down (left), and up (if $x < y$) and right (if
    $x > y$) and both when degree is 3. The rightmost figure uses
    $\gamma = 0.4$, and the other two $\gamma=0$. The loss function is
    0-1.}
   \label{multi:deg:fig}
\end{figure}
Only three
classes $k=3$ are used, and the loss function is 0-1 error. We also
fix the number $T$ of remaining rounds of boosting and the value of
the edge $\gamma$ for each plot. For ease of
presentation, the 3-dimensional states $\vs=(s_1,s_2,s_3)$ are
compressed into 2-dimensional pixel coordinates
$(u=s_2-s_1,v=s_3-s_2)$. It can be shown that this does not take away
information required to evaluate the potentials or the degree at
any pixel $(u,v)$. Further, only those states are considered whose
compressed coordinates $u,v$ lie in the range $[-T,T]$; in $T$ rounds,
these account for all the reachable states. The degrees are indicated
on the plot by colors. Our discussion in the previous sections implies
that the possible values of the degree is $2$ or $3$. When the degree
at a pixel $(u,v)$ is $3$, the pixel is colored green, and when the
degree is $2$, it is colored black.

Note that a random walk over the space $\vs\in\R^3$ consisting of
distributions over coordinate steps $\set{(1,0,0), (0,1,0), (0,0,1)}$
translates to a random walk over $(u,v)\in\R^2$ where each step lies
in the set $\set{(-1,-1), (1,0), (0,1)}$. In
Figure~\ref{multi:deg:fig}, these correspond to the directions
diagonally down, up or right. Therefore at a black pixel, the random
walk either chooses between diagonally down and  up, or between
 diagonally down and right, with probabilities
$\set{1/2+\gamma/2,1/2-\gamma/2}$. On the other hand, at a green
pixel, the random walk chooses among diagonally down, up and right
with probabilities
$(\gamma + (1-\gamma)/3,
(1-\gamma)/3,
(1-\gamma)/3)$.
The degree maps are shown for varying values of $T$ and the edge $\gamma$.
While some patterns emerge for the degrees, such as black or green
depending on the parity of $u$ or $v$, the authors found the region
near the line $u=v$ still too complex to admit any  solution apart
from a brute-force computation.

We also plot the potential values themselves in
Figure~\ref{multi:surf:fig} against different states.
\begin{figure}
  \begin{tabular}{cc}
    \includegraphics[width=0.4\textwidth]{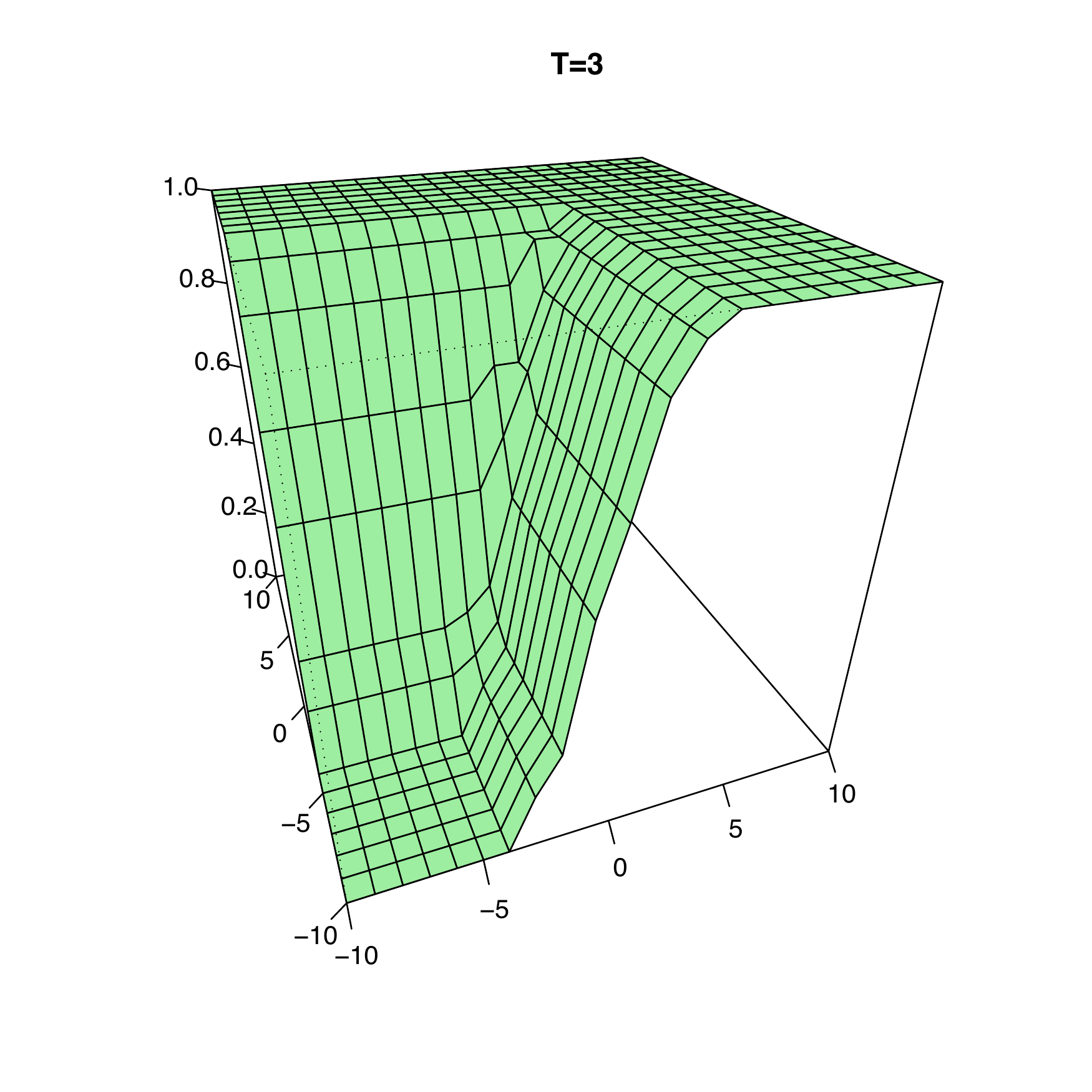} &
    \includegraphics[width=0.4\textwidth]{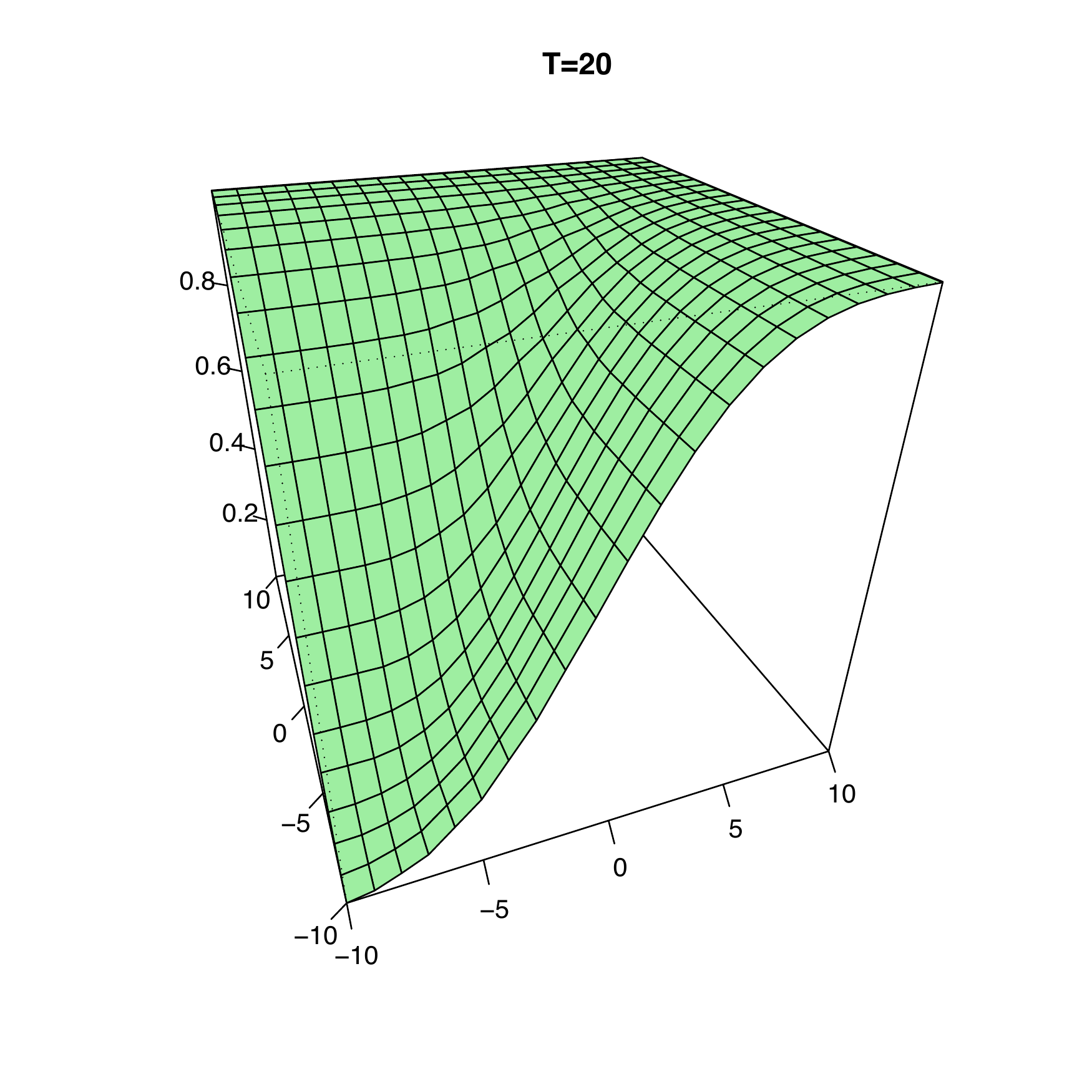}
  \end{tabular}
  \caption[Plot of potential values against states]{Optimum recurrence
    value. We set $\gamma=0$. Surface is 
    irregular for smaller values of $T$, but smoother for larger
    values, admitting hope for approximation.}
    \label{multi:surf:fig}
\end{figure}
In each plot,
the number of iterations remaining, $T$, is held constant, the number of classes is
chosen to be $3$, and the edge $\gamma=0$. The states are compressed
into pixels as before, and the potential is plotted against each
pixel, resulting in a 3-dimensional surface. We include two plots,
with different values for $T$. The surface is irregular for $T=3$
rounds, but smoother for $20$ rounds, admitting some hope for approximation. 

An alternative approach would be to approximate the potential $\phi_t$
by the potential $\phi_t^{\vb}$ for some fixed $\vb\in\dgam$
corresponding to some particular edge-over-random condition. Since
$\phi_t \geq \phi_t^{\vb}$ for all edge-over-random distributions
$\vb$, it is natural to approximate by choosing $\vb$ that maximizes
the fixed edge-over-random potential. (It can be shown that this $\vb$
corresponds to the $\gamma$-biased uniform distribution.) Two plots of
comparing the potential values at $\vzero$, $\phi_T(\vzero)$ and
$\max_{\vb}\phi^{\vb}_T(\vzero)$, which correspond to the respective error upper
bounds, is shown in Figure~\ref{multi:phi_psi:fig}.
\begin{figure}
  \begin{tabular}{cc}
    \includegraphics[width=0.4\textwidth]{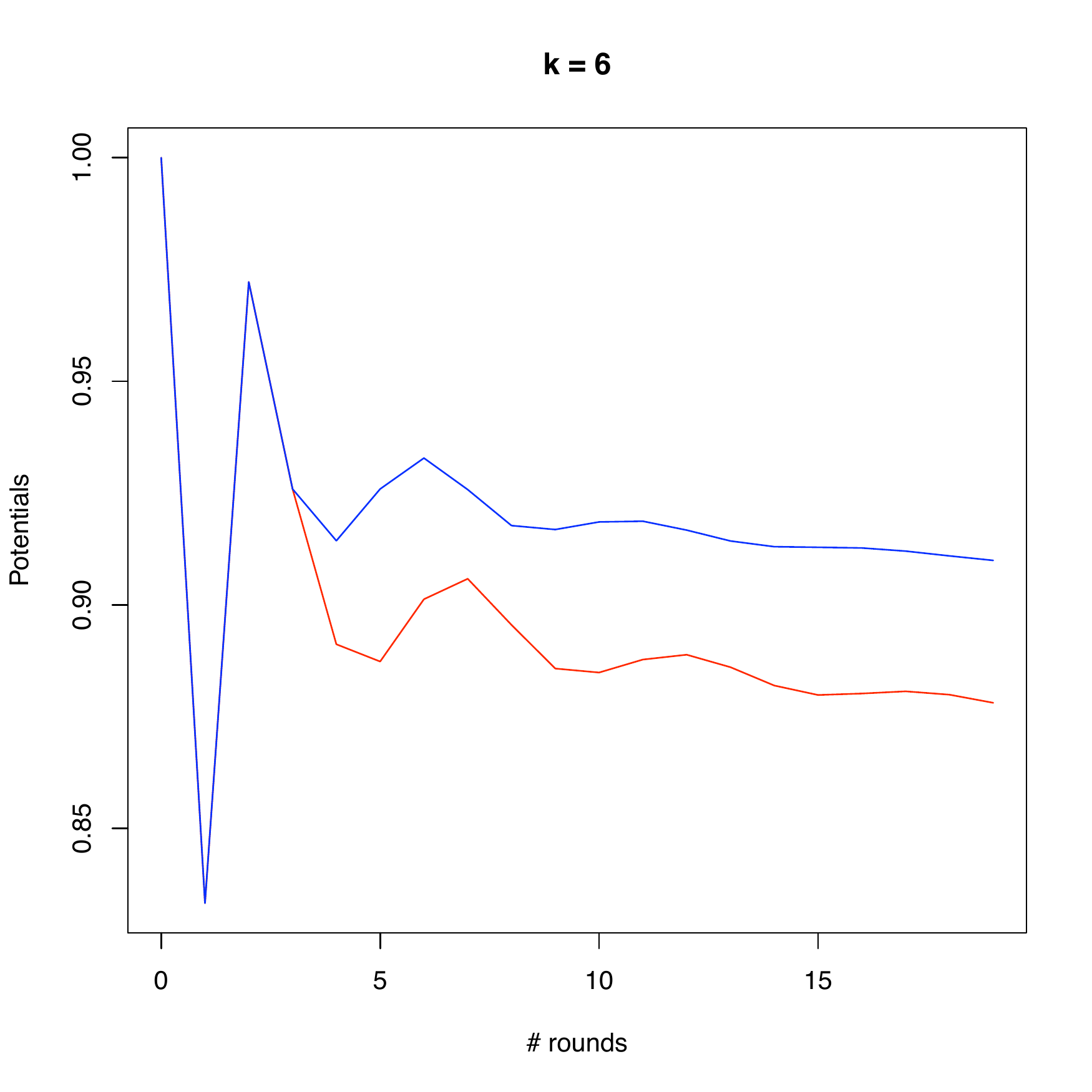} &
    \includegraphics[width=0.4\textwidth]{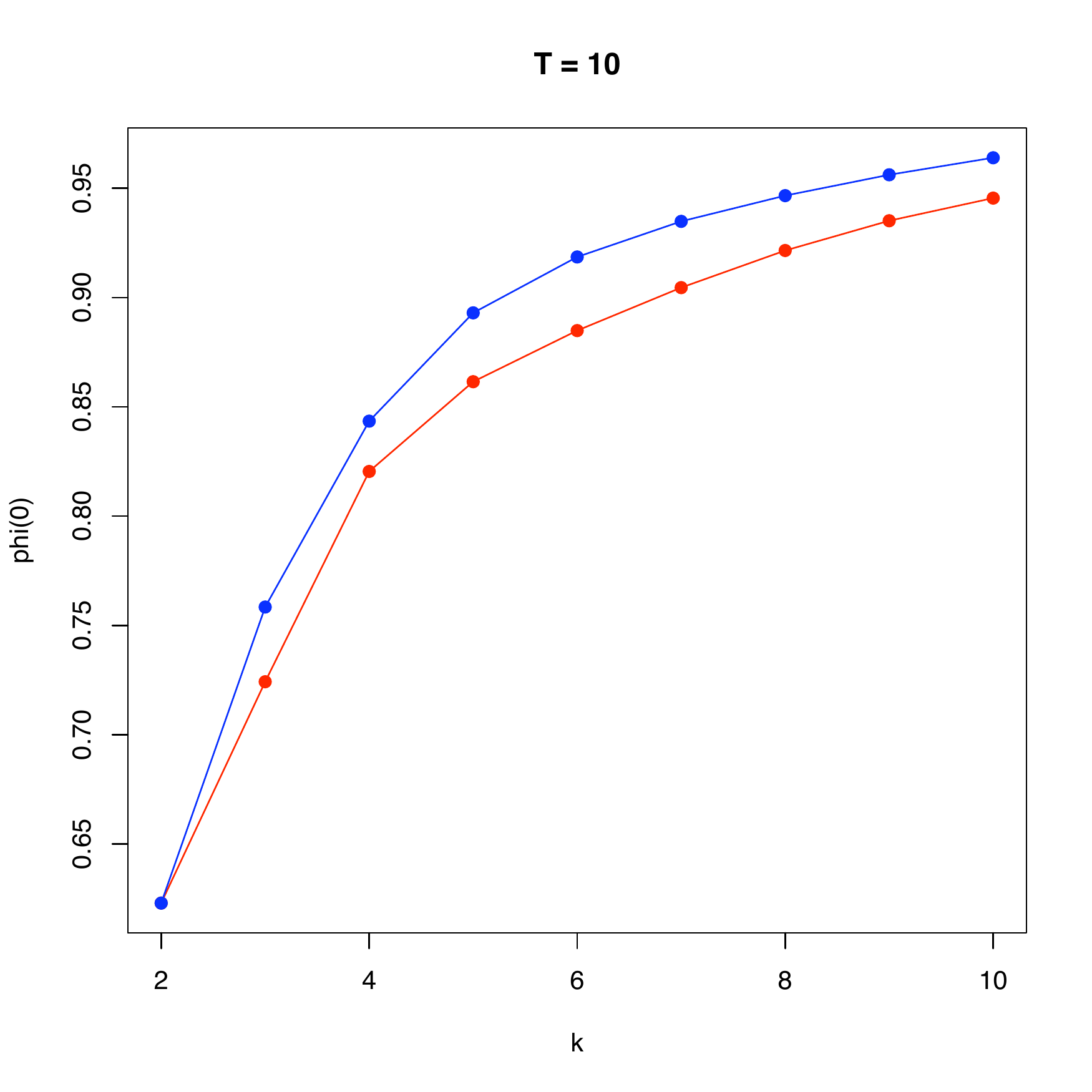}
  \end{tabular}
  \caption[Potential values with minimal and edge-over-random
  conditions]{Comparison of $\phi_t(\vzero)$ (blue) with 
    $\max_{\vq}\phi_t^{\vq}(\vzero)$ (red) over different rounds $t$
    and different number of classes $k$. We set $\gamma=0$ in both.}
  \label{multi:phi_psi:fig}
\end{figure}
In the first plot,
the number of classes $k$ is held fixed at $6$, and the values are
plotted for different values of iterations $T$. In the second plot,
the number of classes vary, and the number of iterations is held at
10. Both plots show that the difference in the values is significant,
and hence $\max_{\vb}\phi^{\vb}_T(\vzero)$ would be a rather
optimistic upper bound on the error when using the minimal weak
learning condition.

If we use exponential loss \eqref{multi:exploss:eqn},
the situation is not much better. The degree maps for varying values
of the weight parameter $\eta$ against fixed values of edge $\gamma=0.1$,
rounds remaining $T=20$ and number of classes $k=3$ are plotted in
Figure~\ref{multi:degexp:fig}.
\begin{figure}
  \begin{tabular}{cccc}
    \includegraphics[width=0.3\textwidth]{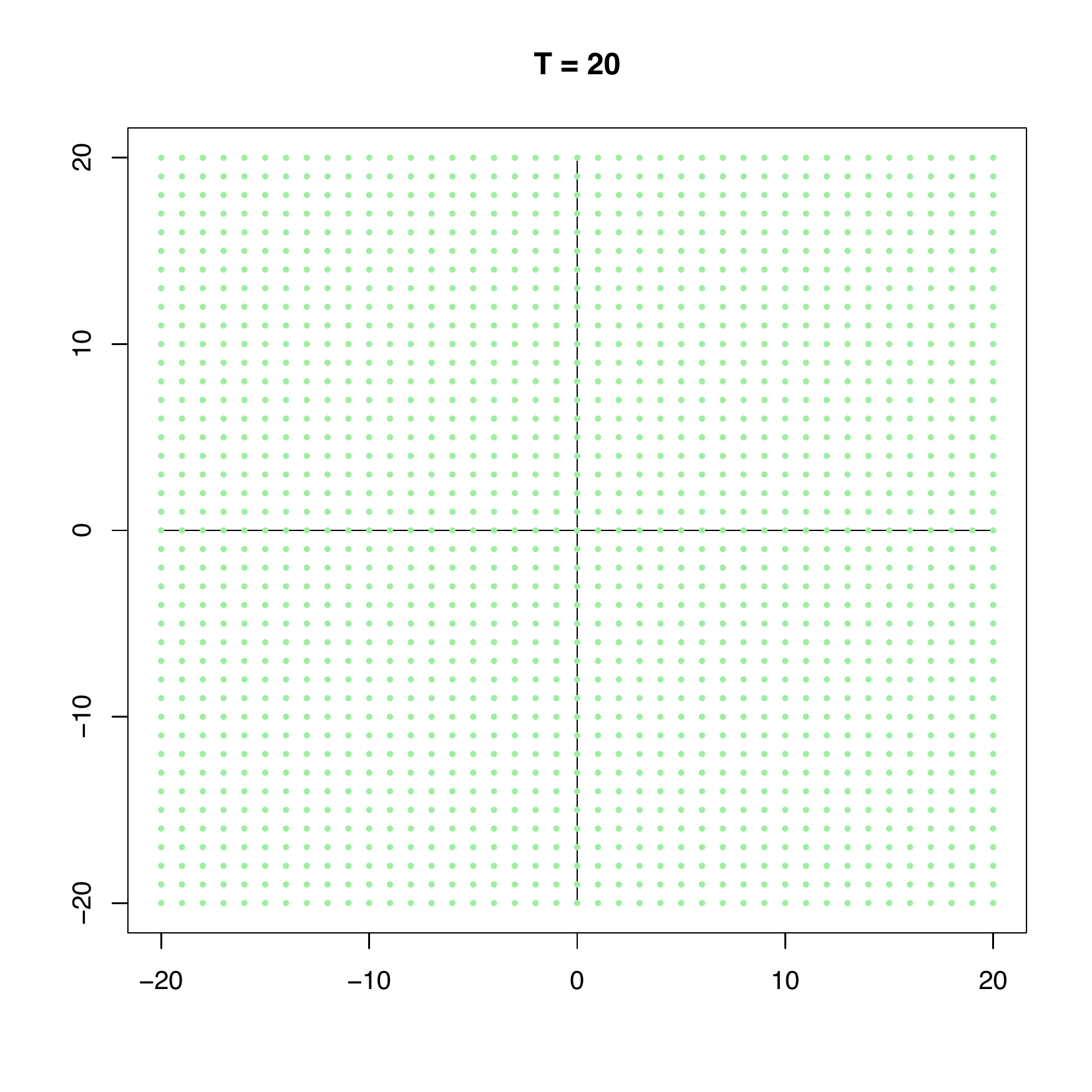} &
    \includegraphics[width=0.3\textwidth]{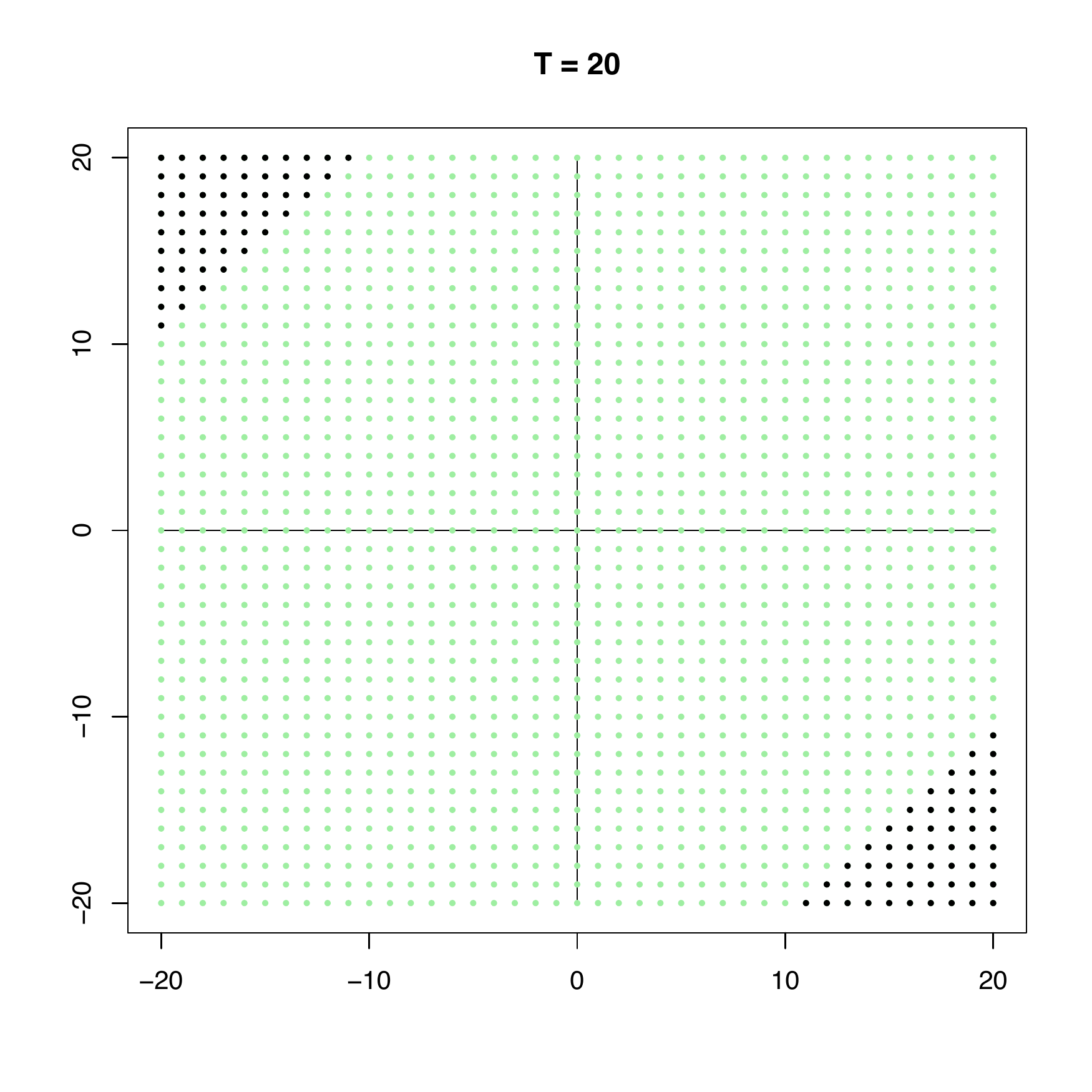} &
    \includegraphics[width=0.3\textwidth]{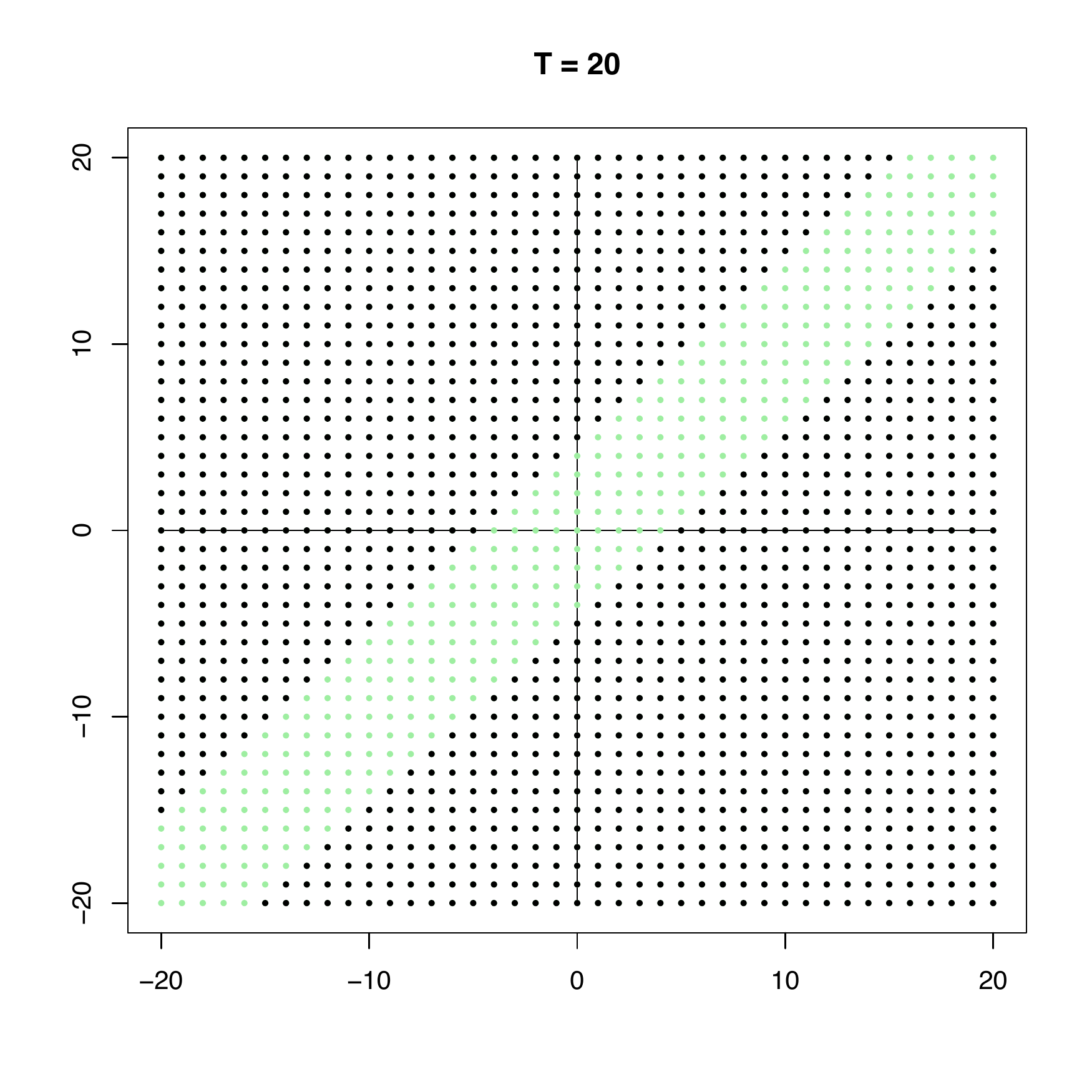}
  \end{tabular}
  \caption[Degree plots with the exponential loss function]{Green
    pixels have degree 3, black pixels have degree 2. Each 
    step is diagonally down (left), and up (if $x < y$) and right (if
    $x > y$) and both when degree is 3. Each plot uses
    $T=20,\gamma=0.1$. The values of $\eta$ are $0.08$, $0.1$ and
    $0.3$, respectively. With smaller values of $\eta$, more pixels
    have degree 3.}
    \label{multi:degexp:fig}
\end{figure}
Although the patterns are simple, with
the degree 3 pixels forming a diagonal band, we found it hard to prove
this fact formally, or compute the exact boundary of the band. However
the plots suggest that when $\eta$ is small, all pixels have degree
3. We next find conditions under which this opportunity for tractable
computation exists.

\paragraph{Efficient computation in special cases.}
Here we show that when using the exponential loss, if the edge
$\gamma$ is very small, then the potentials can be computed
efficiently. We first show an intermediate result. We already observed
empirically that when the weight parameter $\eta$ is small, the degrees all
become equal to $k$.  Indeed, we can prove this fact.
\begin{lemma}
  \label{multi:small_eta:lem}
  If the loss function being used is exponential loss
  \eqref{multi:exploss:eqn}
  and the weight parameter $\eta$ is small
  compared to the number of rounds
  \begin{equation}
    \label{multi:small_eta:eqn}
  \eta \leq \frac{1}{4}\min\set{\frac{1}{k-1}, \frac{1}{T}},
  \end{equation}
  then the optimal value
  of the degree $a$ in \eqref{multi:pot_simp:eqn} is always $k$. Therefore, in
  this situation, the
  potential $\phi_t$ using the minimal weak learning condition is the
  same as the potential $\phi_t^{\vu}$ using the $\gamma$-biased
  uniform distribution $\vu$,
  \begin{equation}
    \label{multi:ugam:eqn}
  \vu=\enc{\frac{1-\gamma}{k} + \gamma,
  \frac{1-\gamma}{k},
  \ldots,
  \frac{1-\gamma}{k}},
\end{equation}
and hence can be efficiently computed.
\end{lemma}
\begin{proof}
  We show $\phi_t = \phi_t^{\vu}$ by induction on the remaining number $t$
  of boosting iterations.
  The base case holds since, by definition, $\phi_0 = \phi_0^{\vu} =
  \Lexpe$.
  Assume, inductively that
  \begin{equation}
    \label{multi:small_eta_zero:eqn}
  \phi_{t-1}(\vs) = \phi_{t-1}^{\vu}(\vs)
  = \kappa(\gamma,\eta)^{t-1}\sum_{l=2}^ke^{\eta(s_l-s_1)},
  \end{equation}
  where the second equality follows from \eqref{multi:unifpot:eqn}.
  We show that
  \begin{equation}
    \label{multi:small_eta_one:eqn}
  \phi_t(\vs) = \E_{l\sim\vu}\enco{\phi_{t-1}(\vs + \ve_l)}.
  \end{equation}
  By the inductive hypothesis and \eqref{multi:fixed_dgsimp:eqn}, the
  right hand side of \eqref{multi:small_eta_one:eqn} is in fact equal to
  $\phi_t^{\vu}$, and we will have shown $\phi_t = \phi_t^{\vu}$. The
  proof will then follow by induction.
  
  In order to show \eqref{multi:small_eta_one:eqn}, by
  Lemma~\ref{multi:homognecsufsol:lem}, it suffices to show that 
  the optimal degree $a$ maximizing the right hand side of
  \eqref{multi:pot_simp:eqn} is 
  always $k$:
  \begin{equation}
    \label{multi:akineq:eqn}
  \E_{l\sim \vb^{\pi}_a}\enco{\phi_{t-1}\enc{\vs + \ve_l}}
  \leq
  \E_{l\sim \vb^{\pi}_k}\enco{\phi_{t-1}\enc{\vs + \ve_l}}.
  \end{equation}
  By \eqref{multi:small_eta_zero:eqn}, $\phi_{t-1}\enc{\vs+\ve_{l_0}}$ may
  be written as $\phi_{t-1}(\vs) + \kappa(\gamma,\eta)^{t-1}\cdot \xi_{l_0}$,
  where the term $\xi_{l_0}$ is:
  \[
  \xi_{l_0} =
  \begin{cases}
    (e^\eta-1)e^{\eta(s_{l_0} - s_1)} & \mbox { if } l_0 \neq 1, \\
    (e^{-\eta}-1)\sum_{l=2}^ke^{\eta(s_l-s_1)} & \mbox { if } l_0 = 1.
  \end{cases}
  \]
  Therefore \eqref{multi:akineq:eqn} is the same as:
  $
  \E_{l\sim \vb^{\pi}_a}\enco{\xi_l}
  \leq
  \E_{l\sim \vb^{\pi}_k}\enco{\xi_l}
  $.
  Assume (by relabeling if necessary) that $\pi$ is the identity
  permutation on coordinates $\set{2,\ldots,k}$. 
  Then the expression $\E_{l\sim \vb^{\pi}_a}\enco{\xi_l}$ can be
  written as
  \begin{eqnarray*}
    \E_{l\sim \vb^{\pi}_a}\enco{\xi_l} &=&
    \enc{\frac{1-\gamma}{a} + \gamma}\xi_1
    + \sum_{l=k-a+2}^k \enc{\frac{1-\gamma}{a}}\xi_l \\
    &=&
    \gamma\xi_1 +
    (1-\gamma)
    \enct{\frac{\xi_1 + \sum_{l=k-a+2}^k\xi_l}{a}}.
  \end{eqnarray*}
  It suffices to show that the term in curly brackets is maximized
  when $a=k$. We will in fact show that the numerator of the term is
  negative if $a<k$, and non-negative for $a=k$, which will complete
  our proof.
  Notice that the numerator can be written as
  \begin{eqnarray*}
    &&
    (e^\eta-1)\enct{\sum_{l=k-a+2}^ke^{\eta(s_{l} - s_1)}}
   -
   (1-e^{-\eta})\sum_{l=2}^ke^{\eta(s_l-s_1)}\\
   &=&
   (e^\eta-1)\enct{\sum_{l=k-a+2}^ke^{\eta(s_{l} -
       s_1)}-\sum_{l=2}^ke^{\eta(s_l-s_1)}}
   +
   \enct{(e^\eta-1)-(1-e^{-\eta})}
   \sum_{l=2}^ke^{\eta(s_l-s_1)}\\
   &=&
   \enct{e^\eta+e^{-\eta}-2}
   \sum_{l=2}^ke^{\eta(s_l-s_1)}
   -
   (e^\eta-1)\enct{
     \sum_{l=2}^{k-a+1}e^{\eta(s_l-s_1)}}.
 \end{eqnarray*}
 When $a=k$, the second summation disappears, and we are left with a
 non-negative expression.
 However when $a<k$, the second summation is
 at least $e^{\eta\enc{s_2-s_1}}$.
 Since $t\leq T$, and in $t$ iterations
 the absolute value of any state coordinate $\abs{s_t(l)}$ is at most
 $T$, the first summation is at most $(k-1)e^{2\eta T}$ and the
 second summation is at least $e^{-2\eta T}$. Therefore 
 the previous expression is at most
 \begin{eqnarray*}
   &&
   (k-1)\enc{e^{\eta}+e^{-\eta}-2}e^{2\eta T}
   -
   (e^\eta-1)e^{-2\eta T} \\
   &=&
   (e^\eta-1)e^{-2\eta T}
   \enct{(k-1)(1-e^{-\eta})e^{4\eta T} -1}.
 \end{eqnarray*}
 We show that the term in curly brackets is negative.
 Firstly, using $e^x \geq 1+x$, we have
 $1-e^{-\eta} \leq \eta \leq 1/(4(k-1))$ by choice of
 $\eta$. Therefore it suffices to show that $e^{4\eta T} < 4$. By
 choice of $\eta$ again, $e^{4\eta T} \leq e^{1} < 4$.
 This completes our proof.
\end{proof}
The above lemma seems to suggest that under certain conditions, a sort
of degeneracy occurs, and the
optimal Booster payoff \eqref{multi:payoff:eqn} is nearly unaffected
by whether we use the minimal weak learning condition, or the
condition $(\Csig,\Ugam)$. Indeed, we next prove this fact.
\old{
\begin{theorem}
  \label{multi:same_wlc:thm}
  Suppose the loss function is as in
  Lemma~\ref{multi:small_eta:lem}, and for some parameter $\eps > 0$,
  the number of examples $m$ is large enough
  \begin{equation}
    \label{multi:mlargetwo:eqn}
    m \geq \frac{Te^{1/4}}{\eps}.
  \end{equation}
  Consider the minimal weak learning
  condition \eqref{multi:minwl:eqn}, and the fixed edge-over-random
  condition $(\Csig,\Ugam)$ corresponding to the $\gamma$-biased
  uniform baseline $\Ugam$.
  Then the optimal booster payoffs using
  either condition is within $\eps$ of each other.
\end{theorem}
}
\begin{proof}
  \old{
  We show that the OS strategies arising out of either condition is
  the same.
  In other words, at any iteration $t$ and state $\vs_t$,
  both strategies play the same cost matrix and enforce the same
  constraints on the response of Weak-Learner.
  The theorem will then
  follow if we can invoke Theorems~\ref{multi:dgstrat:thm} and
  \ref{multi:mindg:thm}.
  For that, the number of examples needs to be as large as in
  \eqref{multi:mlarge:eqn}.
  The required largeness would follow from \eqref{multi:mlargetwo:eqn}
  if the loss function satisfied \eqref{multi:lossvar:eqn} with
  $\diameter(L,T)$ at most $\exp(1/4)$.
  Since the largest discrepancy in losses between two states reachable
  in $T$ iterations is at most $e^{\eta T} - 0$, the bound follows
  from the choice of $\eta$ in \eqref{multi:small_eta:eqn}.
  Therefore, it suffices to show the equivalence of the OS strategies
  corresponding to the two weak learning conditions.
  }

  We first show both strategies play the same cost-matrix.
  Lemma~\ref{multi:small_eta:lem} states that the potential function
  using the minimal weak learning condition is the same as
  when using the fixed condition $(\Csig,\Ugam)$: $\phi_t =
  \phi^{\vu}_t$, where $\vu$ is as in \eqref{multi:ugam:eqn}. Since,
  according to \eqref{multi:fixed_ctopt:eqn} and
  \eqref{multi:ctopt:eqn}, 
  given a state $\vs_t$ and iteration $t$, the two
  strategies choose cost matrices that are identical functions of the
  respective potentials, by the equivalence of the potential
  functions, the resulting cost matrices must be the same.

  Even with the same cost matrix, the two different conditions could
  be imposing different constraints on Weak-Learner, which might
  affect the final payoff. For instance,
  with the baseline $\Ugam$, Weak-Learner has to return a weak
  classifier $h$ satisfying
  \[
  \C_t\bullet\vh \leq \C_t\bullet\Ugam,
  \]
  whereas with the minimal condition, the constraint on $h$ is
  \[
  \C_t\bullet\vh \leq \max_{\B\in\Bgam}\C_t\bullet\B.
  \]
  In order to show that the constraints are the same we therefore need
  to show that for the common cost matrix $\C_t$ chosen, the right
  hand side of the two previous expressions are the same:
  \begin{equation}
    \label{multi:max_is_ugam:eqn}
  \C_t\bullet\Ugam = \max_{\B\in\Bgam}\C_t\bullet\Bgam.
  \end{equation}
  We will in fact show the stronger fact that the equality holds for every row
  separately:
  \begin{equation}
    \label{multi:same_wlc_one:eqn}
  \forall i: \dotp{\C_t(i)}{\vu} = \max_{\vb\in\dgam}\dotp{\C_t(i)}{\vb}.
  \end{equation}
  To see this, first observe that the choice of the optimal
  cost matrix $\C_t$ in \eqref{multi:ctopt:eqn} implies the following
  identity
  \[
  \dotp{\C_t(i)}{\vb}
  = \E_{l\sim\vb}\enco{\phi_{T-t-1}(\vs_t(i)+ \ve_l)}.
  \]
  On the other hand, \eqref{multi:pot_simp:eqn} and
  Lemma~\ref{multi:small_eta:lem} together imply that the distribution
  $\vb$ maximizing the right hand side of the above is the
  $\gamma$-biased uniform distribution, from which
  \eqref{multi:same_wlc_one:eqn} follows.
  Therefore, the constraints placed on Weak-Learner by the
  cost-matrix $\C_t$ is the same whether we
  use minimum weak learning condition or the
  fixed condition $(\Csig,\Ugam)$.
\end{proof}

One may wonder why $\eta$ would be chosen so small, especially since
the previous theorem indicates that such choices for $\eta$ lead to
degeneracies.
To understand this, recall that $\eta$ represents the size of the
weights $\alpha_t$ chosen in every round, and was introduced as a
tunable parameter to help achieve the best possible upper bound on
zero-one error. More 
precisely, recall that the exponential loss $\Lexpe(\vs)$ of the
unweighted state, defined in \eqref{multi:exploss:eqn}, is equal to the
exponential loss $\Lexp(\f)$ on the weighted state, defined in
\eqref{multi:exploss_varstate:eqn}, which in turn is an upper bound on 
the error $\Lzero(\f_T)$ of the final weighted state $\f_T$. Therefore
the potential value $\phi_T(\vzero)$ based on the exponential loss
$\Lexpe$ is an upper bound on the minimum 
error attainable after $T$ rounds of boosting. 
At the same time, $\phi_T(\vzero)$ is a function of $\eta$.
Therefore, we
 may tune this parameter to attain the best bound possible.
 Even with this motivation, it may seem that a properly tuned $\eta$ will
 not be as small as in Lemma~\ref{multi:small_eta:lem}, especially
 since it can be shown that the resulting loss bound $\phi_T(\vzero)$
 will always be larger than a fixed constant (depending on 
 $\gamma,k$), no matter how many rounds $T$ of boosting is used.
 However, the next result identifies conditions under which the tuned
 value of $\eta$ is indeed as small as in Lemma~\ref{multi:small_eta:lem}.
 This happens when the edge $\gamma$ is very small, as is described in
 the next theorem. Intuitively, a weak classifier achieving small edge
 has low accuracy, and a low weight reflects Booster's lack of
 confidence in this classifier.
\begin{theorem}
  \label{multi:small_edge:thm}
  When using the exponential loss function \eqref{multi:exploss:eqn},
  and the minimal weak learning condition \eqref{multi:minwl:eqn},
  the loss upper bound $\phi_T(\vzero)$ provided by
  Theorem~\ref{multi:mindg:thm} is more than $1$ and hence trivial
  unless the parameter $\eta$ is chosen sufficiently small compared to
  the edge $\gamma$:
  \begin{equation}
    \label{multi:eta_gam:eqn}
    \eta \leq \frac{k\gamma}{1-\gamma}.
  \end{equation}
  In particular, when the edge is very small:
  \begin{equation}
    \label{multi:small_gam:eqn}
  \gamma \leq
  \min\set{
    \frac{1}{2},
    \frac{1}{8k}
    \min\set{\frac{1}{k},\frac{1}{T}}
    },
  \end{equation}
  the value of $\eta$ needs to be as small as in
  \eqref{multi:small_eta:eqn}. 
\end{theorem}
\begin{proof}
  Comparing solutions \eqref{multi:rws:eqn} and
  \eqref{multi:fixed_rw:eqn} to the potentials corresponding to the
  minimal weak learning condition and a fixed edge-over-random
  condition, we may conclude that the loss bound $\phi_T(\vzero)$ is
  in the former case is larger than $\phi_T^{\vb}(\vzero)$, for any
  edge-over-random distribution $\vb\in\dgam$. In particular, when
   $\vb$ is set to be the $\gamma$-biased uniform distribution $\vu$, as defined
  in \eqref{multi:ugam:eqn}, we get $\phi_T(\vzero) \geq
  \phi_T^{\vu}(\vzero)$. Now the latter bound, according to
  \eqref{multi:unifpot:eqn}, is $\kappa(\gamma,\eta)^T$, where
  $\kappa$ is defined as in \eqref{multi:kappa:eqn}. Therefore, to get
  non-trivial loss bounds which are at most $1$, we need to choose $\eta$
  such that $\kappa(\gamma,\eta) \leq 1$. By \eqref{multi:kappa:eqn},
  this happens when
  \begin{eqnarray*}
  \enc{1-e^{-\eta}}\gamma &\geq&
  \enc{e^\eta+e^{-\eta}-2}\enc{\frac{1-\gamma}{k}} \\
  \mbox{ i.e., }
  \frac{k\gamma}{1-\gamma}
  &\geq&
  \frac{e^\eta+e^{-\eta}-2}{1-e^{-\eta}}
  = e^\eta-1 \geq \eta.
\end{eqnarray*}
Therefore \eqref{multi:eta_gam:eqn} holds.
When $\gamma$ is as small as in \eqref{multi:small_gam:eqn}, then
$1-\gamma \leq \frac{1}{2}$, and therefore, by
\eqref{multi:eta_gam:eqn}, the bound on $\eta$ becomes 
identical to that in \eqref{multi:small_gam:eqn}.
\end{proof}
The condition in the previous theorem, that of the existence of only a
very small edge, is the most we can assume for most practical
datasets. Therefore, in such situations, we can compute the optimal
Booster strategy that uses the minimal weak learning conditions.
More importantly, using this result, we derive, in the next section,
a highly efficient and practical \emph{adaptive} algorithm,
that is, one that does not require any prior knowledge about the edge
$\gamma$, and will therefore work with any dataset.

\section{Variable edges}
\label{multi:adapt:sec}
So far we have required Weak-Learner to beat random by at least a
fixed amount $\gamma > 0$ in each round of the boosting game. In
reality, the edge over random is larger initially, and gets smaller as
the OS algorithm creates harder cost matrices. Therefore requiring a
fixed edge is either unduly pessimistic or overly optimistic. If the
fixed edge is too small, not enough progress is made in the initial
rounds, and if the edge is too large, Weak-Learner fails to meet the
weak-learning condition in latter rounds. We fix this by not making
any assumption about the edges, but instead 
\emph{adaptively} responding to the edges returned by 
Weak-Learner. In the rest of the section we describe the adaptive
procedure, and the resulting loss bounds guaranteed by it.

The philosophy behind the adaptive algorithm is a boosting game where
Booster and Weak Learner no longer have opposite goals, but cooperate
to reduce error as fast as possible.
However, in order to create a clean abstraction and separate
implementations of the boosting algorithms and the weak learning
procedures as much as possible, we assume neither of the players has
any knowledge of the details of the algorithm employed by the other player.
In particular Booster may only assume that Weak Learner's strategy is
barely strong enough to guarantee boosting. Therefore, Booster's
demands on the weak classifiers returned by Weak Learner should be
minimal, and it should send the weak learning algorithm the
``easiest'' cost matrices that will ensure boostability.
In turn, Weak Learner may only assume a very weak Booster strategy,
and therefore return a weak classifier that performs as well as
possible with respect to the cost matrix sent by Booster.

At a high level, the adaptive strategy proceeds as follows. At any
iteration, based on the states of the examples and number of remaining
rounds of boosting, Booster chooses the game-theoretically optimal
cost matrix assuming only infinitesimal edges in the remaining
rounds.
Intuitively, Booster has no high expectations of Weak Learner,
and supplies it the easiest cost matrices with which it may be able to
boost. 
However, in the adaptive setting, Weak-Learner is no longer adversarial.
Therefore, although only infinitesimal edges are anticipated by
Booster, Weak Learner cooperates in returning weak classifiers that
achieve as large edges as possible, which will be more than just
inifinitesimal.
Based on the exact edge received in each round, Booster chooses the
weight $\alpha_t$ adaptively to reach the most favourable state possible.
Therefore, Booster plays game theoretically assuming an adversarial
Weak Learner and expecting only the smallest edges in the future rounds,
although Weak Learner actually cooperates, and Booster adaptively
exploits this favorable behavior as much as possible.
This way the boosting algorithm remains robust to a poorly performing
Weak Learner, and yet can make use of a powerful weak learning
algorithm whenever possible.

We next describe the details of the adaptive procedure.
With variable weights we need to work with the weighted state
$\f_t(i)$ of each example $i$, defined in \eqref{multi:var_state:eqn}.
To keep the compuations tractable, we will only be
working with the exponential loss $\Lexp(\f)$ on the weighted states.
We first
describe how Booster chooses the cost-matrix in each round. Following
that we describe how it adaptively computes the weights in each round
based on the edge of the weak classifier received.

\paragraph{Choosing the cost-matrix.}
As discussed before, at any iteration $t$ and state $\f_t$ Booster
assumes that it will receive an infinitesimal edge $\gamma$ in each of the
remaining rounds. Since the step size is a function of the edge, which
in turn is expected to be the same tiny value in each round, we may
assume that the step size in each round will also be some fixed
value $\eta$. We are therefore in the setting of
Theorem~\ref{multi:small_edge:thm}, which states that 
the parameter $\eta$ in the exponential loss function
\eqref{multi:exploss:eqn} should also be tiny to get any non-trivial
bound.
But then the loss function satisfies the conditions in
Lemma~\ref{multi:small_eta:lem}, and by
Theorem~\ref{multi:same_wlc:thm}, the game theoretically optimal
strategy remains the same whether we use the minimal condition or
$(\Csig,\Ugam)$. When using the latter condition, 
the optimal choice of the cost-matrix at iteration $t$ and state
$\f_t$, according to \eqref{multi:unifsol2:eqn}, is 
\begin{equation}
  \label{multi:smalleta:eqn}
  C_t(i,l) =
  \begin{cases}
    \enc{e^{\eta}-1}e^{f_{t-1}(i,j)-f_{t-1}(i,1)}& \mbox { if } l>1, \\
    \enc{e^{-\eta}-1}\sum_{j=2}^k e^{f_{t-1}(i,j)-f_{t-1}(i,1)} & \mbox { if
    } l=1.
  \end{cases}
\end{equation}
Further, when using the condition $(\Csig,\Ugam)$,
the average potential of the states $\f_t(i)$, according to
\eqref{multi:unifpot:eqn}, 
is given by the average loss 
\eqref{multi:avgexploss:eqn} of the state times $\kappa(\gamma,\eta)^{T-t}$,
where the function $\kappa$ is defined in
\eqref{multi:kappa:eqn}.
Our goal is to choose $\eta$ as a function of $\gamma$ so that
$\kappa(\gamma,\eta)$ is as small as possible.
Now, there is no lower bound on how small the edge $\gamma$ may get,
and, anticipating the worst, it makes sense to choose an infinitesimal
$\gamma$, in the spirit of~\citep{Freund01}.
Eq.~\eqref{multi:kappa:eqn} then implies that the
choice of $\eta$ should also be infinitesimal.
Then the above choice of the cost matrix becomes the following
(after some rescaling):
\begin{eqnarray}
  \label{multi:adapt_copt:eqn}
  C_t(i,l) &=& \lim_{\eta\to 0} C_\eta(i,l) \eqdef \frac{1}{\eta}
\begin{cases}
  \enc{e^{\eta}-1}e^{f_{t-1}(i,j)-f_{t-1}(i,1)}& \mbox { if } l>1, \\
  \enc{e^{-\eta}-1}\sum_{j=2}^k e^{f_{t-1}(i,j)-f_{t-1}(i,1)} & \mbox { if
  } l=1.
\end{cases} \nonumber \\
&=&
\begin{cases}
  e^{f_{t-1}(i,j)-f_{t-1}(i,1)}& \mbox { if } l>1, \\
  -\sum_{j=2}^k e^{f_{t-1}(i,j)-f_{t-1}(i,1)} & \mbox { if } l=1.
\end{cases}
\end{eqnarray}
We have therefore derived the optimal cost matrix played by the
adaptive boosting strategy, and we record this fact.
\begin{lemma}
  \label{multi:adapt_copt:lem}
  Consider the boosting game using the minimal weak learning condition
  \eqref{multi:minwl:eqn}.
  Then, in iteration $t$ at state $\f_t$, the game-theoretically
  optimal Booster strategy chooses the cost matrix $\C_t$ given in
  \eqref{multi:adapt_copt:eqn}. 
\end{lemma}
We next show how to adaptively choose the weights $\alpha_t$.

\paragraph{Adaptively choosing weights.}
Once Weak Learner returns a weak classifier $h_t$, Booster chooses the
optimum weight $\alpha_t$ so that the resulting states $\f_t = \f_{t-1}
+ \alpha_t\vh_t$ are as favorable as possible, that is, minimizes the
total potential of its states.
By our previous discussions, these are proportional to the total loss
given by
$
Z_t = \sum_{i=1}^m\sum_{l=2}^ke^{f_t(i,l) - f_t(i,1)}.
$
For any choice of $\alpha_t$, the difference $Z_t - Z_{t-1}$ between
the total loss in   
rounds $t-1$ and $t$ is given by 
\begin{eqnarray*}
&& \enc{e^{\alpha_t}-1}
\sum_{i\in S_-}e^{f_{t-1}(i,h_t(i)) - f_{t-1}(i,1)}
-\enc{1-e^{-\alpha_t}}
\sum_{i\in S_+}\Lexp(\f_{t-1}(i))\\
&=& \enc{e^{\alpha_t}-1} A^t_-
-
\enc{1-e^{-\alpha_t}} A^t_+\\
&=& \enc{A^t_+e^{-\alpha_t} + A^t_-e^{\alpha_t}}
- \enc{A^t_+ + A^t_-},
\end{eqnarray*}
where $S_+$ denotes the set of examples that $h_t$ classifies
correctly, $S_-$ the incorrectly classified examples, and $A^t_-,A^t_+$
denote the first and second summations, respectively.
Therefore, the task of choosing $\alpha_t$ can be cast as a simple
optimization problem minimizing the previous expression.
In fact, the optimal value of $\alpha_t$ is given by the following
closed form expression
\begin{equation}
  \label{multi:alphaone:eqn}
\alpha_t = \frac{1}{2}\ln\enc{\frac{A^t_+}{A^t_-}}.
\end{equation}
With this choice of weight, one can show (with some straightforward
algebra) that the total loss of the 
state falls by a factor less than 1. In fact the factor is
exactly
\begin{equation}
  \label{multi:optfac:eqn}
(1-c_t) - \sqrt{c_t^2 - \delta_t^2},
\end{equation}
where
\begin{equation}
  \label{multi:ct:eqn}
c_t = (A^t_+ + A^t_-)/Z_{t-1},
\end{equation}
and $\delta_t$ is the edge of the
returned classifier $h_t$ on the supplied cost-matrix $\C_t$.
Notice that the quantity $c_t$ is at most $1$, and hence the factor
\eqref{multi:optfac:eqn} can be upper bounded by
$\sqrt{1-\delta_t^2}$.
We next show how to compute the edge $\delta_t$. The definition of the
edge depends on the weak learning condition being used, and in this
case we are using the minimal condition
\eqref{multi:minwl:eqn}. Therefore the edge $\delta_t$ is the largest
$\gamma$ such that the following still holds
\[
\C_t\bullet\vh \leq \max_{\B\in\Bgam}\C_t\bullet\B.
\]
However, since $\C_t$ is the optimal cost matrix when using
exponential loss with a tiny value of $\eta$, we can use arguments 
in the proof of Theorem~\ref{multi:same_wlc:thm} to simplify the
computation. In particular, eq. \eqref{multi:max_is_ugam:eqn}  implies
that the edge $\delta_t$ may be computed as the largest $\gamma$
satisfying the following simpler inequality
\begin{eqnarray}
\delta_t &=& \sup\set{\gamma:\C_t\bullet\mat{1}_{h_t} \leq
  \C_t\bullet\Ugam}
\nonumber \\
&=& \sup\set{\gamma:\C_t\bullet\mat{1}_{h_t} \leq
  -\gamma\sum_{i=1}^m\sum_{l=2}^ke^{f_{t-1}(i,l) - f_{t-1}(i,1)}}
\nonumber \\
\implies
\delta_t &=& \gamma:\;
\C_t\bullet\mat{1}_{h_t} =
-\gamma\sum_{i=1}^m\sum_{l=2}^ke^{f_{t-1}(i,l) - f_{t-1}(i,1)}
\nonumber \\
  \label{multi:deltat:eqn}
\implies
\delta_t &=&
\frac{-\C_t\bullet\mat{1}_{h_t}}
{\sum_{i=1}^m\sum_{l=2}^ke^{f_{t-1}(i,l) - f_{t-1}(i,1)}}
=\frac{-\C_t\bullet\mat{1}_{h_t}}{Z_t},
\end{eqnarray}
where the first step follows by expanding $\C_t\bullet\Ugam$.
We have therefore an adaptive strategy which efficiently reduces
error. We record our results.
\begin{lemma}
  \label{multi:adaone:thm}
  If the weight $\alpha_t$ in each round is chosen as in
  \eqref{multi:alphaone:eqn}, and the edge $\delta_t$ is given by
  \eqref{multi:deltat:eqn}, then the total loss $Z_t$ falls by the
  factor given in \eqref{multi:optfac:eqn}, which is at most
  $\sqrt{1-\delta_t^2}$.
\end{lemma}

The choice of $\alpha_t$ in \eqref{multi:alphaone:eqn} is optimal, but depends
on quantities other than just the edge $\delta_t$. We next show a way
of choosing $\alpha_t$ based only on $\delta_t$ that still causes the
total loss to drop by a factor of $\sqrt{1-\delta_t^2}$. 
\begin{lemma}
\label{multi:adaptcal:lem}
  Suppose cost matrix $\C_t$ is chosen as in~\eqref{multi:adapt_copt:eqn},
  and the returned weak classifier $h_t$ has edge $\delta_t$
  i.e.  $\C_t\bullet \mat{1}_{h_t} \leq \C_t \bullet
  \U_{\delta_t}$. Then choosing any weight $\alpha_t>0$ for $h_t$
  makes the loss $Z_t$
  at most a   factor
  \[
  1-\frac{1}{2}(e^{\alpha_t}-e^{-\alpha_t})\delta_t +
  \frac{1}{2}(e^{\alpha_t} + e^{-\alpha_t} - 2)
  \] of the previous loss $Z_{t-1}$.
  In particular by choosing
  \begin{equation}
    \label{multi:alphatwo:eqn}
    \alpha_t = \frac{1}{2}\ln\enc{\frac{1+\delta_t}{1-\delta_t}},
  \end{equation}
  the drop factor is at most $\sqrt{1-\delta_t^2}$.
\end{lemma}
\begin{proof}
  We borrow notation from earlier discussions.
  The edge-condition implies
  \begin{eqnarray*}
    A^t_{-}-A^t_{+} = \C_t \bullet
    \mat{1}_{h_t} \leq
    \C_t\bullet
    \U_{\delta_t} = -\delta_t Z_{t-1}
    \implies
    A^t_{+}-A^t_{-} \geq     \delta_tZ_{t-1}.
  \end{eqnarray*}
  On the other hand, the drop in loss after choosing $h_t$ with weight
  $\alpha_t$ is
  \begin{eqnarray*}
    &&\enc{1-e^{-\alpha_t}}A^t_+ -
     \enc{e^{\alpha_t}-1}A^t_- \\
     &=&
    \enc{\frac{e^{\alpha_t} - e^{-\alpha_t}}{2}}
    \enc{A^t_+-A^t_-} 
    -
    \enc{\frac{e^{\alpha_t} + e^{-\alpha_t} - 2}{2}}
    \enc{A^t_+ + A^t_-}.
  \end{eqnarray*}
  We have already shown that $A^t_+ - A^t_- \geq \delta_t Z_{t-1}$.
  Further, $A^t_+ + A^t_-$ is at most $Z_{t-1}$.
  Therefore the loss drops by a factor of at least
  \[
  1-\frac{1}{2}(e^{\alpha_t}-e^{-\alpha_t})\delta_t +
  \frac{1}{2}(e^{\alpha_t} + e^{-\alpha_t} - 2)
  = \frac{1}{2}\enct{(1-\delta_t)e^{\alpha_t} +
  (1+\delta_t)e^{-\alpha_t}}.
  \]
  Tuning
  $\alpha_t$ as in \eqref{multi:alphatwo:eqn}
  causes the drop factor to be at least $\sqrt{1-\delta_t^2}$.
\end{proof}
\begin{algorithm}
  \caption{AdaBoost.MM}
  \label{multi:adapt:algo}
  \begin{algorithmic}
    \REQUIRE Number of classes $k$, number of examples $m$.
    \REQUIRE Training set $\set{(x_1,y_1),\ldots,(x_m,y_m)}$
    with $y_i\in\set{1,\ldots,k}$ and $x_i\in X$.

    \medskip

    \STATE $\bullet$ Initialize $m\times k$ matrix $f_0(i,l) = 0$ for
    $i=1,\ldots,m$, and $l=1,\ldots,k$. 
    \FOR {$t=1$ to $T$}
    \STATE $\bullet$ Choose cost matrix $\C_t$ as follows:
    \begin{eqnarray*}
      C_t(i,l) 
      &=&
      \begin{cases}
        e^{f_{t-1}(i,l)-f_{t-1}(i,y_i)}& \mbox { if } l\neq y_i, \\
        -\sum_{l\neq y_i} e^{f_{t-1}(i,j)-f_{t-1}(i,y_i)} & \mbox { if } l=1.
      \end{cases}
    \end{eqnarray*}
    \STATE $\bullet$ Receive weak classifier
    $h_t:X\to\set{1,\ldots,k}$ from weak learning algorithm 
    \STATE $\bullet$ Compute edge $\delta_t$ as follows:
    \[
    \delta_t =
    \frac{-\sum_{i=1}^mC_t(i,h_t(x_i))}
    {\sum_{i=1}^m\sum_{l\neq y_i}e^{f_{t-1}(i,l) - f_{t-1}(i,y_i)}}
    \]
    \STATE $\bullet$ Choose $\alpha_t$ either as
    \begin{equation}
      \label{multi:approx_step:eqn}
    \alpha_t = \frac{1}{2}\ln\enc{\frac{1+\delta_t}{1-\delta_t}},
    \end{equation}
    \hspace{5pt} or, for a slightly bigger drop in the loss, as
    \begin{equation}
      \label{multi:exact_step:eqn}
    \alpha_t = \frac{1}{2}\ln\enc{
      \frac{\sum_{i:h_t(x_i)=y_i}\sum_{l\neq y_i}e^{f_{t-1}(i,l)-f_{t-1}(i,y_i)}}
      {\sum_{i:h_t(x_i)\neq y_i}e^{f_{t-1}(i,h_t(x_i)) - f_{t-1}(i,y_i)}}}
    \end{equation}
    \STATE $\bullet$ Compute $\f_t$ as:
    \[
    f_t(i,l) = f_{t-1}(i,l) + \alpha_t\1\enco{h_t(x_i) = l}.
    \]
    \ENDFOR
    \STATE $\bullet$ Output weighted combination of weak classifiers
    $F_T:X\times\set{1,\ldots,k}\to\R$ defined as:
    \begin{equation}
      \label{multi:ft:eqn}
    F_T(x,l) = \sum_{t=1}^T \alpha_t\1\enco{h_t(x)=l}.
    \end{equation}
    \STATE $\bullet$ Based on $F_T$, output a classifier
    $H_T:X\to\set{1,\ldots,k}$ that predicts as
    \begin{equation}
      \label{multi:ht:eqn}
      H_T(x) = \argmax_{l=1}^kF_T(x,l).
    \end{equation}
  \end{algorithmic}
\end{algorithm}
Algorithm~\ref{multi:adapt:algo} contains pseudocode for the adaptive
algorithm, and includes both ways of choosing $\alpha_t$.
We call both versions of this algorithm AdaBoost.MM.
\old{
With the approximate way of choosing the step length in
\eqref{multi:approx_step:eqn}, AdaBoost.MM turns out to be identical
to AdaBoost.M2~\citep{FreundSc97} or AdaBoost.MR~\citep{SchapireSi99},
provided the weak 
classifier space is transformed in an appropriate way to be acceptable
by AdaBoost.M2 or AdaBoost.MR.
We emphasize that AdaBoost.MM and AdaBoost.M2 are products of very
different theoretical considerations, and this similarity should be
viewed as a coincidence arising because of the particular choice of
loss function, infinitesimal edge and approximate step size.
For instance, when the step sizes are chosen instead as in
\eqref{multi:exact_step:eqn}, the training error falls more rapidly,
and the resulting algorithm is different.
}

As a summary
of all the discussions in the section, we record the following
theorem.
\begin{theorem}
  \label{multi:adapt:thm}
  The boosting algorithm AdaBoost.MM, shown in
  Algorithm~\ref{multi:adapt:algo}, is 
  the optimal strategy for playing the adaptive boosting game, and is
  based on the minimal weak learning condition. Further
  if the edges returned in each round are $\delta_1,\ldots,\delta_T$,
  then the error after $T$ rounds is
  $(k-1)\prod_{t=1}^T
  \sqrt{1-\delta_t^2} \leq
  (k-1)\exp\enct{-(1/2)\sum_{t=1}^T\delta_t^2}$. 

  In particular, if a
  weak hypothesis space is used that satisfies the optimal weak
  learning condition \eqref{multi:minwl:eqn}, for some $\gamma$, then
  the edge in each round is large,  $\delta_t \geq \gamma$, and
  therefore the error after $T$ rounds is exponentially small,
  $(k-1)e^{-T\gamma^2/2}$. 
\end{theorem}

The theorem above states that as long as the minimal weak learning
condition is satisfied, the error will decrease exponentially
fast. Even if the condition is not satisfied, the error rate will keep
falling rapidly provided the edges achieved by the weak classifiers
are relatively high. However, our theory so far can provide no guarantees on these
edges, and therefore it is not clear what is the best error rate
achievable in this case, and how quickly it is achieved.
The assumptions of boostability, and hence our minimal weak learning
condition does not hold for the vast majority of practical datasets,
and as such it is important to know what happens in such settings.
In particular, an important requirement is \emph{empirical
  consistency}, where we 
want that for any given weak classifier space, the algorithm converge,
if allowed to run forever, to the weighted combination of classifiers
that minimizes error on the training set.
Another important criterion is \emph{universal consistency}, which
requires that the algorithm converge, when provided sufficient
training data, to the classifier combination that minimizes error on
the test dataset.
In the next section, we show
that AdaBoost.MM satisfies such consistency
requirements. Both the choice of the minimal weak learning condition
as well as the setup of the adaptive game framework will play
crucial roles in ensuring consistency.
These results therefore provide evidence that game theoretic
considerations can have strong statistical implications.

\section{Consistency of the adaptive algorithm}

The goal in a classification task is to design a classifier that
predicts with high accuracy on 
unobserved or test data. 
This is usually carried out by ensuring the classifier fits training
data well without being overly complex.
Assuming the training and test data are
reasonably similar, one can show that the above procedure achieves
high test accuracy, or is consistent.
Here we work in a probabilistic setting that connects
training and test data by assuming
both consist of examples and labels drawn from a common, unknown
distribution.

Consistency for multiclass classification in the probabilistic setting
has been studied by \citet{TewariBa07}, who show that, unlike in the
binary setting, many natural approaches fail to achieve consistency.
In this section, we show that 
AdaBoost.MM described in the previous section
avoids such pitfalls and enjoys various consistency results.
We begin by laying down some standard assumptions and setting up some
notation.
Then we prove our first result showing that our algorithm minimizes a
certain exponential loss function on the training data at a fast rate.
Next, we build upon this result and improve along two fronts: firstly
we change our metric from exponential loss to the more relevant
classification error metric, and secondly we
show fast convergence on not just training data, but also the test
set. 
For the proofs, we heavily reuse existing machinery in the
literature.

Throughout the rest of this section we consider the version of
AdaBoost.MM that picks weights according to the
approximate rule in \eqref{multi:approx_step:eqn}.
All our results most probably hold with the other rule
for picking weights in \eqref{multi:exact_step:eqn} as well, but we
did not verify that. 
These results hold without any boostability requirements on the space
$\H$ of weak classifiers, and are therefore widely applicable in practice. 
While we do not assume any weak learning condition, we will require a
fully cooperating Weak Learner.
In particular, we will require that in each round Weak Learner picks
the weak classifier suffering minimum cost with respect to the cost matrix
provided by the boosting algorithm, or equivalently achieves the
highest edge as defined in \eqref{multi:deltat:eqn}.
Such assumptions are both necessary and standard in the literature,
and are frequently met in practice.

In order to state our results, we will need to setup some notation.
The space of examples will be denoted by $\X$, and the set of labels
by $\Y = \set{1,\ldots,k}$.
We also fix a finite weak classifier space $\H$ consisting of
classifiers $h:\X\to\Y$.
We will be interested in functions $F:\X\times\Y\to\R$ that assign a
score to every example and label pair.
Important examples of such functions are the weighted majority
combinations \eqref{multi:ft:eqn} output by the adaptive algorithm.
In general, any such combination of the weak
classifiers in space $\H$ is 
specified by some weight function $\alpha:\H\to\R$; the resulting
function is denoted by $F_{\alpha}:\X\times\Y\to\R$, and satisfies:
\[
F_{\alpha}(x,l) = \sum_{h\in\H}\alpha(h) \1\enco{h(x)=l}.
\]
We will be interested in measuring the average exponential loss of
such functions. To measure this, we introduce the $\ersk$ operator:
\begin{equation}
  \label{multi:ersk:eqn}
   \ersk(F) \eqdef
  \frac{1}{m}
  \sum_{i=1}^m\sum_{l\neq y_i}e^{F(x_i,l) - F(x_i,y_i)}.
\end{equation}
With this setup, we can now state our simplest
consistency result, which ensures that the algorithm converges to a
weighted combination of classifiers in the space $\H$ that achieves
the minimum exponential loss over the training set at an efficient rate.
\begin{lemma}
  \label{multi:ersk_cons:lem}
  The $\ersk$ of the predictions
  $F_T$, as  defined in \eqref{multi:ft:eqn},
  converges to that of the optimal 
  predictions of any combination of the weak classifiers in $\H$ at
  the rate $O(1/T)$: 
  \begin{equation}
    \label{multi:ersk_cons:eqn}
    \ersk(F_T)
    -
    \inf_{\alpha:\H\to\R}
    \ersk(F_\alpha)
    \leq \frac{C}{T},
  \end{equation}
  where $C$ is a constant depending only on the dataset.
\end{lemma}
A slightly stronger result would state that the average exponential
loss when measured with respect to the \emph{test set}, and not just
the empirical set, also converges.
The test set is generated by some target
distribution $D$ over example label pairs, and
we introduce the $\rsk$ operator to measure the exponential loss for any
function $F:\X\times\Y\to\R$ with respect to $D$:
\[
\rsk(F) = \E_{(x,y)\sim D}\enco{\sum_{l\neq y}e^{F(x,l)-F(x,y)}}.
\]
We show this stronger result holds if the function $F_T$ is 
modified to the function $\cF_T:\X\times\Y\to\R$ that takes values in
the range $[0,-C]$, for some large constant $C$:
\begin{equation}
  \label{multi:cft:eqn}
\cF_T (x,l) \eqdef \max\set{-C, F_T(x,l) - \max_{l'} F_T(x,l')}.
\end{equation}
\begin{lemma}
  \label{multi:rsk_cons:lem}
  If $\cF_T$ is as in \eqref{multi:cft:eqn}, and the number of rounds
  $T$ is set to $T_m=\sqrt{m}$, then its $\rsk$ converges
  to the optimal value as $m\to\infty$ with high probability:
  \begin{equation}
    \label{multi:rsk_cons:eqn}
    \Pr\enco{
    \rsk\enc{\cF_{T_m}} \leq \inf_{F:\X\times\Y\to\R}\rsk(F) +
    O\enc{m^{-c}}}
   \geq 1 - \frac{1}{m^2},
  \end{equation}
  where $c>0$ is some absolute constant, and the probability is over
  the draw of training examples.
\end{lemma}
We prove Lemmas~\ref{multi:ersk_cons:lem} and
\ref{multi:rsk_cons:lem}  
by demonstrating a strong correspondence between
AdaBoost.MM and binary AdaBoost, and then
leveraging almost identical known consistency results for AdaBoost
\citep{BartlettTr07}. 
Our proofs will closely follow the exposition in
Chapter 12 of \citep{SchapireFr12} on the consistency of AdaBoost, and
are deferred to the appendix.

So far we have focused on $\rsk$, but 
a more desirable consistency result would state that
the test \emph{error} of the final classifier 
output by AdaBoost.MM converges
to the Bayes optimal error.
The test error is measured by the $\err$ operator, and is given
by
\begin{equation}
  \label{multi:err:eqn}
  \err(H) = \Pr_{(x,y)\sim D}\enco{H(x)\neq y}.
\end{equation}
The Bayes optimal classifier $\Hopt$ is a classifier achieving the
minimum error among all possible classifying functions
\begin{equation}
  \label{multi:bayes_opt:eqn}
\err(\Hopt)
= \inf_{H:\X\to\Y}\err(H),
\end{equation}
and we want our algorithm to output a classifier whose $\err$
approaches $\err(\Hopt)$.
In designing the algorithm, our main focus was on reducing
the exponential loss, captured by $\rsk$ and $\ersk$. Unless these
loss functions are aligned properly with classification error, we
cannot hope to achieve optimal error. The next result shows that our
loss functions are correctly aligned, or more technically \emph{Bayes
  consistent}. In other words, if a scoring function
$F:\X\times\Y\to\R$ is close to achieving optimal $\rsk$, then the
classifier $H:\X\to\Y$ derived from it as follows:
\begin{equation}
  \label{multi:h_f:eqn}
H(x) \in \argmax_{l\in\Y}F(x,y),
\end{equation}
also approaches the Bayes optimal error.
\begin{lemma}
  \label{multi:Bayes_cons:lem}
  Suppose $F$ is a scoring function achieving close to optimal risk
  \begin{equation}
    \label{multi:Bayes_cons_one:eqn}
  \rsk(F) \leq \inf_{F':\X\times\Y\to\R}\rsk(F') + \eps,
  \end{equation}
  for some $\eps\geq 0$.
  If $H$ is the classifier derived from it as in
  \eqref{multi:h_f:eqn}, then it achieves close to the Bayes optimal
  error 
  \begin{equation}
  \label{multi:Bayes_cons_two:eqn}
  \err(H) \leq \err(\Hopt) + \sqrt{2\eps}.
  \end{equation}
\end{lemma}
\begin{proof}
  The proof is similar to that of Theorem~12.1 in
  \citep{SchapireFr12}, which in turn is based on the work by
  \citet{Zhang04} and \citet{BartlettJoMc06}. 
  Let $p(x) = \Pr_{(x',y')\sim D}\enc{x'=x}$ denote the the
  marginalized probability of drawing 
  example $x$ from $D$, and let $p^x_y = \Pr_{(x',y')\sim
    D}\enco{y'=y|x'=x}$ denote the conditional probability of drawing
  label $y$ given we have drawn example $x$.
  We first rewrite the difference in errors between $H$ and $\Hopt$
  using these probabilities.
  Firstly note that the accuracy of any classifier $H'$ is given by
  \[
  \sum_{x\in\X}D(x,H'(x)) = \sum_{x\in\X}p(x)p^x_{H'(x)}.
  \]
  If $\X'$ is the set of examples where the
  predictions of $H$ and $\Hopt$ differ, $\X'=\set{x\in\X: H(x) \neq
    \Hopt(x)}$, then we may bound the error differences as
  \begin{equation}
    \label{multi:errbnd:eqn}
  \err(H) - \err(\Hopt) =
  \sum_{x\in\X'}p(x)\enc{p^x_{\Hopt(x)}-p^x_{H(x)}}.
  \end{equation}
  We next relate this expression to the difference of the losses.
  
  Notice that for any scoring function $F'$, the $\rsk$ can be
  rewritten as follows :
  \[
  \rsk(F') = \sum_{x\in\X}p(x)\sum_{l< l'}
  \enct{p^x_le^{F'(x,l') - F'(x,l)}
    + p^x_{l'}e^{F'(x,l) - F'(x,l')}}.
  \]
  Denote the inner summation in curly brackets by
  $L_{F'}^{l,l'}(x)$,
  and
  notice this quantity is minimized if
  \begin{eqnarray*}
    e^{F'(x,l) - F'(x,l')} =
    \sqrt{p^x_{l}/p^{x}_{l'}},
    &\mbox{ i.e., if }
    F'(x,l) - F'(x,l') =
    \frac{1}{2}\ln p^x_l
    - \frac{1}{2}\ln p^x_{l'}.
  \end{eqnarray*}
  Therefore, defining $\Fopt(x,l) =  \frac{1}{2}\ln p^x_l$ leads
  to a $\rsk$ minimizing function $\Fopt$. Furthermore, for any
  example and pair
  of labels $l,l'$, the quantity $L_{\Fopt}^{l,l'}(x)$ is at most
  $L_{F}^{l,l'}(x)$, and therefore the difference in losses of $\Fopt$
  and $F$ may be lower bounded as follows:
  \begin{eqnarray}
  \eps \geq \rsk(F) - \rsk(\Fopt)
  &=& \sum_{x\in\X}p(x)
  \sum_{l\neq l'}\enc{L_{F}^{l,l'}-L_{\Fopt}^{l,l'}} \nonumber \\
  \label{multi:Bayes_cons_three:eqn}
  &\geq& \sum_{x\in\X'}p(x)
  \enct{L_{F}^{H(x),\Hopt(x)}-L_{\Fopt}^{H(x),\Hopt(x)}}.
  \end{eqnarray}
  We next study the term in the curly brackets for a fixed $x$.
  Let $A$ and $B$ denote $H(x)$ and $\Hopt(x)$, respectively.
  We have already seen that $L_{\Fopt}^{A,B} = 2\sqrt{p^x_Ap^x_B}$.
  Further, by definition of Bayes optimality, $p^x_A \geq p^x_B$. 
  On the other hand, since $x\in\X'$, we know that $B\neq A$, and
  hence, $F(x,A) \geq F(x,B)$. Let $e^{F(x,B) - F(x,A)} = 1+\eta$, for
  some $\eta\geq 0$.
  The quantity $L_{F}^{A,B}$ may be lower bounded as:
  \begin{eqnarray*}
  L_{F}^{A,B} &=&
  p^x_Ae^{F(x,B) - F(x,A)}
  + p^x_{B}e^{F(x,A) - F(x,B)}\\
  &=& (1+\eta)p^x_A + (1+\eta)^{-1}p^x_B\\
  &\geq& (1+\eta)p^x_A + (1-\eta)p^x_B \\
  &=& p^x_A + p^x_B + \eta(p^x_A - p^x_B)
  \geq p^x_A + p^x_B.
\end{eqnarray*}
Combining we get
\[
L_F^{A,B} - L^{A,B}_{\Fopt} \geq
p^x_A + p^x_B - 2\sqrt{p^x_Ap^x_B} =
\enc{\sqrt{p^x_A} -  \sqrt{p^x_B}}^2.
\]
Plugging back into \eqref{multi:Bayes_cons_three:eqn} we get
\begin{equation}
  \label{multi:Bayes_cons_four:eqn}
\sum_{x\in\X'}p(x)\enc{\sqrt{p^x_{H(x)}} -  \sqrt{p^x_{\Hopt(x)}}}^2
\leq \eps.
\end{equation}
Now we connect \eqref{multi:errbnd:eqn} to the previous expression as
follows
\begin{eqnarray}
  \lefteqn{
  \enct{\err(H)-\err(\Hopt)}^2}
  \nonumber \\
  &=&
  \enct{\sum_{x\in\X'}p(x)\enc{p^x_{\Hopt(x)}-p^x_{H(x)}}}^2
  \nonumber \\
  &\leq&
  \enc{\sum_{x\in\X'}p(x)}
  \enc{\sum_{x\in\X'}p(x)
    \enc{p^x_{\Hopt(x)}-p^x_{H(x)}}^2}
  \mbox{ (Cauchy-Schwartz) }
  \nonumber \\
  &\leq&
  \label{multi:Bayes_cons_six:eqn}  
  \sum_{x\in\X'}p(x)
  \enc{\sqrt{p^x_{\Hopt(x)}}-\sqrt{p^x_{H(x)}}}^2
  \enc{\sqrt{p^x_{\Hopt(x)}}+\sqrt{p^x_{H(x)}}}^2  \\
  \label{multi:Bayes_cons_five:eqn}  
  &\leq&
  2\sum_{x\in\X'}p(x)
  \enc{\sqrt{p^x_{\Hopt(x)}}-\sqrt{p^x_{H(x)}}}^2
  \\
  &\leq&
  2\eps, \mbox{ (by \eqref{multi:Bayes_cons_four:eqn})}
  \nonumber
\end{eqnarray}
where \eqref{multi:Bayes_cons_six:eqn} holds since
\[
\sum_{x\in\X'}p(x) = \Pr_{(x',y')\sim D}\enco{x'\in\X'} \leq 1,
\]
and \eqref{multi:Bayes_cons_five:eqn} holds since
\begin{eqnarray*}
  p^x_{H(x)}+p^x_{\Hopt(x)} &=& \Pr_{(x',y')\sim
  D}\enco{y'\in\set{H(x),\Hopt(x)} | x}
\leq 1 \\
\implies
\sqrt{p^x_{H(x)}}+\sqrt{p^x_{\Hopt(x)}} &\leq& \sqrt{2}.
\end{eqnarray*}
Therefore, $\err(H) - \err(\Hopt) \leq \sqrt{2\eps}$.
\end{proof}
Note that the classifier $\cH_T$, derived from the truncated scoring function
$\cF_T$ in the manner provided in \eqref{multi:h_f:eqn}, makes
identical predictions to, and hence has the same $\err$ as, the
classifier $H_T$ output by the adaptive algorithm.
Further, Lemma~\ref{multi:rsk_cons:lem} seems to suggest that $\cF_T$
satisfies the condition in \eqref{multi:Bayes_cons_one:eqn}, which,
combined with our previous observation $\err(H)=\err(\cH_T)$,
would imply $H_T$ approaches the optimal error.
However, the condition \eqref{multi:Bayes_cons_one:eqn} requires
achieving optimal 
risk over all scoring functions, and not just ones achievable as a
combination of weak classifiers in $\H$.
Therefore, in order to use Lemma~\ref{multi:Bayes_cons:lem}, we 
require the weak classifier space to be sufficiently rich, so that
some combination of the weak
classifiers in $\H$  
attains $\rsk$ arbitrarily close to the minimum attainable by any
function:
\begin{eqnarray}
  \label{multi:richness:eqn}
  \inf_{\alpha:\H\to\R}\rsk(F_\alpha)
 =
\inf_{F:\X\times\Y\to\R}\rsk(F).
\end{eqnarray}
The richness condition, along with our previous arguments and
Lemma~\ref{multi:rsk_cons:lem}, immediately imply the following
result. 
\begin{theorem}
  \label{multi:err_convg:thm}
  If the weak classifier space $\H$ satisfies the richness condition
  \eqref{multi:richness:eqn}, and the number of rounds
  $T$ is set to $\sqrt{m}$, then the error of the final classifier $H_T$
  approaches the Bayes optimal error:
  \begin{equation}
    \label{multi:err_convg:eqn}
    \Pr\enco{
    \err\enc{H_{\sqrt{m}}} \leq \err(\Hopt) +
    O\enc{m^{-c}}}
   \geq 1 - \frac{1}{m^2},
  \end{equation}
  where $c>0$ is some positive constant, and the probability is over
  the draw of training examples.  
\end{theorem}
A consequence of the theorem is our strongest consistency result:
\begin{corollary}
  \label{multi:cons:cor}
  Let $\Hopt$ be the Bayes optimal classifier, and let the weak
  classifier space $\H$ 
  satisfy the richness condition \eqref{multi:richness:eqn}.
  Suppose $m$ example and label pairs
  $\set{(x_1,y_1),\ldots,(x_m,y_m)}$ are sampled from the distribution
  $D$, the number of rounds $T$ is set to be $\sqrt{m}$, and these are
  supplied to AdaBoost.MM.
  Then, in the limit $m\to\infty$,
  the final classifier $H_{\sqrt{m}}$ output by AdaBoost.MM achieves
  the Bayes optimal error almost surely:
  \begin{equation}
    \label{multi:cons:eqn}
    \Pr\enco{\enct{\lim_{m\to\infty}\err(H_{\sqrt{m}})}
      = \err(\Hopt)} = 1,
  \end{equation}
  where the probability is over the randomness due to the draw of
  training examples. 
\end{corollary}
The proof of Corollary~\ref{multi:cons:cor}, based on the
Borel-Cantelli Lemma,
is very similar to that of Corollary~12.3 in \citep{SchapireFr12}, and so
we omit it.
When $k=2$, AdaBoost.MM is identical to
AdaBoost. For Theorem~\ref{multi:err_convg:thm} to hold for AdaBoost, the
richness assumption \eqref{multi:richness:eqn} is necessary, since
there are examples due to \citet{LongSe10} showing that the theorem
may not hold when that assumption is violated.

Although we have seen that technically
AdaBoost.MM is consistent under broad
assumptions, intuitively perhaps it is not clear what properties were
responsible for this desirable behavior. We next briefly study the
high level ingredients necessary for consistency in boosting
algorithms.

\paragraph{Key ingredients for consistency.}
We show here how both the
choice of the loss function as well as the weak learning condition
play crucial roles in ensuring consistency. If the loss function were
not Bayes consistent as in Lemma~\ref{multi:Bayes_cons:lem}, driving
it down arbitrarily could still lead to high test error. For example,
the loss employed by SAMME~\citep{ZhuZoRoHa09} does not upper bound
the error, and therefore although it can manage to drive down its loss
arbitrarily when supplied by the dataset discussed in
Figure~\ref{multi:samme:fig}, although its error remains
high.

Equally important is the weak learning condition. Even if the loss
function is chosen to be error, so that it is trivially Bayes
consistent, choosing the wrong weak learning condition could lead to
inconsistency. In particular, if the weak learning condition is
stronger than necessary, then, even on a boostable dataset where the
error can be driven to zero, the boosting algorithm may get stuck
prematurely because its stronger than necessary demands cannot be met by the weak
classifier space. We have already seen theoretical examples of such
datasets, and we will see some practical instances of this phenomenon
in the next section. 

On the other hand, if the weak learning condition
is too weak, then a lazy Weak Learner may satisfy the Booster's demands
by returning weak classifiers belonging only to a non-boostable subset
of the available weak classifier space. For instance, consider again
the dataset in Figure~\ref{multi:samme:fig}, and assume that this time
the weak classifier space is much richer, and consists of all possible
classifying functions. However, in any round, Weak Learner searches
through the space, first trying hypotheses $h_1$ and $h_2$ shown in
the figure, and only if neither satisfy the Booster, search for
additional weak classifiers. In that case, any algorithm using SAMME's
weak learning condition, which is known to be too weak and satisfiable
by just the two hypotheses $\set{h_1,h_2}$, would only receive $h_1$
or $h_2$ in each round, and therefore be unable to reach the optimum
accuracy. Of course, if the Weak Learner is extremely generous and
helpful, then it may return the right collection of weak classifiers
even with a null weak learning condition that places no demands on
it. However, in practice, many Weak Learners used are similar to the
lazy weak learner described since these are computationally
efficient.

To see the effect of inconsistency arising from too weak
learning conditions in practice, we need to test boosting algorithms
relying on such datasets on significantly hard datasets, where only
the strictest Booster strategy can extract the necessary service from Weak Learner
for creating an optimal classifier.
We did not include
such experiments, and it will be an interesting empirical conjecture to
be tested in the future.
However, we did include experiments that illustrate
the consequence of using too strong conditions, and we discuss those
in the next section.

\section{Experiments}
\label{multi:expts:section}

In the final section of this \thesis{chapter}\notthesis{paper},
we report preliminary experimental results on 13 UCI datasets:
\textsf{letter, nursery, pendigits, satimage, segmentation, vowel,
  car, chess, connect4, forest, magic04, poker, abalone}. These
datasets are all multiclass except for \textsf{magic04}, have a wide
range of sizes, contain all combinations of real and categorical features,
have different number of examples to number of features per example
ratios, and are drawn from a 
variety of real-life situations. Most sets come with prespecified
train and test splits which we use; if not, we picked a random $4:1$
split. 
Throughout this section by \opt{} we refer to the version of
AdaBoost.MM studied in the consistency section, which uses the
approximate step size \eqref{multi:approx_step:eqn}.

There were two kinds of experiments. In the first, we took
a standard implementation \ada{} of Adaboost.M1 with C4.5 as weak
learner, and the Boostexter implementation \mh{} of Adaboost.MH using
stumps~\citep{SchapireSi00}, and compared it against
 our method \opt{}
with a naive greedy tree-searching weak-learner \greedy{}. The size of
trees to be used can be specified to our weak learner, and was chosen
to be the of the same order as the tree sizes used by \ada{}. The
test-error after $500$ rounds of boosting for each algorithm and
dataset is bar-plotted in Figure~\ref{multi:final:fig}. The performance is
comparable with \ada{} and far better than \mh{} (understandably since
stumps are far weaker than trees), even though our weak-learner is
very naive. The convergence rates of error with rounds of \ada{} and
\opt{} are also comparable, as shown in Figure~\ref{multi:ratebig:fig} (we
omitted the curve for \mh{} since it lay far above both \ada{} and \opt{}).
\begin{figure}
  \caption[Comparing final test error with standard boosting algorithms]{This is a plot
    of the final test-errors of standard 
    implementations of \ada{}, \mh{} and \opt{} after 500 rounds of
    boosting on different datasets. Both \ada{} and \opt{} achieve
    comparable error, which is often larger than that achieved by
    \mh{}.
    This is because \ada{} and \opt{} used trees of
    comparable sizes which were often much larger and powerful
    than the decision stumps that \mh{} boosted.}
  \label{multi:final:fig}
  \begin{center}
  \includegraphics[width=\textwidth]{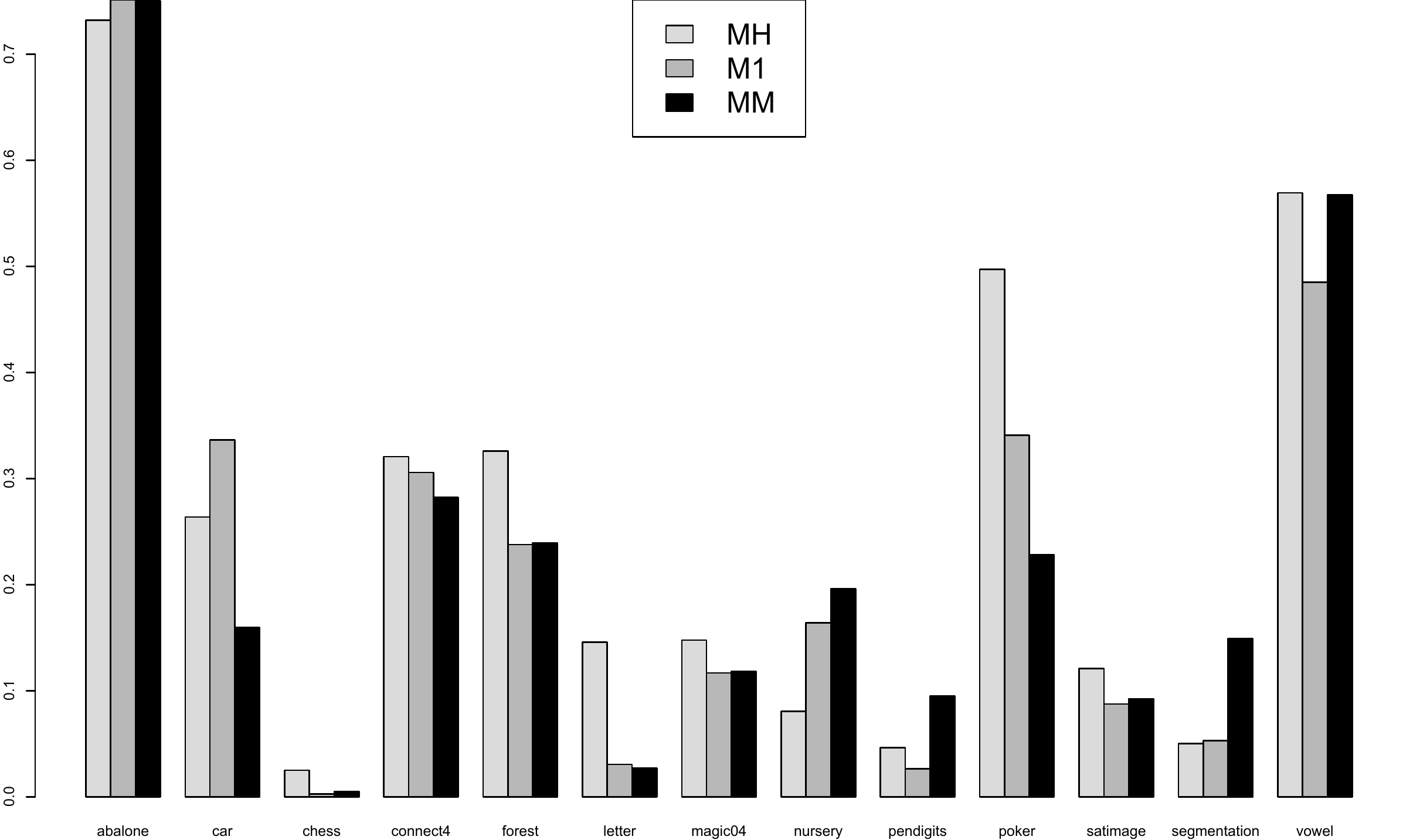}
  \end{center}
\end{figure}
\begin{figure}
  \caption[Comparing rate of convergence of test error with
  AdaBoost.M1]{Plots of the rates at which \ada{}(black,dashed) and 
    \opt{}(red,solid) drive down test-error on different data-sets
    when using trees of comparable sizes as weak classifiers.
    \ada{} called C4.5, and \opt{} called \greedy{}, respectively, as
    weak-learner.
    The tree sizes returned by C4.5  were used as a bound on the size
    of the trees that \greedy{} was allowed to return.
    This bound on the tree-size depended on the dataset, and are shown next to
    the dataset labels.}
  \label{multi:ratebig:fig}
  \begin{center}
  \includegraphics[width=\textwidth]{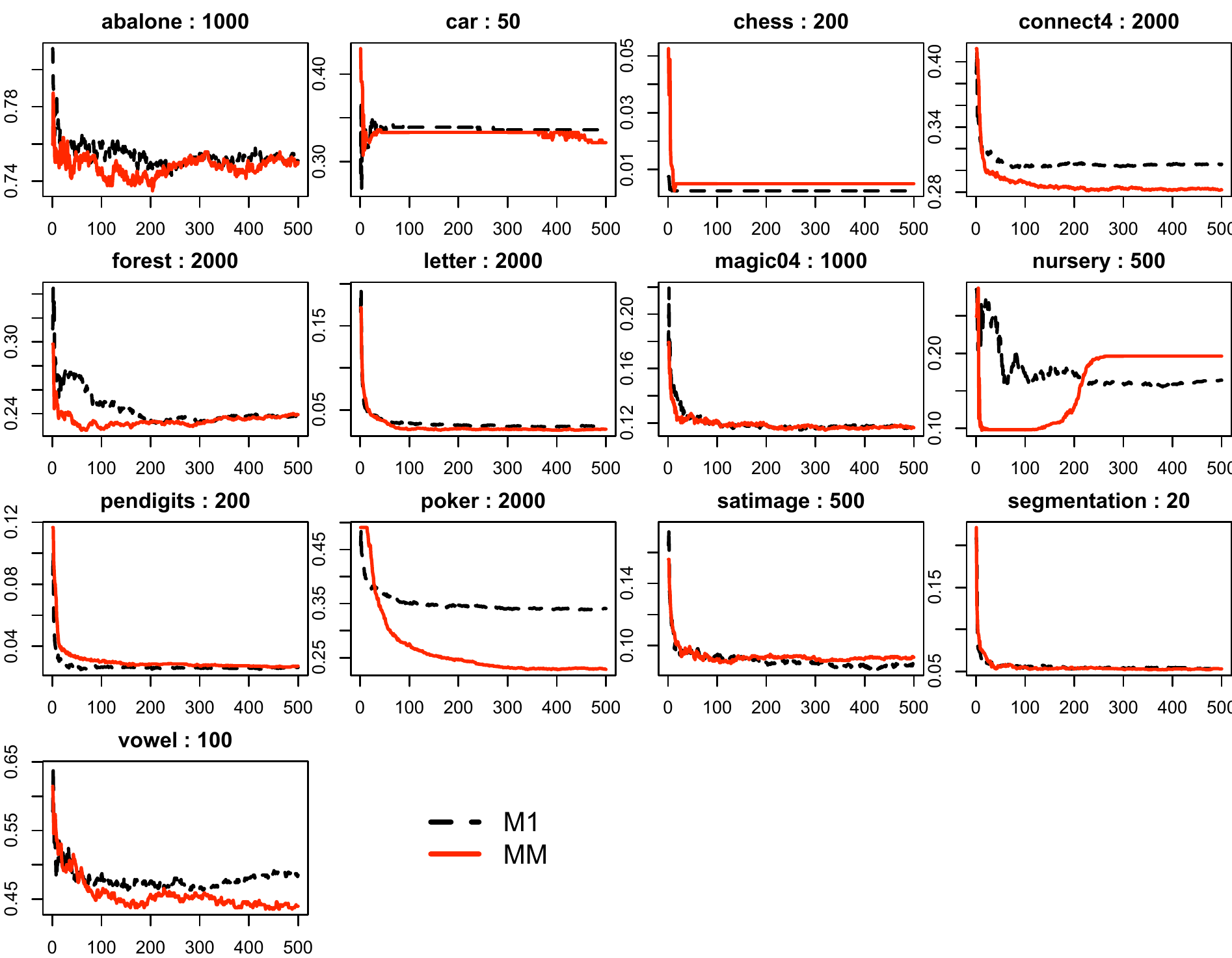}
  \end{center}
\end{figure}

We next investigated how each algorithm performs with less powerful
weak-learners. We modified \mh{} so that it uses a tree returning
a single multiclass prediction on each example.
For \mh{} and \opt{} we used the \greedy{} weak
learner, while for \ada{} we used a more powerful-variant \greedyinfo{}
whose greedy criterion was information gain rather than error (we also
ran \ada{} on top of \greedy{} but \greedyinfo{} consistently gave better
results so we only report the latter). We tried all tree-sizes in the
set \{10, 20, 50, 100, 200, 500, 1000, 2000, 4000\} up to the
tree-size used 
by \ada{} on C4.5 for each data-set. We plotted the error of each
algorithm against tree size for each data-set in
Figure~\ref{multi:errVsTreeSz:fig}.
\begin{figure}
  \caption[Test error with weak classifiers of varying complexity]
  {For this figure, \ada{}(black,
    dashed), \mh{}(blue, dotted) and \opt{}(red,solid) were designed to
    boost decision trees of restricted sizes.
    The final test-errors of the three algorithms after 500 rounds of
    boosting are plotted  against the maximum tree-sizes allowed for
    the weak classifiers. 
    \opt{} achieves much lower error when the weak classifiers are very
    weak, that is, with smaller trees.}
  \label{multi:errVsTreeSz:fig}
  \begin{center}
  \includegraphics[width=\textwidth]{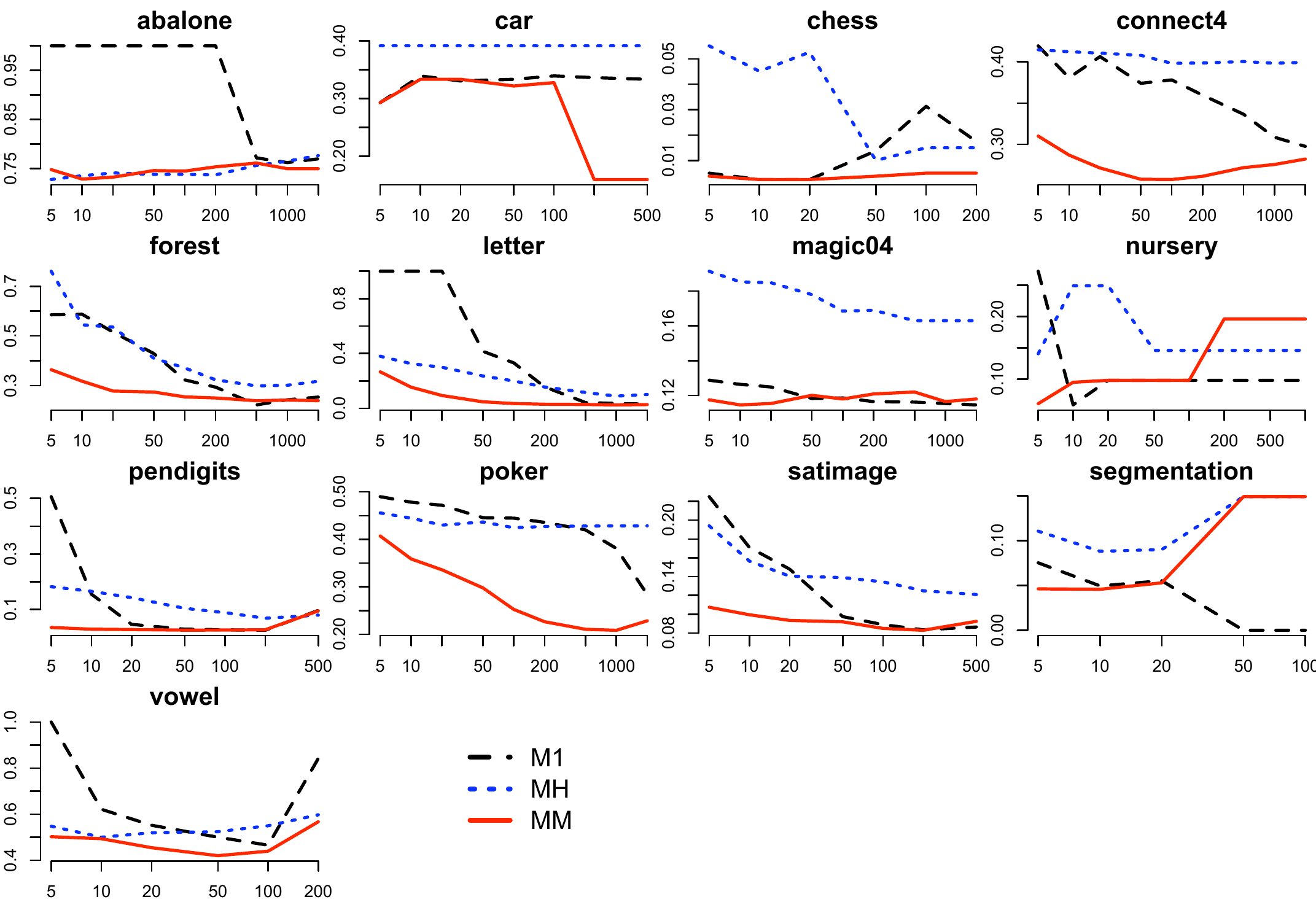}
  \end{center}
\end{figure}
As predicted by our theory, our
algorithm succeeds in boosting the accuracy even when the tree size
is too small to meet the stronger weak learning assumptions of the
other algorithms. More insight is provided by plots in
Figure~\ref{multi:errs-5:fig} of the rate of convergence of error with
rounds when the tree size allowed is very small ($5$).
\begin{figure}
  \caption[Test error through rounds with very simple weak classifiers]
  {A plot of how fast the test-errors of the three algorithms
    drop with rounds when the weak classifiers are trees with a size
    of at most 5. Algorithms \ada{} and \mh{} make strong demands which
  cannot be met by the extremely weak classifiers after a few rounds,
  whereas \opt{} makes gentler demands, and is hence able to drive
  down error through all the rounds of boosting.}
  \label{multi:errs-5:fig}
  \begin{center}
  \includegraphics[width=\textwidth]{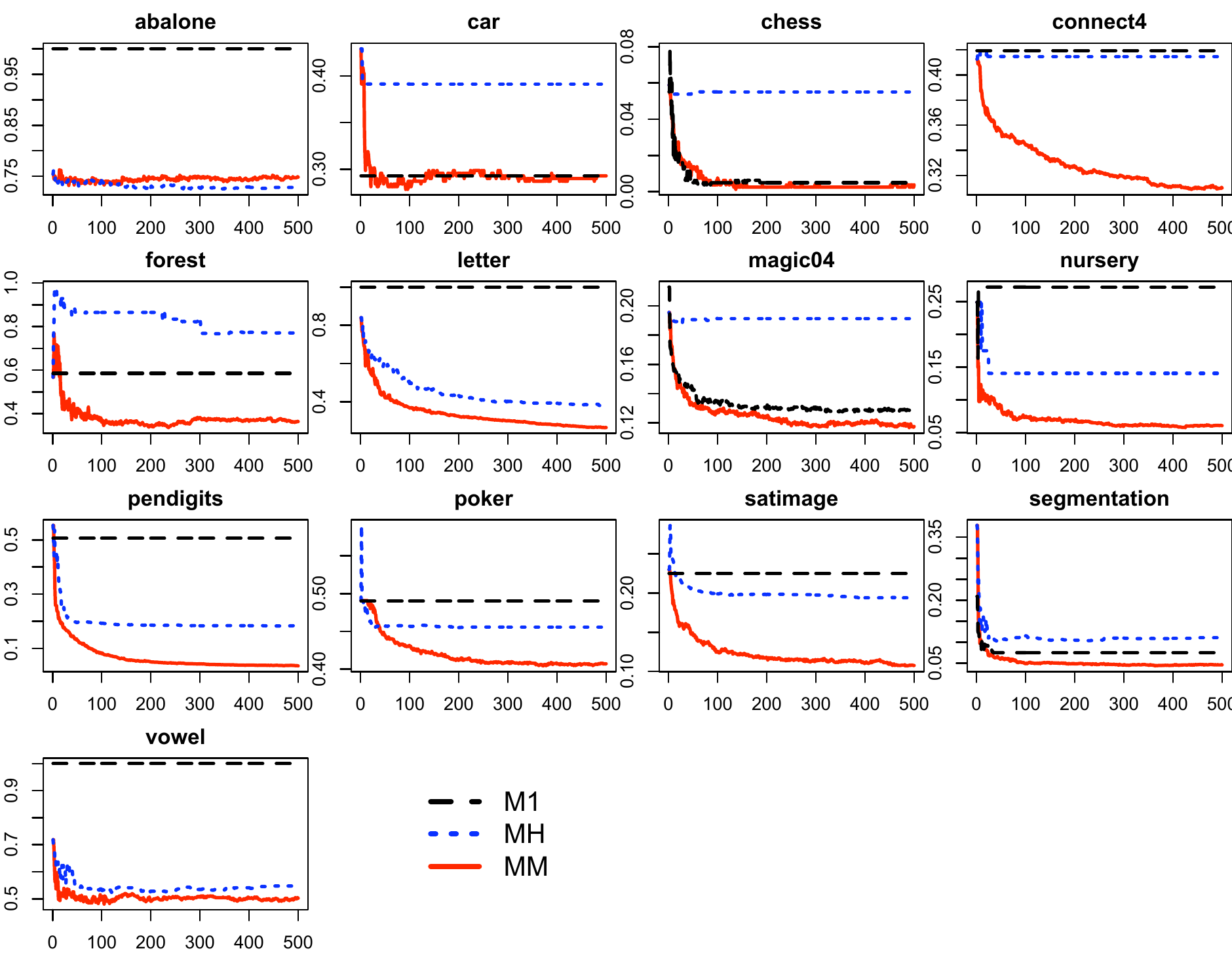}
  \end{center}
\end{figure}
Both \ada{}
and \mh{} drive down the error for a few rounds. But since boosting
keeps creating harder distributions, very soon the small-tree learning
algorithms \greedy{} and \greedyinfo{} are no longer able to meet the
excessive requirements of \ada{} and \mh{} respectively. However, our
algorithm makes more reasonable demands that are easily met by
\greedy{}.

\old{
  \section{Conclusion}
  In summary, we create a new framework for studying multiclass
  boosting.
  This framework is very general and captures the weak learning
  conditions implicitly used by many earlier multiclass boosting
  algorithms as well as novel conditions, including
  the minimal condition under which boosting is possible.
  We also show how to design boosting algorithms relying on
  these weak learning conditions that drive down training error
  rapidly.
  These algorithms are the optimal strategies for playing  certain two
  player games.
  Based on this game-theoretic approach, we also design a
  multiclass boosting algorithm that is consistent, i.e.,
  approaches the minimum  empirical risk, and under some basic
  assumptions, the Bayes optimal test error.
  Preliminary experiments show that this algorithm can achieve much
  lower error compared to existing algorithms when used with very
  weak classifiers. 

  Although we can efficiently compute the game-theoretically optimal
  strategies under most conditions, when using the minimal weak
  learning condition, and non-convex 0-1 error as
  loss function, we require
  exponential computational time to solve the corresponding boosting
  games.
  Boosting algorithms based on error are potentially far more noise
  tolerant than those based on convex loss functions, and finding
  efficiently computable near-optimal strategies in this situation
  is an important problem left for future work.
  Further, we primarily work with weak classifiers that output a single
  multiclass prediction per example, whereas weak hypotheses that make
  multilabel multiclass predictions are typically more powerful.
  We believe that multilabel predictions do not increase the power of
  the weak learner in our framework, and our theory can be extended
  without much work to include such hypotheses, but we do not address
  this here. 
  Finally, it will be interesting to see if the notion of minimal weak
  learning condition can be extended to boosting settings beyond
  classification, such as ranking.
}

\notthesis{
  \subsection*{Acknowledgments}
  This research was funded by the National Science Foundation under
  grants IIS-0325500 and IIS-1016029.
}

\thesis{\section{Appendix}}
\notthesis{\section*{Appendix}}

\old{
\notthesis{\subsection*{Optimality of the OS strategy}}
\thesis{\subsection{Optimality of the OS strategy}}

  Here we prove Theorem~\ref{multi:dgstrat:thm}.
  The proof of the upper bound on the loss is very similar to the
  proof of Theorem~2 in \citep{Schapire01}.
  For the lower bound, a similar result is proven in Theorem~3 in
  \citep{Schapire01}.
  However, the proof relies on certain assumptions that may not hold
  in our setting\thesis{.}\notthesis{,
  and we instead follow the more direct lower bounding techniques in
  Section 5 of \citep{MukherjeeSc10a}.} 

  We first show that the average potential of states does not increase
  in any round.
  The dual form of the recurrence \eqref{multi:dgrec_dual:eqn} and the
  choice of the cost matrix $\C_t$ in \eqref{multi:os_ct:eqn} together ensure
  that for each example $i$,
  \begin{eqnarray*}
  \phi^{\B(i)}_{T-t}\enc{\vs_t(i)}
  &=&
  \max_{l=1}^k
  \enct{\phi^{\B(i)}_{T-t-1}\enc{\vs_t(i)+\ve_l}
    - \enc{\C_t(i)(l)-\dotp{\C_t(i)}{\B(i)}}}\\
  &\geq&
  \phi^{\B(i)}_{T-t-1}\enc{\vs_t(i)+\ve_{h_t(x_i)}}
  - \enc{C_t(i,h_t(x_i))-\dotp{\C_t(i)}{\B(i)}}.
\end{eqnarray*}
Summing up the inequalities over all examples, we get
\[
\sum_{i=1}^m \phi^{\B(i)}_{T-t-1}\enc{\vs_t(i)+\ve_{h_t(x_i)}}
\leq
\sum_{i=1}^m\phi^{\B(i)}_{T-t}\enc{\vs_t(i)}
+
\sum_{i=1}^m\enct{C_t(i,h_t(x_i))-\dotp{\C_t(i)}{\B(i)}}
\]
The first two summations are the total potentials in round $t+1$ and
$t$, respectively, and the third summation is the difference in the
costs incurred by the weak-classifier $h_t$ returned in iteration $t$
and the baseline $\B$.
By the weak learning condition, this difference is non-positive, implying
that the average potential does not increase.

Next we show that the bound is tight.
In particular choose any accuracy parameter $\eps>0$, and total number
of iterations $T$, and let $m$ be as large as in \eqref{multi:mlarge:eqn}.
We show that in any iteration $t\leq T$, based on Booster's choice
of cost-matrix $\C$, an adversary can choose a weak
classifier $h_t\in\Hall$ such that the weak learning
condition is satisfied, and the average potential does
not fall by more than an amount $\eps/T$.
In fact, we show how to choose labels $l_1,\ldots, l_m$ such 
that the following hold simultaneously:
\begin{eqnarray}
  \label{multi:label_cost:eqn}
  \sum_{i=1}^m C(i,l_i) &\leq& \sum_{i=1}^m \dotp{\C(i)}{\B(i)}\\
  \label{multi:label_pot:eqn}
  \sum_{i=1}^m\phi^{\B(i)}_{T-t}\enc{\vs_t(i)}
  &\leq&
  \frac{m\eps}{T} + \sum_{i=1}^m\phi^{\B(i)}_{T-t-1}\enc{\vs_t(i)+\ve_{l_i}}
\end{eqnarray}
This will imply that the final potential or loss is at least $\eps$
less than the bound in \eqref{multi:dgstratbnd:eqn}.

We first construct, for each example $i$, a distribution
$\vp_i\in\Delta\set{1,\ldots,k}$ such that the size of the support of
$\vp_i$ is either 1 or 2, and
\begin{equation}
  \label{multi:prec:eqn}
\phi^{\B(i)}_{T-t}(\vs_t(i)) = \E_{l\sim\vp_i}
\enco{\phi_{T-t-1}^{\B(i)}\enc{\vs_t(i)+\ve_l}}.
\end{equation}
To satisfy \eqref{multi:prec:eqn}, by \eqref{multi:dgrec:eqn}, we may
choose $\vp_i$ as any optimal response of the max player in the minmax
recurrence when the min player chooses $\C(i)$:
\begin{eqnarray}
  \label{multi:popt:eqn}
\vp_i &\in& 
  \argmax_{\vp\in\mathcal{P}_i}
  \enct{\E_{l\sim \vp}\enco{\phi^{\B(i)}_{t-1}\enc{\vs + \ve_l}}}
  \\
  \label{multi:ptope:eqn}
  \mbox { where }
  \ptope_i
  &=& \set{\vp\in\Delta\set{1,\ldots,k}:
  \E_{l\sim \vp} \enco{C(i,l)} \leq \dotp{\C(i)}{\B(i)}}.
\end{eqnarray}
The existence of $\vp_i$ is guaranteed, since, by
Lemma~\ref{multi:pexists:lem}, the polytope $\ptope_i$ is non-empty
for each $i$.
The next result shows that we may choose $\vp_i$ to have a support of
size 1 or 2.
\begin{lemma}
  \label{multi:supp_size:lem}
  There is a $\vp_i$ satisfying \eqref{multi:popt:eqn} with
  either 1 or 2 non-zero coordinates.
\end{lemma}
\begin{proof}
  Let $\vps$ satisfy \eqref{multi:popt:eqn}, and let its support set
  be $S$.
  Let $\mu_i$ denote the mean cost under this distribution:
  \[
  \mu_i = \E_{l\sim\vps}\enco{C(i,l)} \leq \dotp{\C(i)}{\B(i)}.
  \]
  If the support has size at most 2, then we are done.
  Further, if each non-zero coordinate $l\in S$ of $\vps$ satisfies
  $C(i,l) = \mu_i$,
  then the distribution $\vp_i$ that concentrates all its weight
  on the label $l^{\min}\in S$ minimizing
  $\phi^{\B(i)}_{t-1}\enc{\vs + \ve_{l^{\min}}}$ is an optimum solution
  with support of size 1.
  Otherwise, we can pick labels $l^{\min}_1,l^{\min}_2\in S$ such that
  \[
  C(i,l^{\min}_1) <  \mu_i < C(i,l^{\min}_2).
  \]
  Then we may choose a distribution $\vq$ supported on these two labels
  with mean $\mu_i$:
  \[
  \E_{l\sim\vq}\enco{C(i,l)}
  = q(l^{\min}_1)C(i,l^{\min}_1) + q(l^{\min}_2)C(i,l^{\min}_2)
  = \mu_i.
  \]
  Choose $\lambda$ as follows:
  \[
  \lambda = \min\set{
    \frac{p^*(l^{\min}_1)}{q(l^{\min}_1)},
    \frac{p^*(l^{\min}_2)}{q(l^{\min}_2)}
  },
  \]
  and
  write $\vps = \lambda \vq + (1-\lambda)\vp$.
  Then both $\vp,\vq$ belong to the polytope $\ptope_i$, and have
  strictly fewer non-zero coordinates than $\vps$.
  Further, by linearity, one of $\vq,\vp$ is also optimal.
  We repeat the process on the new optimal distribution till we
  find one which has only 1 or 2 non-zero entries.
\end{proof}

We next show how to choose the labels $l_1,\ldots,l_m$ using the
distributions $\vp_i$.
For each $i$, let $\set{l^+_i,l^-_i}$ be the support of $\vp_i$ so
that
\[
C\enc{i,l^+_i} \leq \E_{l\sim\vp_i}\enco{C(i,l)} \leq
C\enc{i,l^-_i}.
\]
(When $\vp_i$ has only one non-zero element, then $l^+_i = l^-_i$.)
For brevity, we use $p^+_i$ and $p^-_i$ to denote $p_i\enc{l^+_i}$ and
$p_i\enc{l^-_i}$, respectively.
If the costs of both labels are equal, we assume without loss of
generality that $\vp_i$ is concentrated on label $l^-_i$:
\begin{equation}
  \label{multi:zero_conv:eqn}
  C\enc{i,l^-_i} - C\enc{i,l^-_i} = 0
  \implies
  p^+_i = 0, p^-_i = 1.
\end{equation}

We will choose each label $l_i$ from the set $\set{l^-_i,l^+_i}$.
In fact, we will choose a partition $S_+,S_-$  of the examples
$1,\ldots,m$ and choose the label depending on which side $S_{\xi}$, for
$\xi\in\set{-,+}$, of the partition element $i$ belongs to:
\[
l_i = l^{\xi}_i \mbox{ if } i \in S_{\xi}.
\]
In order to guide our choice for the partition, 
we introduce parameters $a_i,b_i$ as follows:
\begin{eqnarray*}
  a_i &=& C(i,l^-_i) - C(i,l^+_i), \\
  b_i &=& \phi^{\B(i)}_{T-t-1}\enc{\vs_t(i) + \ve_{l^-_i}}
  - \phi^{\B(i)}_{T-t-1}\enc{\vs_t(i) + \ve_{l^+_i}}.
\end{eqnarray*}
Notice that for each example $i$ and each sign-bit $\xi\in\set{-1,+1}$,
we have the following relations:
\begin{eqnarray}
  \label{multi:pcost:eqn}
  C(i,l^{\xi}_i) &=& \E_{l\sim\vp_i}\enco{C(i,l)} - \xi(1-p^{\xi}_i)a_i\\
  \label{multi:ppot:eqn}
  \phi^{\B(i)}_{T-t-1}\enc{\vs_t(i) + \ve_{l^{\xi}_i}}
  &=&
  \E_{l\sim\vp_i}\enco{\phi^{\B(i)}_{T-t}(i,l)}
   - \xi(1-p^{\xi}_i)b_i.
\end{eqnarray}
Then the cost incurred by the choice of labels can be expressed in
terms of the parameters $a_i,b_i$ as follows:
\begin{eqnarray}
  \sum_{i\in S_+}C(i,l^+_i)
  + \sum_{i\in S_-}C(i,l^-_i)
  &=&
  \sum_{i\in S_+}
  \enct{ \E_{l\sim\vp_i}\enco{C(i,l)} - a_i + p^+_ia_i}
  \nonumber \\
  && +
  \sum_{i\in S_-}
  \enct{ \E_{l\sim\vp_i}\enco{C(i,l)} + p^+_ia_i}
  \nonumber \\
  &=&
  \sum_{i=1}^m\E_{l\sim\vp_i}\enco{C(i,l)}
  + \enc{
    \sum_{i=1}^mp^+_ia_i
    - \sum_{i\in S_+}a_i}
  \nonumber \\
  \label{multi:aineq:eqn}
  &\leq&
  \sum_{i=1}^m \dotp{\C(i)}{\B(i)}
 + \enc{
    \sum_{i=1}^mp^+_ia_i
    - \sum_{i\in S_+}a_i},
\end{eqnarray}
where the first equality follows from \eqref{multi:pcost:eqn}, and the
inequality follows from the constraint on $\vp_i$ in
\eqref{multi:ptope:eqn}.
Similarly, the potential of the new states is given by
\begin{eqnarray}
  \lefteqn{
    \sum_{i\in S_+}\phi^{\B(i)}_{T-t-1}\enc{\vs_t(i) + \ve_{l^+_i}} +
    \sum_{i\in S_-}\phi^{\B(i)}_{T-t-1}\enc{\vs_t(i) + \ve_{l^-_i}}
  }\\
  &=&
  \sum_{i\in S_+}
  \enct{ \E_{l\sim\vp_i}
    \enco{\phi^{\B(i)}_{T-t-1}\enc{\vs_t(i) + \ve_{l}}}
    - b_i + p^+_ib_i}
  \nonumber \\
  && +
  \sum_{i\in S_-}
  \enct{ \E_{l\sim\vp_i}
    \enco{\phi^{\B(i)}_{T-t-1}\enc{\vs_t(i) + \ve_{l}}}
    + p^+_ib_i}
  \nonumber \\
  &=&
  \sum_{i=1}^m\E_{l\sim\vp_i}
    \enco{\phi^{\B(i)}_{T-t-1}\enc{\vs_t(i) + \ve_{l}}}
  + \enc{
    \sum_{i=1}^mp^+_ib_i
    - \sum_{i\in S_+}b_i}
  \nonumber \\
  \label{multi:bineq:eqn}
  &=&
  \sum_{i=1}^m \phi^{\B(i)}_{T-t}\enc{\vs_t(i)}
 + \enc{
    \sum_{i=1}^mp^+_ib_i
    - \sum_{i\in S_+}b_i},
\end{eqnarray}
where the first equality follows from \eqref{multi:ppot:eqn}, and the
last equality from an optimal choice of $\vp_i$ satisfying
\eqref{multi:prec:eqn}.
Now, \eqref{multi:aineq:eqn} and \eqref{multi:bineq:eqn} imply that in
order to satisfy \eqref{multi:label_cost:eqn} and 
\eqref{multi:label_pot:eqn}, it suffices to choose a subset $S_+$ 
satisfying
\begin{align}
  \label{multi:req:eqn}
  \sum_{i\in S_+}a_i &\geq \sum_{i=1}^mp^+_ia_i, &
\sum_{i\in S_+}b_i &\leq \frac{m\eps}{T} + \sum_{i=1}^mp^+_ib_i.
\end{align}
We simplify the required conditions.
Notice the first constraint tries to ensure that $S_+$ is big, while
the second constraint forces it to be small, provided the $b_i$ are
non-negative.
However, if $b_i<0$ for any example $i$, then
adding this example to $S_+$ only helps both inequalities.
In other words, if we can always construct a set $S_+$ satisfying
\eqref{multi:req:eqn} in the case where the $b_i$ are non-negative,
then we may handle the more general situation by just adding the
examples $i$ with negative $b_i$ to the set $S_+$ that would be
constructed by considering only the examples $\set{i:b_i \geq 0}$.
Therefore we may assume without loss of generality that the $b_i$ are
non-negative. 
Further, assume (by relabeling if necessary) that $a_1,\ldots,a_{m'}$ are
positive and $a_{m'+1}, \ldots a_m = 0$, for some $m'\leq m$.
By \eqref{multi:zero_conv:eqn}, we have $p^+_i = 0$ for $i>m'$.
Therefore, by assigning the examples $m'+1,\ldots,m$ to the opposite partition
$S_-$, we can ensure that \eqref{multi:req:eqn} holds if the following
is true:
\begin{eqnarray}
  \label{multi:areq:eqn}
  \sum_{i\in S_+}a_i &\geq& \sum_{i=1}^{m'}p^+_ia_i, \\
  \label{multi:breq:eqn}
\sum_{i\in S_+}b_i &\leq& \max_{i=1}^{m'}\abs{b_i} +
\sum_{i=1}^{m'}p^+_ib_i,
\end{eqnarray}
where, for \eqref{multi:breq:eqn}, we additionally used that, by the
choice of $m$ \eqref{multi:mlarge:eqn} and the bound on loss
variation \eqref{multi:lossvar:eqn}, we have 
\[
\frac{m\eps}{T} \geq \diameter(L,T) \geq b_i
\mbox{ for } i=1,\ldots,m.
\]
The next lemma shows how to construct such a subset $S_+$, and
concludes our lower bound proof.
\begin{lemma}
  \label{multi:balance:lem}
  Suppose $a_1,\ldots,a_{m'}$ are positive and
  $b_1,\ldots,b_{m'}$ are non-negative reals, and
  $p^+_1,\ldots,p^+_{m'}\in [0,1]$ are probabilities.
  Then there exists a subset $S_+\subseteq\set{1,\ldots,m'}$ such that
  \eqref{multi:areq:eqn} and \eqref{multi:breq:eqn} hold.
\end{lemma}
\begin{proof}
  Assume, by relabeling if necessary, that the following ordering holds:
  \begin{equation}
    \label{multi:order:eqn}
  \frac{a(1)-b(1)}{a(1)}
  \geq
  \cdots
  \geq
  \frac{a(m')-b(m')}{a(m')}.
  \end{equation}
  Let $I\leq m'$ be the largest integer such that
  \begin{equation}
    \label{multi:Isuma:eqn}
  a_1 + a_2 + \cdots + a_I < \sum_{i=1}^{m'}p^+_ia_i.
\end{equation}
Since the $p^+_i$ are at most $1$, $I$ is in fact at most
$m'-1$. 
We will choose $S_+$ to be the first $I+1$ examples
  $S_+ = \set{1,\ldots,I+1}$. 
  Observe that \eqref{multi:areq:eqn} follows immediately from the
  definition of $I$.
  Further, \eqref{multi:breq:eqn} will hold if the following is true
  \begin{equation}
    \label{multi:Isumb:eqn}
  b_1 + b_2 + \cdots + b_I \leq \sum_{i=1}^{m'}p^+_ib_i,
  \end{equation}
  since the addition of one more example $I+1$ can exceed this
  bound by at most $b_{I+1} \leq \max_{i=1}^{m'}\abs{b_i}$.
  We prove \eqref{multi:Isumb:eqn} by showing that the left hand side
  of this equation is not much more than the left hand side of
  \eqref{multi:Isuma:eqn}.
  We first rewrite the latter summation differently.
  The inequality in \eqref{multi:Isuma:eqn} implies we can pick
  $\tp_1,\ldots,\tp_{m'}\in [0,1]$ (e.g., by simply scaling the
  $p^+_i$'s appropriately) such that 
  \begin{eqnarray}
    \label{multi:tpeq:eqn}
    a_1 + \ldots + a_I &=& \sum_{i=1}^{m'}\tp_ia_i\\
    \label{multi:tpineq:eqn}
    \mbox{ for } i=1,\ldots, m' \mbox{: } \tp_i &\leq& p_i.
  \end{eqnarray}
  By subtracting off the first $I$ terms in the right hand side of
  \eqref{multi:tpeq:eqn} from both sides we get
  \[
  (1-\tp_1)a_1 + \cdots + (1-\tp_I)a_I
  =
  \tp_{I+1}a_{I+1} + \cdots + \tp_{m'}a_{m'}.
  \]
  Since the terms in the summations are non-negative,
  we may combine the above with the ordering property in
  \eqref{multi:order:eqn} to get
  \begin{eqnarray}
  \lefteqn{(1-\tp_1)a_1\enc{\frac{a_1-b_1}{a_1}}
  + \cdots +
  (1-\tp_I)a_I\enc{\frac{a_I-b_I}{a_I}}}
\nonumber \\
\label{multi:part_ineq:eqn}
&\geq&
  \tp_{I+1}a_{I+1}\enc{\frac{a_{I+1}-b_{I+1}}{a_{I+1}}}
  + \cdots +
  \tp_{m'}a_{m'}\enc{\frac{a_{m'}-b_{m'}}{a_{m'}}}.
  \end{eqnarray}
Adding the expression
\[
\tp_1a_1\enc{\frac{a_1-b_1}{a_1}}
+ \cdots +
\tp_Ia_I\enc{\frac{a_I-b_I}{a_I}}
\]
to both sides of \eqref{multi:part_ineq:eqn} yields
\begin{eqnarray}
  \sum_{i=1}^I a_i\enc{\frac{a_i-b_i}{a_i}}
  &\geq&
  \sum_{i=1}^{m'}\tp_ia_i\enc{\frac{a_i-b_i}{a_i}}
  \nonumber \\
  \mbox{ i.e. }
  \sum_{i=1}^Ia_i - \sum_{i=1}^Ib_i
  &\geq&
  \sum_{i=1}^{m'}\tp_ia_i
  -\sum_{i=1}^{m'}\tp_ib_i
  \nonumber \\
  \label{multi:btpineq:eqn}
  \mbox{ i.e. }
  \sum_{i=1}^Ib_i
  &\leq&
  \sum_{i=1}^{m'}\tp_ib_i,
\end{eqnarray}
where the last inequality follows from \eqref{multi:tpeq:eqn}.
Now \eqref{multi:Isumb:eqn} follows from \eqref{multi:btpineq:eqn}
using \eqref{multi:tpineq:eqn} and the fact that the $b_i$'s are
non-negative.
\end{proof}
This completes the proof of the lower bound.
}

\notthesis{\subsection*{Consistency proofs}}
\thesis{\subsection{Consistency proofs}}

Here we sketch the proofs of Lemmas~\ref{multi:ersk_cons:lem}
  and \ref{multi:rsk_cons:lem}.
Our approach will be to relate our algorithm to AdaBoost and then use
relevant known results on the consistency of AdaBoost.
We first describe the correspondence between the
two algorithms, and then state and connect the relevant
results on AdaBoost to the ones in this section.

For any given
multiclass dataset and weak 
classifier space, we will 
obtain a transformed binary dataset and weak classifier space, such
that the run of AdaBoost.MM on the original
dataset will be in perfect correspondence with the run of AdaBoost on
the transformed dataset. In particular, the loss and error on both the
training and test set of the combined classifiers produced by our
algorithm will be exactly equal 
to those produced by AdaBoost, while the space of functions and
classifiers on the two datasets will be in
correspondence.

Intuitively, we transform our multiclass classification problem into a
single binary classification problem in a way similar to the all-pairs
multiclass to binary reduction.
A very similar reduction was carried out by \cite{FreundSc97}.
Borrowing their terminology, the transformed dataset roughly consists
of \emph{mislabel} triples $(x,y,l)$ where $y$ is the true label of
the example and $l$ is an incorrect example.
The new binary label of a mislabel triple is always $-1$, signifying
that $l$ is not the true label.
A multiclass classifier becomes a binary classifier that predict $\pm
1$ on the mislabel triple $(x,y,l)$ depending on whether the
prediction on $x$ matches label $l$; therefore error on the
transformed binary dataset is low whenever the multiclass accuracy is high.
The details of the transformation are provided in
Figure~\ref{multi:tf:fig}.

\begin{figure}
  \centering
  \thesis{\footnotesize}
  \notthesis{\footnotesize}
  \begin{tabular}{|p{2.5cm}|l|l|}
    \hline
    & AdaBoost.MM
    & AdaBoost
     \bigstrut \\ \hline \hline \bigstrut
    Labels
    & $\Y=\set{1,\ldots,k}$
    & $\tf{\Y}=\set{-1,+1}$
     \bigstrut \\ \hline \bigstrut
    Examples
    & $\X$
    &
    $\tf{\X} = \X \times
    \enc{\enc{\Y\times \Y}
      \setminus
      \set{(y,y):y\in\Y}}$
     \bigstrut \\ \hline \bigstrut
    Weak classifiers
    & $h:\X\to\Y$
    &
    $ \tf{h}:\tf{X}\to\set{-1,0,+1}$, where\\
    &&
    $\tf{h}(x,y,l) = \1\enco{h(x)=l} - \1\enco{h(x)=y}$
    \bigstrut \\ \hline \bigstrut
    Classifier space
    & $\H$
    & $\tf{\H} = \set{\tf{h}:h\in\H}$
    \bigstrut \\ \hline \bigstrut
    Scoring function
    & $F:\X\times\Y\to\R$
    & $\tf{F}:\tf{\X}\to\R$ where \\
    &&
    $\tf{F}(x,y,l) = F(x,l)-F(x,y)$
    \bigstrut \\ \hline \bigstrut
    Clamped function
    & $\cF(x,y) = $
    &
    $\bar{\tf{F}}(x,y,l) = \tf{F}(x,y,l)$,
    if $\abs{\tf{F}(x,y,l)}  \leq C$\\
    &
    $\max\set{-C, F(x,l) - \max_{l'} F_T(x,l')}$
    &
    $\bar{\tf{F}}(x,y,l) = C$,
    if $\abs{\tf{F}(x,y,l)} > C$
    \bigstrut \\ \hline \bigstrut
    Classifier weights
    & $\alpha:\H\to\R$
    & $\tf{\alpha}:\tf{\H}\to\R$ where \\
    &&
    $\tf{\alpha}\enc{\tf{h}} =
    \alpha(h)$
    \bigstrut \\ \hline \bigstrut
    Combined hypothesis
    & $F_\alpha$ where
    & $\tf{F}_{\tf{\alpha}}$ where \\
    &
    $F_{\alpha}(x,l) =
    \sum_{h\in\H}\alpha(h)
    \1\enco{h(x)=l}
    $
    &
    $\tf{F}_{\tf{\alpha}}(x,y,l)
    =
    \sum_{\tf{h}\in\tf{\H}}
    \tf{\alpha}\enc{\tf{h}}
    \tf{h}(x,y,l)
    $
    \bigstrut \\ \hline \bigstrut
    Training set
    & $S = \set{(x_i,y_i): 1\leq i\leq m}$
    & $\tf{S} =$\\
    &&
    $\set{((x_i,y_i,l),\xi): \xi=-1,l\neq y_i, 1\leq i\leq m}$ 
    \bigstrut \\ \hline \bigstrut
    Test distribution
    & $D$ over $\X\times\Y$
    & $\tf{D}$ over $\tf{X}\times\tf{Y}$ where\\
    &&
    $\tf{D}((x,y,l),-1) = D(x,y)/(k-1)$\\
    &&
    $\tf{D}((x,y,l),+1) = 0$
    \bigstrut \\ \hline \bigstrut
    Empirical risk
    & $\ersk(F) =$
    & $\tf{\ersk}\enc{\tf{F}}$\\
    &
    $\frac{1}{m}
    \sum_{i=1}^m\sum_{l\neq y_i}e^{F(x_i,l) - F(x_i,y_i)}$
    &
    $\frac{1}{m(k-1)}
    \sum_{i=1}^m\sum_{l\neq y_i}
    e^{-\xi\tf{F}(x_i,y_i,l)}$
    \bigstrut \\ \hline \bigstrut
    Test risk
    &
    $\rsk(F) =$
    &
    $\tf{\rsk}\enc{\tf{F}} =$ \\    
    &
    $\E_{(x,y)\sim D}
    \enco{\sum_{l\neq y}
    e^{F(x,l)-F(x,y)}}$
    &
    $\E_{((x,y,l),\xi)\sim \tf{D}}
    \enco{e^{-\xi\tf{F}(x,y,l)}}$
    \bigstrut \\ \hline    
  \end{tabular}
  \thesis{\normalsize}
  \notthesis{\normalsize}
  \caption{Details of transformation between
    AdaBoost.MM and AdaBoost.}
  \label{multi:tf:fig}
\end{figure}

Some of the properties between the functions and their transformed
counterparts are described in the next lemma, showing that we are
essentially dealing with similar objects.
\begin{lemma}
  \label{multi:ident:lem}
  The following are identities for any scoring function
  $F:\X\times\Y\to\R$ and weight function $\alpha:\H\to\R$:
  \begin{eqnarray}
    \label{multi:ersk_map:eqn}
    \ersk \enc{F_\alpha} &=& \tf{\ersk} \enc{\tf{F}_{\tf{\alpha}}} \\
    \label{multi:rsk_map:eqn}    
    \rsk \enc{\cF} &=& \tf{\rsk} \enc{\bar{\tf{F}}}.
  \end{eqnarray}
\end{lemma}
The proofs involve doing straightforward algebraic manipulations to
verify the identities and are omitted.

The next lemma connects the two algorithms.
We show that the scoring function output by AdaBoost when run on
the transformed dataset is the transformation of the function output
by our algorithm. The proof again involves tedious but straightforward
checking of details and is omitted.
\begin{lemma}
  \label{multi:same_runs:lem}
  If AdaBoost.MM produces scoring function
  $F_{\alpha}$ when run for $T$ rounds with the training set $S$ and
  weak classifier space $\H$, then
  AdaBoost produces the scoring function $\tf{F}_{\tf{\alpha}}$ when
  run for $T$ rounds with the training set $\tf{S}$ and space
  $\tf{\H}$. We assume that for both the algorithms, Weak Learner
  returns the weak classifier in each round that achieves the maximum
  edge. Further we consider the version of
  AdaBoost.MM that chooses weights according to the approximate rule
  \eqref{multi:approx_step:eqn}. 
\end{lemma}

We next \thesis{restate}\notthesis{state} the result for AdaBoost corresponding to
Lemma~\ref{multi:ersk_cons:lem}
\notthesis{, which appears 
  in \citep{MukherjeeRuSc11}.}
\thesis{which we have already seen in Chapter~\ref{thesis:rate:chap}}.
\begin{lemma}
  \notthesis{[Theorem~8 in \citep{MukherjeeRuSc11}]}
  \thesis{[Theorem~\ref{rate:rate:thm}]}
  \label{multi:ada_ersk:lem}
  Suppose AdaBoost produces the scoring function
  $\tf{F}_{\tf{\alpha}}$
  when run for $T$ rounds with the training set $\tf{S}$ and space
  $\tf{\H}$.
  Then
  \begin{equation}
    \label{multi:ada_ersk:eqn}
    \tf{\ersk} \enc{\tf{F}_{\tf{\alpha}}}
    \leq
    \inf_{\tf{\beta}:\tf{\H}\to\R}
    \tf{\ersk} \enc{\tf{F}_{\tf{\beta}}}
    + C/T,
  \end{equation}
  where the constant $C$ depends only on the dataset.
\end{lemma}
The previous lemma, along with \eqref{multi:ersk_map:eqn} immediately
proves Lemma~\ref{multi:ersk_cons:lem}.
The result for AdaBoost corresponding to
Lemma~\ref{multi:rsk_cons:lem} appears 
in \citep{SchapireFr12}.
\begin{lemma}[Theorem~12.2 in \citep{SchapireFr12}] 
  \label{multi:ada_rsk:lem}
  Suppose AdaBoost produces the scoring function
  $\tf{F}$  when run for $T=\sqrt{m}$ rounds with the training set
  $\tf{S}$ and space $\tf{\H}$.
  Then
  \begin{equation}
    \label{multi:ada_rsk:eqn}
     \Pr\enco{
       \rsk\enc{\bar{\tf{F}}} \leq
       \inf_{\tf{F'}:\tf{\X}\to\R}\rsk(\tf{F'}) +
       O\enc{m^{-c}}}
     \geq 1 - \frac{1}{m^2},
  \end{equation}
  where the constant $C$ depends only on the dataset.
\end{lemma}
The proof of Lemma~\ref{multi:rsk_cons:lem} follows immediately from
the above lemma and \eqref{multi:rsk_map:eqn}.